\documentclass[11pt,english]{article}
\usepackage[T1]{fontenc}
\usepackage[latin9]{inputenc}
\usepackage[active]{srcltx}
\usepackage{babel}
\usepackage{float}
\usepackage{booktabs}
\usepackage{units}
\usepackage{amsmath}
\usepackage{amsthm}
\usepackage{amssymb}
\usepackage{graphicx}
\usepackage[authoryear,numbers, compress]{natbib}
\PassOptionsToPackage{normalem}{ulem}
\usepackage{ulem}
\usepackage[unicode=true]
 {hyperref}

\makeatletter

\providecommand{\tabularnewline}{\\}
\floatstyle{ruled}
\newfloat{algorithm}{tbp}{loa}
\providecommand{\algorithmname}{Algorithm}
\floatname{algorithm}{\protect\algorithmname}

\theoremstyle{plain}
\newtheorem{thm}{\protect\theoremname}
\theoremstyle{definition}
\newtheorem{defn}[thm]{\protect\definitionname}
\theoremstyle{definition}
\newtheorem{example}[thm]{\protect\examplename}
\theoremstyle{plain}
\newtheorem{prop}[thm]{\protect\propositionname}
\theoremstyle{plain}
\newtheorem{lem}[thm]{\protect\lemmaname}

\usepackage[letterpaper, margin=1in]{geometry}

\usepackage{breqn}
\usepackage{bbm} 
\usepackage{booktabs}
\usepackage{algorithm,algorithmicx,algpseudocode}
\usepackage{enumitem}
\setlist[enumerate]{leftmargin=*,wide} 
\setlist[itemize]{leftmargin=*,wide} 

\DeclareMathOperator*{\argmax}{arg\,max} \DeclareMathOperator*{\argmin}{arg\,min}  
\DeclareMathOperator*{\Tr}{Tr}
\DeclareMathOperator{\diag}{diag} 
\DeclareMathOperator{\rank}{rank} 
\DeclareMathOperator{\Span}{span}
\DeclareMathOperator{\st}{subject\,to\;}

\global\long\def\rep{R}
\global\long\def\pre{Q}

\global\long\def\P{\mathbb{P}}
\global\long\def\E{\mathbb{E}}

\global\long\def\I{\mathbbm{1}}
\global\long\def\d{\mathrm{d}}

\global\long\def\regret{\mathsf{regret}}
\global\long\def\pure{\mathsf{pure}}
\global\long\def\mix{\mathsf{mix}}
\global\long\def\pl{\mathsf{L}}

\global\long\def\loss{\mathsf{loss}}

\global\long\def\dfn{:=}

\global\long\def\trre[#1,#2]{\overset{{\scriptstyle (#2)}}{#1}} 
\renewcommand\[{\begin{equation}}
\renewcommand\]{\end{equation}}

\algdef{SE}[DOWHILE]{Do}{doWhile}{\algorithmicdo}[1]{\algorithmicwhile\ #1}%

\allowdisplaybreaks

\makeatother

\providecommand{\definitionname}{Definition}
\providecommand{\examplename}{Example}
\providecommand{\lemmaname}{Lemma}
\providecommand{\propositionname}{Proposition}
\providecommand{\theoremname}{Theorem}

\begin{document}
\title{A representation-learning game for classes of prediction tasks}
\author{Neria Uzan and Nir Weinberger\thanks{The authors are with the Department of Electrical and Computer Engineering,
Technion -- Israel Institute of Technology. Emails: \{\texttt{neriauzan@gmail.com},
\texttt{nirwein@technion.ac.il}\}. This research was supported by
the Israel Science Foundation (ISF), grants no. 1782/22.}}
\maketitle
\begin{abstract}
We propose a game-based formulation for learning dimensionality-reducing
representations of feature vectors, when only a prior knowledge on
future prediction tasks is available. In this game, the first player
chooses a representation, and then the second player adversarially
chooses a prediction task from a given class, representing the prior
knowledge. The first player aims is to minimize, and the second player
to maximize, the \emph{regret}: The minimal prediction loss using
the representation, compared to the same loss using the original features.
For the canonical setting in which the representation, the response
to predict and the predictors are all linear functions, and under
the mean squared error loss function, we derive the theoretically
optimal representation in pure strategies, which shows the effectiveness
of the prior knowledge, and the optimal regret in mixed strategies,
which shows the usefulness of randomizing the representation. For
general representations and loss functions, we propose an efficient
algorithm to optimize a randomized representation. The algorithm only
requires the gradients of the loss function, and is based on incrementally
adding a representation rule to a mixture of such rules. 
\end{abstract}

\section{Introduction}

Commonly, data of unlabeled feature vectors $\{\boldsymbol{x}_{i}\}\subset{\cal X}$
is collected without a \emph{specific} downstream prediction task
it will be used for. When a prediction task becomes of interest, responses
$\boldsymbol{y}_{i}\in{\cal Y}$ are also collected, and a learning
algorithm is trained on $\{(\boldsymbol{x}_{i},\boldsymbol{y}_{i})\}$.
Modern sources, such as high-definition images or genomic sequences,
have high dimensionality, and this necessitates \emph{dimensionality-reduction},
either for better generalization \citep{goodfellow2016deep}, for
storage/communication savings \citep{tsitsiklis1989decentralized,Nguyen2009OnSurrogate,Duchi2018Multiclass},
or for interpretability \citep{Schapire2012Boosting}. The goal is
thus to find a low-dimensional \emph{representation }$\boldsymbol{\boldsymbol{z}}=\rep(\boldsymbol{x})\in\mathbb{R}^{r}$,
that preserves the relevant part of the features, without a full knowledge
of the downstream prediction task. In this paper, we propose a game-theoretic
framework for this goal, by assuming that the learner has prior knowledge
on the \emph{class} of downstream prediction tasks. Our contributions
are a theoretical solution in the linear setting, under the mean squared
error (MSE) loss, and an algorithm for the general setting. 

Unsupervised methods for dimensionality reduction, such as \emph{principal
component analysi}s (PCA) \citep{pearson1901liii,jolliffe2005principal,cunningham2015linear,johnstone2018pca},
and non-linear extensions such as kernel PCA \citep{scholkopf1998nonlinear}
and \emph{auto-encoders }(AE) \citep{kramer1991nonlinear,hinton2006reducing,lee2011unsupervised,goodfellow2016deep},
aim that the representation $\boldsymbol{\boldsymbol{z}}$ will maximally
preserve the \emph{variation} in $\boldsymbol{x}$, and thus ignore
any prior knowledge on future prediction tasks. This prior knowledge
may indicate, e.g., that highly varying directions in the feature
space are, in fact, irrelevant for downstream prediction tasks. From
the supervised learning perspective, the \emph{information bottleneck}
(IB) principle \citep{tishby2000information,chechik2003information,slonim2006multivariate,harremoes2007information}
was used to postulate that efficient supervised learning necessitates
representations that are both low-complexity and relevant \citep{tishby2015deep,shwartz2017opening,shwartz2022information,achille2018emergence,achille2018information}
(see Appendix \ref{sec:Additional-related-work}). To corroborate
this claim, \citet{dubois2020learning} proposed a game-theoretic
formulation (based on a notion of \emph{usable information}, introduced
by \citet{xu2020theory}), in which Alice selects a prediction problem
of $\boldsymbol{y}$ given $\boldsymbol{x}$, and then Bob selects
the representation $\boldsymbol{z}$. The pay-off is the minimal risk
possible using this representation. \citet{dubois2020learning} showed
that ideal generalization is obtained for representations that solve
the resulting \emph{decodable IB problem.}

Evidently, the game of \citet{dubois2020learning} is tailored to
supervised learning, since the representation is optimized based on
the prediction problem. In this paper, we assume the learner collected
unlabeled feature vectors, and has a prior knowledge that the downstream
prediction problem belongs to a class ${\cal F}$ of response functions.
We propose a game different from \citet{dubois2020learning}: First,
the \emph{representation player} chooses a rule $\rep$ to obtain
$\boldsymbol{z}=\rep(\boldsymbol{x})\in\mathbb{R}^{r}$. Second, the
\emph{response function player} chooses $\boldsymbol{y}=f(\boldsymbol{x})$
where $f\in{\cal F}$ is possibly random. The payoff for $(\rep,f)$
is the \emph{regret}: The minimal prediction loss of $\boldsymbol{y}$
based on $\boldsymbol{z}$ compared to the minimal prediction loss
based on $\boldsymbol{x}$. The goal of the representation player
(resp. response function player) is to minimize (resp. maximize) the
payoff, and the output of this game is the saddle-point representation.
Compared to \citet{dubois2020learning}, the representation is chosen
based only on the \emph{class} of response functions, rather than
a specific one. The \emph{minimax regret} measures the worst-case
regret (over functions in ${\cal F}$) of a learner that uses the
saddle-point representation. We derive the minimax representation
both in \emph{pure strategies} and in \emph{mixed strategies} \citep{owen2013game}.
Mixed strategies use \emph{randomized} representation rules, whose
utilization is illustrated as follows: Assume that routinely collected
images are required to be compressed. There are two compression (representation)
methods. The first smoothens the images, and the second preserves
their edges; exclusively using the first method would hinder any possibility
of making predictions reliant on the image's edges. This is prevented
by randomly alternating the use of both methods. 

The class ${\cal F}$ manifests the prior knowledge on the downstream
prediction tasks that will use the represented features, and may stem
from domain specific considerations; imposed by privacy or fairness
constraints; or emerge from transfer or continual learning settings;
see Appendix \ref{sec:Classes-of-response} for an extended discussion.
The resulting formulation encompasses an entire spectrum of possibilities:
(1) \emph{Supervised learning}: ${\cal F}=\{f\}$ is a singleton,
and thus known to the learner. (2) \emph{Multitask learning} \citep{baxter2000model,maurer2016benefit,tripuraneni2020theory,tripuraneni2021provable}:
${\cal F}=\{f_{1},f_{2},\cdots,f_{t}\}$ is a finite set of functions.
(3) \emph{Prior expert knowledge}: ${\cal F}$ represents a \emph{continuous,
yet restricted} set of functions. For example, the response functions
in ${\cal F}$ may be more sensitive to certain features than others.
(4) \emph{No supervision}: ${\cal F}$ is the class of \emph{all}
possible response functions, which is essentially an unsupervised
learning problem, since no valuable information is known. We focus
on the intermediate regimes of partial supervision, in which the learner
should (and can) optimize its representation to be jointly efficient
for all tasks in ${\cal F}$. In this respect, multitask learning
(case 2) is a generalization of supervised learning, since the learner
can simulate the $t$ functions in ${\cal F}$. Prior knowledge (case
3) is more similar to unsupervised learning, since the learner will
not be able to efficiently do so.

\paragraph*{Theoretical contribution}

We address the fundamental setting in which the representation, the
response, and the prediction are all linear functions, under the mean
squared error (MSE) loss (Section \ref{sec:The-linear-setting}).
The class is $\mathcal{F}_{S}=\{\|f\|_{S}\leq1\}$ for a known symmetric
matrix $S$. Combined with the covariance matrix of the features,
they determine the relevant directions of the function in the feature
space, in contrast to just the features variability, as in standard
unsupervised learning. We establish the optimal representation and
regret in pure strategies, which shows the utility of the prior information,
and in mixed strategies, which shows that randomizing the representation
yields \emph{strictly lower} regret. We prove that randomizing between
merely $\ell^{*}$ different representation rules suffices, where
$r+1\leq\ell^{*}\leq d$ is a precisely characterized \emph{effective
dimension}. Interestingly, mixed representations do not improve the
regret in standard unsupervised learning (see Proposition \ref{prop: regret PCA}
in Appendix \ref{subsec:Standard-principle-component} for the PCA
setting).

\paragraph*{Algorithmic contribution}

We develop an algorithm for optimizing mixed representations (Section
\ref{sec:An-iterative-algorithm}) for general representations/response/predictors
and loss functions, based only on their gradients. Similarly to boosting
\citep{Schapire2012Boosting}, the algorithm operates incrementally.
At each iteration it finds the response function in ${\cal F}$ that
is most poorly predicted by the current mixture of representation
rules. An additional representation rule is added to the mixture,
based on this function and the ones from previous iterations. The
functions generated by the algorithm can be considered a ``self-defined
signals'', similarly to \emph{self-supervised learning} \citep{oord2018representation,shwartz2023compress}.
To optimize the weights of the representation, the algorithm solves
a two-player game using the classic multiplicative weights update
(MWU) algorithm \citep{freund1999adaptive} (a follow-the-regularized-leader
algorithm \citep{shalev2012online,hazan2016introduction}).

\paragraph*{Related work }

In Appendix \ref{sec:Additional-related-work} we discuss: The IB
principle and compare the game of \citet{dubois2020learning} with
ours; The generalization error bounds of \citet{maurer2016benefit,tripuraneni2020theory}
for multi-task learning and learning-to-learn, and how our regret
bound complements these results; The use of randomization in representation
learning, similarly to our mixed strategies solution; Iterative algorithms
for solving minimax games, and specifically the incremental algorithm
approach for learning mixture models \citep{Schapire2012Boosting,tolstikhin2017adagan},
which we adopt.

\section{Problem formulation \label{sec:Problem-formulation}}

Notation conventions are mostly standard, and are detailed in Appendix
\ref{sec:Notation-conventions}. Specifically, the eigenvalues of
$S\in\mathbb{S}_{+}^{d}$ are denoted as $\lambda_{\text{max}}(S)\equiv\lambda_{1}(S)\geq\cdots\geq\lambda_{d}(S)=\lambda_{\text{min}}(S)$
and $v_{i}(S)$ denotes an eigenvector corresponding to $\lambda_{i}(S)$
such that $V=$ $V(S)\dfn[v_{1}(S),v_{2}(S),\cdots,v_{d}(S)]\in\mathbb{R}^{d\times d}$
and $S=V(S)\Lambda(S)V^{\top}(S)$ is an eigenvalue decomposition.
$W_{i:j}\dfn[w_{i},\ldots,w_{j}]\in\mathbb{R}^{(j-i+1)\times d}$
is the matrix comprised of the columns of $W\in\mathbb{R}^{d\times d}$
indexed by $\{i,\ldots,j\}$. The probability law of a random variable
$\boldsymbol{x}$ is denoted as $\pl(\boldsymbol{x})$. 

Let $\boldsymbol{x}\in{\cal X}$ be a random feature vector, with
probability law $P_{\boldsymbol{x}}\dfn\pl(\boldsymbol{x})$. Let
$\boldsymbol{y}\in{\cal Y}$ be a corresponding response drawn according
to a probability kernel $\boldsymbol{y}\sim f(\cdot\mid\boldsymbol{x}=x)$,
where for brevity, we will refer to $f$ as the \emph{response function}
(which can be \emph{random}). It is known that $f\in{\cal F}$ for
some known class ${\cal F}$. Let $\boldsymbol{z}\dfn\rep(\boldsymbol{x})\in\mathbb{R}^{r}$
be an $r$-dimensional representation of $\boldsymbol{x}$ where $\rep\colon{\cal X}\to\mathbb{R}^{r}$
is chosen from a class ${\cal R}$ of representation functions, and
let $\pre\colon{\cal X}\to{\cal Y}$ be a prediction rule from a class
${\cal Q}_{{\cal X}}$, with the loss function $\loss\colon{\cal Y}\times{\cal Y}\to\mathbb{R}_{+}$.
The pointwise \emph{regret} of $(\rep,f)$ is 
\begin{equation}
\regret(\rep,f\mid P_{\boldsymbol{x}})\dfn\min_{\pre\in{\cal Q}_{\mathbb{R}^{r}}}\E\left[\loss(\boldsymbol{y},\pre(\rep(\boldsymbol{x})))\right]-\min_{\pre\in{\cal Q}_{{\cal X}}}\E\left[\loss(\boldsymbol{y},\pre(\boldsymbol{x}))\right].\label{eq: pointwise regret}
\end{equation}
The \emph{minimax regret in mixed strategies} is the worst case response
function in ${\cal F}$ given by 
\begin{equation}
\regret_{\mix}({\cal R},{\cal F}\mid P_{\boldsymbol{x}})\dfn\min_{\pl(\boldsymbol{\rep})\in{\cal P}({\cal R})}\max_{f\in{\cal F}}\E\left[\regret(\boldsymbol{\rep},f\mid P_{\boldsymbol{x}})\right],\label{eq: minimax regret mixed}
\end{equation}
where ${\cal P}({\cal R})$ is a set of probability measures on the
possible set of representations ${\cal R}$. The \emph{minimax regret
in pure strategies }restricts ${\cal P}({\cal R})$ to degenerated
measures (deterministic), and so the expectation in \eqref{eq: minimax regret mixed}
is removed. Our main goal is to determine the optimal representation
strategy, either in pure $\rep^{*}\in{\cal R}$ or mixed strategies
$\pl(\boldsymbol{\rep}^{*})\in{\cal P}({\cal R})$. To this end, we
will also utilize the \emph{maximin }version of \eqref{eq: minimax regret mixed}.
Specifically, let ${\cal P}({\cal F})$ denote a set of probability
measures supported on ${\cal F}$, and assume that for any $\rep\in{\cal R}$,
there exists a measure in ${\cal P}({\cal R})$ that puts all its
mass on $\rep$. Then, the \emph{minimax theorem }\citep[Chapter 2.4]{owen2013game}
\citep{sion1958general} implies that 
\begin{align}
\regret_{\mix}({\cal R},{\cal F}\mid P_{\boldsymbol{x}}) & =\max_{\pl(\boldsymbol{f})\in{\cal P}({\cal F})}\min_{\rep\in{\cal R}}\E\left[\regret(\rep,\boldsymbol{f}\mid P_{\boldsymbol{x}})\right].\label{eq: maximin regret mixed}
\end{align}
The right-hand side of \eqref{eq: maximin regret mixed} is the \emph{maximin
regret in mixed strategies, }and the maximizing probability law $\pl(\boldsymbol{f}^{*})$
is known as the \emph{least favorable prior}. In general, $\regret_{\mix}({\cal R},{\cal F}\mid P_{\boldsymbol{x}})\leq\regret_{\pure}({\cal R},{\cal F}\mid P_{\boldsymbol{x}})$,
and the inequality can be strict. We focus on the representation aspect,
and thus assumed that $P_{\boldsymbol{x}}$ is known and that sufficient
labeled data will be provided to the learner from the subsequent prediction
task. We also note that, as common in game-theory, the mixed minimax
regret is achieved for \emph{repeating} representation games \citep{owen2013game},
which fits the scenario in which routinely collected data is to be
represented. The size of the dataset for each of these games should
be large enough to allow for accurate learning of $\boldsymbol{f}$
to be used by the predictor. By contrast, the pure minimax regret
guarantee is valid for a single representation, and thus more conservative
from this aspect.

\section{The linear setting under MSE loss \label{sec:The-linear-setting}}

In this section, we focus on linear classes and the MSE loss function.
The response function class ${\cal F}$ is characterized by a quadratic
constraint specified by a matrix $S\in\mathbb{S}_{++}^{d}$, which
represents the relative importance of each direction in the feature
space in determining $\boldsymbol{y}$.
\begin{defn}[The linear MSE setting]
\label{def: linear MSE}Assume that ${\cal X}=\mathbb{R}^{d}$, that
${\cal Y}=\mathbb{R}$ and a squared error loss function $\loss(y_{1},y_{2})=|y_{1}-y_{2}|^{2}$.
Assume that $\E[\boldsymbol{x}]=0$ and let $\Sigma_{\boldsymbol{x}}\dfn\E[\boldsymbol{x}\boldsymbol{x}^{T}]\in\mathbb{S}_{++}^{d}$
be its invertible covariance matrix. The classes of representations,
response functions, and predictors are all linear, that is: (1) The
representation is $z=\rep(x)=R^{\top}x$ for $R\in{\cal R}:=\mathbb{R}^{d\times r}$
where $d>r$; (2) The response function is $f\in{\cal F}\subset\mathbb{R}^{d}$
, and $\boldsymbol{y}=f^{\top}\boldsymbol{x}+\boldsymbol{n}\in\mathbb{R}$,
where $\boldsymbol{n}\in\mathbb{R}$ is a heteroscedastic noise that
satisfies $\E[\boldsymbol{n}\mid\boldsymbol{x}]=0$, and given some
specified $S\in\mathbb{S}_{++}^{d}$ it holds that 
\begin{equation}
f\in{\cal F}_{S}\dfn\left\{ f\in\mathbb{R}^{d}\colon\|f\|_{S}^{2}\leq1\right\} ,\label{eq: quadratic constraints}
\end{equation}
where $\|f\|_{S}\dfn\|S^{-1/2}f\|_{2}=(f^{\top}S^{-1}f)^{1/2}$ is
the Mahalanobis norm; (3) The predictor is $\pre(z)=q^{\top}z\in\mathbb{R}$
for  $q\in\mathbb{R}^{r}$. Since the regret will depend on $P_{\boldsymbol{x}}$
only via $\Sigma_{\boldsymbol{x}}$, we will abbreviate the notation
of the pure (resp. mixed) minimax regret to $\regret_{\pure}({\cal F}\mid\Sigma_{\boldsymbol{x}})$
(resp. $\regret_{\mix}({\cal F}\mid\Sigma_{\boldsymbol{x}})$). 
\end{defn}

In Appendix \ref{subsec:Standard-principle-component} we show that
standard PCA can be similarly formulated, by assuming that ${\cal F}$
is a singleton containing the noiseless identity function, so that
$\boldsymbol{y}=\boldsymbol{x}$ surely holds, and $\hat{x}=\pre(z)\in\mathbb{R}^{d}$.
Proposition \ref{prop: regret PCA} therein shows that the pure and
mixed minimax representations are both $R=V_{1:r}(\Sigma_{\boldsymbol{x}})$,
and so randomization is unnecessary. We begin with the pure minimax
regret.
\begin{thm}
\label{thm: pure minimax regret linear quadratic constraint}For the
linear MSE setting (Definition \ref{def: linear MSE}) 
\begin{equation}
\regret_{\pure}({\cal F}_{S}\mid\Sigma_{\boldsymbol{x}})=\lambda_{r+1}\left(\Sigma_{\boldsymbol{x}}^{1/2}S\Sigma_{\boldsymbol{x}}^{1/2}\right).\label{eq: Linear MSE pure minimax regret}
\end{equation}
A minimax representation matrix is 
\[
R^{*}\dfn\Sigma_{\boldsymbol{x}}^{-1/2}\cdot V_{1:r}\left(\Sigma_{\boldsymbol{x}}^{1/2}S\Sigma_{\boldsymbol{x}}^{1/2}\right),
\]
and the worst case response function is 
\[
f^{*}\dfn S^{1/2}\cdot v_{r+1}\left(\Sigma_{\boldsymbol{x}}^{1/2}S\Sigma_{\boldsymbol{x}}^{1/2}\right).
\]
\end{thm}

The optimal representation thus whitens the feature vector $\boldsymbol{x}$,
and then projects it on the top $r$ eigenvectors of the adjusted
covariance matrix $\Sigma_{\boldsymbol{x}}^{1/2}S\Sigma_{\boldsymbol{x}}^{1/2}$,
which reflects the prior knowledge that $f\in{\cal F}_{S}$. The proof
in Appendix \ref{subsec:Analysis-of-pure-mixed} has the following
outline: Plugging the optimal predictor into the regret results a
quadratic form in $f\in\mathbb{R}^{d}$, determined by a matrix which
depends on the subspace spanned by the representation $R$. The worst-case
$f$ is the determined via the \emph{Rayleigh quotient theorem}, and
the optimal $R$ is found via the \emph{Courant--Fischer variational
characterization} (see Appendix \ref{sec:Useful-mathematical-results}
for a summary of these results). We next consider the mixed minimax
regret:
\begin{thm}
\label{thm: mixed minimax regret linear quadratic constraint}For
the linear MSE setting (Definition \ref{def: linear MSE}), let $\lambda_{i}\equiv\lambda_{i}(S^{1/2}\Sigma_{\boldsymbol{x}}S^{1/2})$
for $i\in[d]$ and $\lambda_{d+1}\equiv0$, and let $\ell^{*}$ be
any member of 
\begin{equation}
\left\{ \ell\in[d]\backslash[r]\colon(\ell-r)\cdot\lambda_{\ell}^{-1}\leq\sum_{i=1}^{\ell}\lambda_{i}^{-1}\leq(\ell-r)\cdot\lambda_{\ell+1}^{-1}\right\} .\label{eq: condition on optimal rank of least favorable covariance matrix}
\end{equation}
\begin{itemize}
\item The minimax regret in mixed strategies is
\end{itemize}
\begin{equation}
\regret_{\mix}({\cal F}_{S}\mid\Sigma_{\boldsymbol{x}})=\frac{\ell^{*}-r}{\sum_{i=1}^{\ell^{*}}\lambda_{i}^{-1}}.\label{eq: Linear MSE mixed minimax regret}
\end{equation}
\begin{itemize}
\item The covariance matrix of the least favorable prior of $\boldsymbol{f}$
is 
\begin{equation}
\Sigma_{\boldsymbol{f}}^{*}\dfn\frac{V^{\top}\Lambda_{\ell^{*}}^{-1}V}{\sum_{i=1}^{\ell^{*}}\lambda_{i}^{-1}}.\label{eq: maximin Sigma_f}
\end{equation}
where $\Lambda_{\ell}\dfn\diag(\lambda_{1},\ldots,\lambda_{\ell^{*}},0,\cdots,0)$,
and $V\equiv V(S^{1/2}\Sigma_{\boldsymbol{x}}S^{1/2})$. 
\item The probability law of the minimax representation: Let $\overline{A}\in\{0,1\}^{\ell^{*}\times{\ell^{*} \choose r}}$
be a matrix whose columns are the members of the set $\overline{{\cal A}}\dfn\{\overline{a}\in\{0,1\}^{\ell^{*}}\colon\|\overline{a}\|_{1}=\ell^{*}-r\}$
(in some order). Let $\overline{b}=(b_{1},\ldots,b_{\ell^{*}})^{\top}$
be such that 
\[
b_{i}=(\ell^{*}-r)\cdot\frac{\lambda_{i}^{-1}}{\sum_{j=1}^{\ell^{*}}\lambda_{j}^{-1}}.
\]
Then, there exists a solution $p\in[0,1]^{{\ell^{*} \choose r}}$
to $\overline{A}p=\overline{b}$ with support size at most $\ell^{*}+1$.
For $j\in[{\ell^{*} \choose r}]$, let ${\cal I}_{j}\dfn\{i\in[\ell^{*}]\colon\overline{A}_{ij}=0\}$
be the zero indices on the $j$th column of $\overline{A}$, and let
$V_{{\cal I}_{j}}$ denote the $r$ columns of $V$ whose index is
in ${\cal I}_{j}$. A minimax representation is $\boldsymbol{R}^{*}=\Sigma_{\boldsymbol{x}}^{-1/2}V_{{\cal I}_{j}}$
with probability $p_{j}$, for $j\in[{\ell^{*} \choose r}]$.
\end{itemize}
\end{thm}

Interestingly, while the eigenvalues $\lambda_{i}(\Sigma_{\boldsymbol{x}}^{1/2}S\Sigma_{\boldsymbol{x}}^{1/2})=\lambda_{i}(S^{1/2}\Sigma_{\boldsymbol{x}}S^{1/2})$
are equal, the pure minimax regret utilizes the eigenvectors of $\Sigma_{\boldsymbol{x}}^{1/2}S\Sigma_{\boldsymbol{x}}^{1/2}$
whereas the mixed minimax regret utilizes those of $S^{1/2}\Sigma_{\boldsymbol{x}}S^{1/2}$,
which are possibly different. The proof of Theorem \ref{thm: mixed minimax regret linear quadratic constraint}
is also in Appendix \ref{subsec:Analysis-of-pure-mixed}, and is substantially
more complicated than for the pure regret. The reason is that directly
maximizing over $\pl(\boldsymbol{R})$ is challenging, and so we take
a two-step indirect approach. The outline is as follows: First, we
solve the \emph{maximin problem} \eqref{eq: maximin regret mixed},
and find the least favorable prior $\pl(\boldsymbol{f}^{*})$. Second,
we propose a probability law for the representation $\pl(\boldsymbol{R})$,
and show that its regret equals the maximin value, and thus also the
minimax. With more detail, in the first step, we show that the regret
only depends on $\pl(\boldsymbol{f})$ via $\Sigma_{\boldsymbol{f}}=\E[\boldsymbol{f}\boldsymbol{f}^{\top}]$,
and we explicitly construct $\pl(\boldsymbol{f})$ that is supported
on ${\cal F}_{S}$ and has this covariance matrix. This reduces the
problem from optimizing $\pl(\boldsymbol{f})$ to optimizing $\Sigma_{\boldsymbol{f}}$,
whose solution (Lemma \ref{lem: minimal eigs maximization under modified trace constraints})
results the least favorable $\Sigma_{\boldsymbol{f}}^{*}$, and then
the maximin value. In the second step, we construct a representation
that achieves the maximin regret. Concretely, we construct representation
matrices that use $r$ of the $\ell^{*}$ principal components of
$\Sigma_{\boldsymbol{x}}^{1/2}S\Sigma_{\boldsymbol{x}}^{1/2}$, where
$\ell^{*}>r$. The defining property of $\ell^{*}$ \eqref{eq: condition on optimal rank of least favorable covariance matrix},
established in the maximin solution, is utilized to find proper weights
on the ${\ell^{*} \choose r}$ possible representations, which achieve
the maximin solution, and thus also the minimax. The proof then uses
Carath\'{e}odory's theorem (see Appendix \ref{sec:Useful-mathematical-results})
to establish that the optimal $\{p_{j}\}$ is supported on at most
$\ell^{*}+1$ matrices, much less than the potential support of ${\ell^{*} \choose r}$.
We next make a few comments:
\begin{enumerate}
\item \uline{Computing the mixed minimax probability:} This requires
solving $\overline{A}^{\top}p=\overline{b}$ for a probability vector
$p$. This is a linear-program feasibility problem, which is routinely
solved \citep{bertsimas1997introduction}. For illustration, if $r=1$
then $p_{j}=1-(\ell^{*}-1)\lambda_{j}^{-1}/(\sum_{i=1}^{\ell^{*}}\lambda_{i}^{-1})$
for $j\in[\ell^{*}]$, and if $\ell^{*}=r+1$ then $p_{j}=(\lambda_{j}^{-1})/(\sum_{j'=1}^{\ell^{*}}\lambda_{j'}^{-1})$
on the first $\ell^{*}$ standard basis vectors. Nonetheless, the
dimension of $p$ is ${\ell^{*} \choose r}$ and thus increases fast
as $\Theta((\ell^{*})^{r})$, which is infeasible for high dimensions.
In this case, the algorithm of Section \ref{sec:An-iterative-algorithm}
can be used. As we empirically show it approximately achieves the
optimal regret, with slightly more than $\ell^{*}+1$ atoms (see Example
\ref{exa: Algorithm linear MSE} next).\textbf{ }
\item \uline{Required randomness:} The regret formulation \eqref{eq: minimax regret mixed}
assumes that the actual realization of the representation rule is
known to the predictor. Formally, this can be conveyed to the predictor
using an small header of less than $\log_{2}(\ell^{*}+1)\leq\log(d+1)$
bits. Practically, this is unnecessary and an efficient predictor
can be learned from a labeled dataset $(\boldsymbol{z},\boldsymbol{y})$.
\item \uline{\label{enu:The-rank-of}The rank of \mbox{$\Sigma_{\boldsymbol{f}}^{*}$}:}
The rank of the covariance matrix of the least favorable prior is
an \emph{effective dimension}, satisfying (see \eqref{eq: Linear MSE mixed minimax regret})
\begin{equation}
\ell^{*}=\argmax_{\ell\in[d]\backslash[r]}\frac{1-(r/\ell)}{\frac{1}{\ell}\sum_{i=1}^{\ell}\lambda_{i}^{-1}}.\label{eq: determinig ell star}
\end{equation}
By convention, $\{\lambda_{i}^{-1}\}_{i\in[d]}$ is a monotonic non-decreasing
sequence, and so is the partial Cesàro mean $\psi(\ell)\dfn\frac{1}{\ell}\sum_{i=1}^{\ell}\lambda_{i}^{-1}$.
For example, if $\lambda_{i}=i^{-\alpha}$ with $\alpha>0$ then $\psi(\ell)=\Theta(\ell^{\alpha})$.
So, if, e.g., $\psi(\ell)=\ell^{\alpha}$, then it is easily derived
that $\ell^{*}\approx\min\{\frac{\alpha+1}{\alpha}r,d\}$. Hence,
if $\alpha\geq\frac{r}{d-r}$ is large enough and the decay rate of
$\{\lambda_{i}\}$ is fast enough then $\ell^{*}<d$, but otherwise
$\ell^{*}=d$. As the decay rate of $\{\lambda_{i}\}$ becomes faster,
the rank of $\Sigma_{\boldsymbol{f}}^{*}$ decreases from $d$ to
$r$. Importantly, $\ell^{*}\geq r+1$ always holds, and so the optimal
mixed representation is not deterministic even if $S^{1/2}\Sigma_{\boldsymbol{x}}S^{1/2}$
has less than $r$ significant eigenvalues (which can be represented
by a single matrix $R\in\mathbb{R}^{d\times r}$). Hence, the mixed
minimax regret is always \emph{strictly lower} than the pure minimax
regret. Thus, even when $S=I_{d}$, and no valuable prior knowledge
is known on the response function, the mixed minimax representation
is different from the standard PCA solution of top $r$ eigenvectors
of $\Sigma_{\boldsymbol{x}}$. 
\item \uline{\label{enu:Uniqueness-of-the}Uniqueness of the optimal
representation:} Since one can always post-multiply $R^{\top}x$ by
some invertible matrix, and then pre-multiply $z=R^{\top}x$ by its
inverse, the following holds: If ${\cal R}$ and ${\cal Q}$ are not
further restricted, then if $\boldsymbol{R}$ is a minimax representation,
and $W(\boldsymbol{R})\in\mathbb{R}^{r\times r}$ is an invertible
matrix, then $\boldsymbol{R}\cdot W(\boldsymbol{R})$ is also a minimax
representation. 
\item \uline{Infinite-dimensional features:} Theorems \ref{thm: pure minimax regret linear quadratic constraint}
and \ref{thm: mixed minimax regret linear quadratic constraint} assume
a finite dimensional feature space. We show in Appendix \ref{sec:The-Hilbert-space}
that the results can be generalized to an infinite dimensional Hilbert
space ${\cal X}$, in the more restrictive setting that the noise
$\boldsymbol{n}$ is statistically independent of $\boldsymbol{x}$. 
\end{enumerate}
\begin{example}
\label{exa: Identity of response weights}Assume $S=I_{d}$, and denote,
for brevity, $V\equiv V(\Sigma_{\boldsymbol{x}})\dfn[v_{1},\ldots,v_{d}]$
and $\Lambda\equiv\Lambda(\Sigma_{\boldsymbol{x}})\dfn\diag(\lambda_{1},\ldots,\lambda_{d})$.
The optimal minimax representation in pure strategies (Theorem \ref{thm: pure minimax regret linear quadratic constraint})
is \textbf{
\begin{align}
R^{*}=\Sigma_{\boldsymbol{x}}^{-1/2}\cdot V_{1:r}=V\Lambda_{\boldsymbol{x}}^{-1/2}V^{\top}V_{1:r} & =V\Lambda_{\boldsymbol{x}}^{-1/2}\cdot[e_{1},\ldots,e_{r}]=\left[\lambda_{1}^{-1/2}\cdot v_{1},\ldots,\lambda_{r}^{-1/2}\cdot v_{r}\right],\label{eq: pure minimax representation S is identity}
\end{align}
}which is comprised of the top $r$ eigenvectors of $\Sigma_{\boldsymbol{x}}$,
scaled so that $v_{i}^{\top}\boldsymbol{x}$ has unit variance. By
Comment \ref{enu:Uniqueness-of-the} above, $V_{1:r}$ is also an
optimal minimax representation. The worst case response is $f=v_{r+1}(\Sigma_{\boldsymbol{x}})$
and, as expected, since $\rep$ uses the first $r$ principal directions
\[
\regret_{\pure}({\cal F}\mid\Sigma_{\boldsymbol{x}})=\lambda_{r+1}.
\]
The minimax regret in mixed strategies (Theorem \ref{thm: mixed minimax regret linear quadratic constraint})
is different, and given by
\[
\regret_{\mix}({\cal F}\mid\Sigma_{\boldsymbol{x}})=\frac{\ell^{*}-r}{\sum_{i=1}^{\ell^{*}}\lambda_{i}^{-1}},
\]
where $\ell^{*}$ is determined by the decay rate of the eigenvalues
of $\Sigma_{\boldsymbol{x}}$ (see \eqref{eq: condition on optimal rank of least favorable covariance matrix}).
It can be easily verified that the mixed minimax regret is \emph{strictly
lower} than the pure minimax regret. Indeed, it holds for any $\ell>r$
that 
\begin{equation}
\frac{\ell-r}{\sum_{i=1}^{\ell}\lambda_{i}^{-1}}\leq\frac{\ell-r}{\ell\lambda_{r+1}^{-1}}\leq\lambda_{r+1}.\label{eq: maximin regret is less than minimax regret}
\end{equation}

The least favorable covariance matrix is given by (Theorem \ref{thm: mixed minimax regret linear quadratic constraint})\textbf{
\[
\Sigma_{\boldsymbol{f}}^{*}=\left[\sum_{i=1}^{\ell^{*}}\lambda_{i}^{-1}\right]^{-1}\cdot V\diag\left(\lambda_{1}^{-1},\ldots,\lambda_{\ell^{*}}^{-1},0,\cdots,0\right)\cdot V^{\top}.
\]
}Intuitively, $\Sigma_{\boldsymbol{f}}^{*}$ equalizes the first $\ell^{*}$
eigenvalues of $\Sigma_{\boldsymbol{x}}\Sigma_{\boldsymbol{f}}^{*}$
(and nulls the other $d-\ell^{*}$), to make the representation indifferent
to these $\ell^{*}$ directions. As evident from the regret, the ``equalization''
of the $i$th eigenvalue adds a term of $\lambda_{i}^{-1}$ to the
denominator, and if $\lambda_{i}$ is too small then $v_{i}$ is not
chosen for the representation. This agrees with Comment \ref{enu:The-rank-of},
which states that a fast decay of $\{\lambda_{i}\}$ reduces $\ell_{*}$
away from $d$. A derivation similar to \eqref{eq: pure minimax representation S is identity}
shows that the mixed minimax representation sets 
\[
\boldsymbol{R}^{*}=\Sigma_{\boldsymbol{x}}^{-1/2}\cdot V_{{\cal I}_{j}}=\left[\lambda_{i_{j,1}}^{-1/2}\cdot v_{i_{j,1}},\ldots,\lambda_{i_{j,r}}^{-1/2}\cdot v_{i_{j,r}}\right]
\]
with probability $p_{j}$, where ${\cal I}_{j}\equiv\{i_{j,1},\ldots,i_{j,r}\}$.
Thus, the optimal representation chooses a random subset of $r$ vectors
from $\{v_{1},\ldots,v_{\ell^{*}}\}$. See the left panel of Figure
\ref{fig: linear MSE regret examples} for a numerical example.

\begin{figure}
\centering{}\includegraphics[scale=0.18]{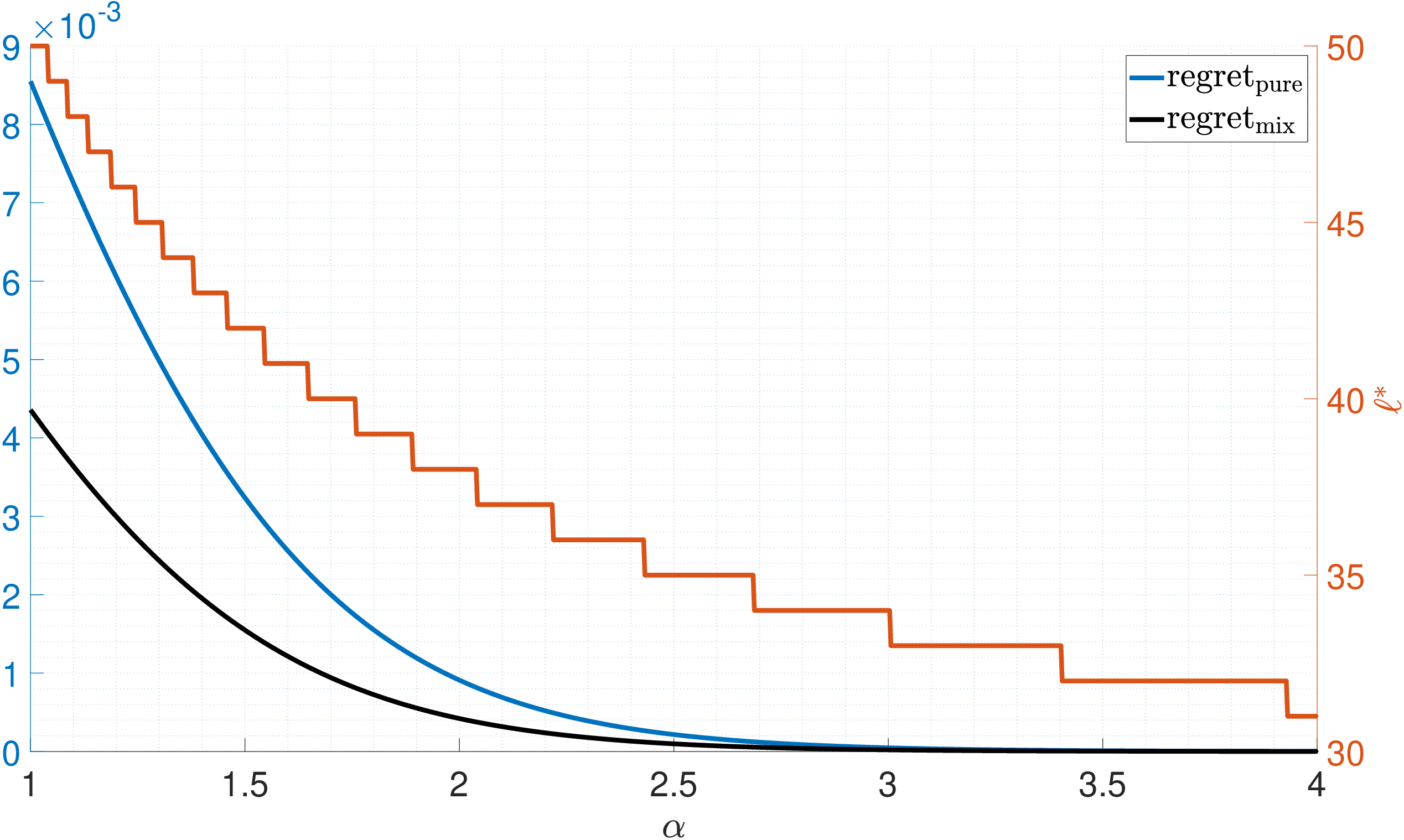}\includegraphics[scale=0.18]{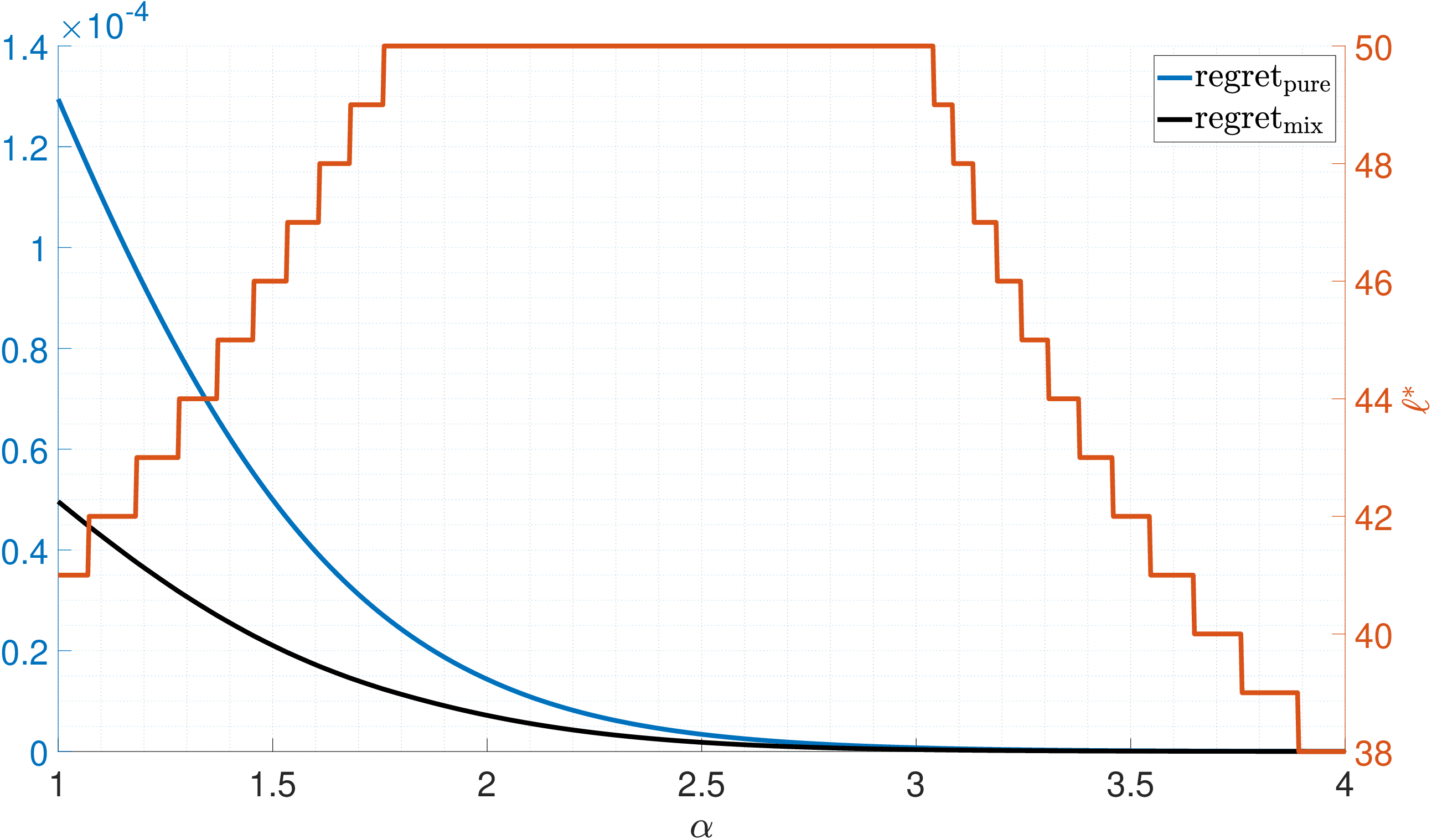}\caption{Left: Pure and mixed minimax regret and $\ell_{*}$ for Example \ref{exa: Identity of response weights},
for $d=50,r=25$, with $\text{\ensuremath{\lambda_{i}=\sigma_{i}^{2}\propto i^{-\alpha}}}$.
Right: Pure and mixed minimax regret and $\ell_{*}$ for Example \ref{exa: diagonal case},
for $d=50,r=25$, with $\text{\ensuremath{\sigma_{i}^{2}\propto i^{-\alpha}} and \ensuremath{s_{i}\propto i^{2}}}$.
The trend of $\ell_{*}$ is reversed for $\alpha>2$. \label{fig: linear MSE regret examples}}
\end{figure}
\end{example}

\begin{example}
\label{exa: diagonal case}To demonstrate the effect of prior knowledge
on the response function, we assume $\Sigma_{\boldsymbol{x}}=\diag(\sigma_{1}^{2},\ldots,\sigma_{d}^{2})$
and $S=\diag(s_{1},\ldots,s_{d})$, where $\sigma_{1}^{2}\geq\sigma_{2}^{2}\geq\cdots\geq\sigma_{d}^{2}$
(but $\{s_{i}\}_{i\in[d]}$ are not necessarily ordered). Letting
$f=(f_{1},\ldots,f_{d})$, the class of response functions is ${\cal F}_{S}\dfn\{f\in\mathbb{R}^{d}\colon\sum_{i=1}^{d}(f_{i}^{2}/s_{i})\leq1\}$,
and so coordinates $i\in[d]$ with a large $s_{i}$ have large influence
on the response. Let $(i_{(1)},\ldots,i_{(d)})$ be a permutation
of $[d]$ so that $\sigma_{i(j)}^{2}s_{i(j)}$ it the $j$th largest
value of $(\sigma_{i}^{2}s_{i})_{i\in[d]}$. The pure minimax regret
is (Theorem \ref{thm: pure minimax regret linear quadratic constraint})
\[
\regret_{\pure}({\cal F}\mid\Sigma_{\boldsymbol{x}})=\sigma_{i_{r+1}}^{2}s_{i_{r+1}}.
\]
The optimal representation is $R=[e_{i_{(1)}},e_{i_{(2)}},\ldots,e_{i_{(r)}}]$,
that is, uses the most influential coordinates, according to $\{s_{i}\}$,
which may be different from the $r$ principal directions of $\Sigma_{\boldsymbol{x}}$.
For the minimax regret in mixed strategies, Theorem \ref{thm: mixed minimax regret linear quadratic constraint}
results
\[
\regret_{\mix}({\cal F}\mid\Sigma_{\boldsymbol{x}})=\frac{\ell^{*}-r}{\sum_{j=1}^{\ell^{*}}(s_{i_{j}}\sigma_{i_{j}}^{2})^{-1}}
\]
for $\ell^{*}\in[d]\backslash[r]$ satisfying \eqref{eq: condition on optimal rank of least favorable covariance matrix},
and the covariance matrix of the least favorable prior is given by
\[
\Sigma_{\boldsymbol{f}}^{*}=\frac{\sum_{j=1}^{\ell^{*}}\sigma_{i_{j}}^{-2}\cdot e_{i_{j}}e_{i_{j}}^{\top}}{\sum_{j=1}^{\ell^{*}}(s_{i_{j}}\sigma_{i_{j}}^{2})^{-1}}.
\]
That is, the matrix is diagonal, and the $k$th term on the diagonal
is $\Sigma_{\boldsymbol{f}}^{*}(k,k)\propto\sigma_{k}^{-2}$ if $k=i_{j}$
for some $j\in[\ell^{*}]$ and $\Sigma_{\boldsymbol{f}}^{*}(k,k)=0$
otherwise. As in Example \ref{exa: Identity of response weights},
$\Sigma_{\boldsymbol{f}}^{*}$ equalizes the first $\ell^{*}$ eigenvalues
of $\Sigma_{\boldsymbol{x}}\Sigma_{\boldsymbol{f}}$ (and nulls the
other $d-\ell^{*}$). However, it does so in a manner that chooses
them according to their influence on the response $\boldsymbol{f}^{\top}\boldsymbol{x}$.
The minimax representation in mixed strategies is 
\[
\boldsymbol{R}^{*}=\left[\sigma_{i_{j,1}}^{-1}\cdot e_{i_{j,1}},\ldots,\sigma_{i_{j,r}}^{-1}\cdot e_{i_{j,r}}\right]
\]
with probability $p_{j}$. Again, the first $\ell^{*}$ coordinates
are used, and not just the top $r$. See the right panel of Figure
\ref{fig: linear MSE regret examples} for a numerical example. Naturally,
in the non-diagonal case, the minimax regret will also depend on the
relative alignment between $S$ and $\Sigma_{\boldsymbol{x}}$.
\end{example}

\section{An algorithm for general classes and loss functions \label{sec:An-iterative-algorithm}}

In this section, we propose an iterative algorithm for optimizing
the representation in mixed strategies, i.e., solving \eqref{eq: minimax regret mixed}
for general classes and loss functions. The algorithm will find a
\emph{finite} mixture of $m$ representations, and so we let $\boldsymbol{\rep}=\rep^{(j)}\in{\cal R}$
with probability $p^{(j)}$ for $j\in[m]$ (this suffices for the
linear MSE setting of Section \ref{sec:The-linear-setting}, but only
approximate solution in general). The main idea is to \emph{incrementally}
add representations $\rep^{(j)}$ to the mixture. Loosely speaking,
the algorithm is initialized with a single representation rule $\rep^{(1)}$.
Then, the response function in ${\cal F}$ which is most poorly predicted
when $\boldsymbol{x}$ is represented by $\rep^{(1)}(\boldsymbol{x})$
is found. The added representation component $\rep^{(2)}$ aims to
allow for accurate prediction of this function. At the next iteration,
the function in ${\cal F}$ which is most poorly predicted by a mixed
representation that randomly uses either $\rep^{(1)}$ or $\rep^{(2)}$
is found, and $\rep^{(3)}$ is then optimized to reduce the regret
for this function, and so on. The idea of optimizing mixture models
by gradually adding components to the mixture is common, and is used
in boosting and generative adversarial network (GANs) (See Appendix
\ref{sec:Additional-related-work} for a review). 

The actual algorithm is more complicated, since it also finds proper
weights $\{p^{(j)}\}_{j\in[m]}$, and also randomizes the response
player. Thus we set $\boldsymbol{f}=f^{(i)}\in{\cal F}$ with probability
$o^{(i)}$ where $i\in[\overline{m}]$, and $\overline{m}=m_{0}+m$
for some $m_{0}\geq0$. The resulting optimization problem then becomes
\begin{equation}
\min_{\{p^{(j)},\rep^{(j)}\in{\cal R}\}}\max_{\{o^{(i)},f^{(i)}\in{\cal F}\}}\sum_{j\in[m]}\sum_{i\in[\overline{m}]}p^{(j)}\cdot o^{(i)}\cdot\regret(\rep^{(j)},f^{(i)}\mid P_{\boldsymbol{x}}),\label{eq: mixed optimization for online alg}
\end{equation}
where $\{p^{(j)}\}_{j\in[m]}$ and $\{o^{(i)}\}_{i\in[\overline{m}]}$
are probability vectors. Note that the prediction rule $\pre^{(j,i)}$
is determined based on both $\rep^{(j)}$ and $f^{(i)}$, and that
the ultimate goal of solving \eqref{eq: mixed optimization for online alg}
is just to extract the optimal $\boldsymbol{\rep}$. 

Henceforth, the index $k\in[m]$ will denote the current number of
representations. Initialization requires a representation $\rep^{(1)}$,
as well as a\emph{ set }of functions $\{f^{(i)}\}_{i\in m_{0}}$,
so that the final support size of $\boldsymbol{f}$ will be $\overline{m}=m_{0}+m$.
Finding this initial representation and the set of functions is based
on the specific loss function and the set of representation/predictors
(see Appendix \ref{sec:Iterative-algorithms-for} for examples). The
algorithm has two phases for each iteration. In the first phase, a
new adversarial function is added to the set of functions, as the
worse function for the current random representation. In the second
phase, a new representation atom is added to the set of possible representations.
This representation is determined based on the given set of functions.
Concretely, at iteration $k$:
\begin{itemize}
\item Phase 1 (Finding adversarial function): Given $k$ representations
$\{\rep^{(j)}\}_{j\in(k)}$ with weights $\{p^{(j)}\}_{j\in[k]}$,
the algorithm determines $f^{(m_{0}+k)}$ as the worst function, by
solving
\begin{equation}
\text{reg}_{k}:=\regret_{\mix}(\{\rep^{(j)},p^{(j)}\}_{j\in[k]},{\cal F}\mid P_{\boldsymbol{x}}):=\max_{f\in{\cal F}}\sum_{j\in[k]}p^{(j)}\cdot\regret(\rep^{(j)},f\mid P_{\boldsymbol{x}}),\label{eq: phase 1 iterative alg}
\end{equation}
and $f^{(m_{0}+k)}$ is set to be the maximizer. This simplifies \eqref{eq: mixed optimization for online alg}
in the sense that $m$ is replaced by $k$, the random representation
$\boldsymbol{\rep}$ is kept fixed, and $f\in{\cal F}$ is optimized
as a pure strategy (the previous functions $\{f^{(i)}\}_{i\in[m_{0}+k-1]}$
are ignored).
\item Phase 2 (Adding a representation atom): Given fixed $\{f^{(j)}\}_{j\in[m_{0}+k]}$
and $\{\rep^{(j)}\}_{j\in[k]}$, a new representation $\rep^{(k+1)}$
is found as the most incrementally valuable representation atom, by
solving
\begin{align}
 & \min_{\rep^{(k+1)}\in{\cal R}}\regret_{\mix}(\{\rep^{(j_{1})}\},\{f^{(j_{2})}\}\mid P_{\boldsymbol{x}}):=\nonumber \\
 & \min_{\rep^{(k+1)}\in{\cal R}}\min_{\{p^{(j_{1})}\}}\max_{\{o^{(j_{2})}\}}\sum_{j_{1}}\sum_{j_{2}}p^{(j_{1})}\cdot o^{(j_{2})}\cdot\regret(\rep^{(j_{1})},f^{(j_{2})}\mid P_{\boldsymbol{x}})\label{eq: phase 2 iterative alg}
\end{align}
where $j_{1}\in[k+1]$ and $j_{2}\in[m_{0}+k]$, the solution $\rep^{(k+1)}$
is added to the representations, and the weights are updated to the
optimal $\{p^{(j_{1})}\}$. Compared to \eqref{eq: mixed optimization for online alg},
the response functions and first $k$ representations are kept fixed,
and only the weights $\{p^{(j_{1})}\}$ $\{o^{(j_{2})}\}$ and $\rep^{(k+1)}$
are optimized. 
\end{itemize}
The procedure is described in Algorithm \ref{alg: Iterative algorithm},
where, following the main loop, $m^{*}=\argmin_{k\in[m]}\text{reg}_{k}$
representation atoms are chosen and the output is $\{\rep^{(j)},p^{(j)}\}_{j\in[m^{*}]}$.
Algorithm \ref{alg: Iterative algorithm} relies on solvers for the
Phase 1 \eqref{eq: phase 1 iterative alg} and Phase 2 \eqref{eq: phase 2 iterative alg}
problems. In Appendix \ref{sec:Iterative-algorithms-for} we propose
two algorithms for these problems, which are based on gradient steps
for updating the adversarial response (assuming ${\cal F}$ is a continuous
class) and the new representation, and on the MWU\emph{ }algorithm
\citep{freund1999adaptive} for updating the weights. In short, the
Phase 1 algorithm updates the response function $f$ via a projected
gradient step of the expected loss, and then adjusts the predictors
$\{\pre^{(j)}\}$ to the updated response function $f$ and the current
representations $\{\rep^{(j)}\}_{j\in[k]}$. The Phase 2 algorithm
only updates the new representation $R^{(k+1)}$ via projected gradient
steps, while keeping $\{R^{(j)}\}_{j\in[k]}$ fixed. Given the representations
$\{\rep^{(j)}\}_{j\in[k+1]}$ and the functions $\{f^{(i)}\}_{i\in[m_{0}+k]}$,
a predictor $\pre^{(j,i)}$ is then fitted to each representation-function
pair, which also determines the loss for this pair. The weights $\{p^{(j)}\}_{j\in[k+1]}$
and $\{o^{(i)}\}_{i\in[m_{0}+k]}$ are updated towards the equilibrium
of the two-player game determined by the loss of the predictors $\{\pre^{(j,i)}\}_{j\in[k+1],i\in[m_{0}+k]}$
via the MWU\emph{ }algorithm. In a multi-task learning setting, in
which the class ${\cal F}$ is finite, the Phase 1 solver can be replaced
by a performing a simple maximization over the functions. 
\begin{algorithm*}
\begin{algorithmic}[1]

\State \textbf{input }$P_{\boldsymbol{x}},{\cal R},{\cal F},{\cal Q},d,r,m,m_{0}$\Comment{Feature
distribution, classes, dimensions and parameters}

\State \textbf{input }$\rep^{(1)}$ $,\{f^{(j)}\}_{j\in[m_{0}]}$\Comment{Initial
representation and initial function (set)}

\State \textbf{begin }

\For{$k=1$ to $m$}\textbf{ }

\State \textbf{phase 1:} $f^{(m_{0}+k)}$ is set by a solver of \eqref{eq: phase 1 iterative alg}
and: \Comment{Solved using Algorithm \ref{alg: phase 1 sol} }
\[
\text{reg}_{k}\leftarrow\regret_{\mix}(\{\rep^{(j)},p^{(j)}\}_{j\in[k]},{\cal F}\mid P_{\boldsymbol{x}})
\]

\State \textbf{phase 2:} $\rep^{(k+1)},\{p_{k}^{(j)}\}_{j\in[k+1]}$
is set by a solver of \eqref{eq: phase 2 iterative alg}\Comment{Solved
using Algorithm \ref{alg: phase 2 sol}}

\EndFor  

\State \textbf{set} $m^{*}=\argmin_{k\in[m]}\text{reg}_{k}$

\State \textbf{return $\{\rep^{(j)}\}_{j\in[m^{*}]}$ and $p_{m_{*}}=\{p_{k}^{(j)}\}_{j\in[m^{*}]}$}

\end{algorithmic}

\caption{Solver of \eqref{eq: mixed optimization for online alg}: An iterative
algorithm for learning mixed representations. \label{alg: Iterative algorithm}}
\end{algorithm*}

Since Algorithm \ref{alg: Iterative algorithm} aims to solve a non
convex-concave minimax game, deriving theoretical bounds on its convergence
seems to be elusive at this point (see Appendix \ref{sec:Additional-related-work}
for a discussion). We next describe a few experiments with the algorithm
(See Appendix \ref{sec:Details-for-the} for details). 
\begin{example}
\label{exa: Algorithm linear MSE}We validated Algorithm \ref{alg: Iterative algorithm}
in the linear MSE setting (Section \ref{sec:The-linear-setting}),
for which a closed-form solution exists. We ran Algorithm \ref{alg: Iterative algorithm}
on randomly drawn diagonal $\Sigma_{\boldsymbol{x}}$, and computed
the ratio between the regret obtained by the algorithm to the theoretical
value. The left panel of Figure \ref{fig: theory-alg} shows that
the ratio is between $1.15-1.2$ in a wide range of $d$ values. Algorithm
\ref{alg: Iterative algorithm} is also useful for this setting since
finding an $(\ell^{*}+1)$-sparse solution to $\overline{A}p=\overline{b}$
is computationally difficult when ${\ell^{*} \choose r}$ is very
large (e.g., in the largest dimension of the experiment ${d \choose r}={19 \choose 5}=11,628$). 
\end{example}

Our next example pertains to a logistic regression setting, under
the cross-entropy loss function.
\begin{defn}[The linear cross-entropy setting]
\label{def: linear logistic log loss }Assume that ${\cal X}=\mathbb{R}^{d}$,
that ${\cal Y}=\{\pm1\}$ and that $\E[\boldsymbol{x}]=0$. Assume
that the class of representation is linear $z=\rep(x)=R^{\top}x$
for some $R\in{\cal R}:=\mathbb{R}^{d\times r}$ where $d>r$. Assume
that a response function and a prediction rule determine the probability
that $y=1$ via logistic regression modeling, as $f(\boldsymbol{y}=\pm1\mid x)=1/[1+\exp(\mp f^{\top}x)]$.
Assume the cross-entropy loss function, where given that the prediction
that $\boldsymbol{y}=1$ with probability $q$ results the loss $\loss(y,q)\dfn-\frac{1}{2}(1+y)\log q-\frac{1}{2}\left(1-y\right)\log(1-q)$.
The set of predictor functions is ${\cal Q}\dfn\left\{ \pre(z)=1/[1+\exp(-q^{\top}\boldsymbol{z})],\;q\in\mathbb{R}^{r}\right\} $.
As for the linear case, we assume that $f\in{\cal F}_{S}$ for some
$S\in\mathbb{S}_{++}^{d}$. The regret is then given by the expected
binary Kullback-Leibler (KL) divergence
\begin{align}
\regret(\rep,f\mid P_{\boldsymbol{x}}) & =\min_{q\in\mathbb{R}^{r}}\E\left[D_{\text{KL}}\left([1+\exp(-f^{\top}\boldsymbol{x})]^{-1}\mid\mid[1+\exp(-q^{\top}R^{\top}\boldsymbol{x})]^{-1}\right)\right].
\end{align}
\end{defn}

\begin{example}
\label{exa: Algorithm linear logistic}We drawn empirical distributions
of features from an isotropic normal distribution, in the linear cross-entropy
setting. We ran Algorithm \ref{alg: Iterative algorithm} using the
closed-form regret gradients from Appendix \ref{sec:Details-for-the}.
The right panel of Figure \ref{fig: theory-alg} shows the reduced
regret obtained by increasing the support size $m$ of the random
representation, and thus the effectiveness of mixed representations.
\end{example}

The last example compares the optimized representation with that of
PCA. 
\begin{example}[Comparison with PCA for multi-label Classification]
\label{exa: images with 4 shapes}We constructed a dataset of images,
each containing $4$ shapes randomly selected from a dictionary of
$6$ shapes. The class ${\cal F}$ is finite and contains the $t=6$
binary classification functions given by indicators for each shape.
The representation is linear and the predictor is based on logistic
regression. In this setting, Algorithm \ref{alg: Iterative algorithm}
can be simplified; See Appendix \ref{subsec:An-experiment-with} for
details (specifically Definition \ref{def:muli-label classification setting}
of the setting and Figure \ref{fig:image with shapes} for an image
example). We ran the simplified version of Algorithm \ref{alg: Iterative algorithm}
on a dataset of $1000$ images, and compared the cross-entropy loss
and the accuracy of optimized representation to that of PCA on a fresh
dataset of $1000$ images. The results in Figure \ref{fig: Optimized representation vs PCA}
show that the regret of PCA is much larger, not only for the worse-case
function but also for the average-case function. For example, using
the representation obtained by the algorithm with $r=3<t=6$ is as
good as PCA with at least $r=12$. From a different point of view,
in order to achieve the loss of the optimized representation output
by our algorithm with $r=1$ more than $r=15$ dimensions are required
for PCA. 
\end{example}

\begin{figure}
\begin{centering}
\includegraphics[scale=0.5]{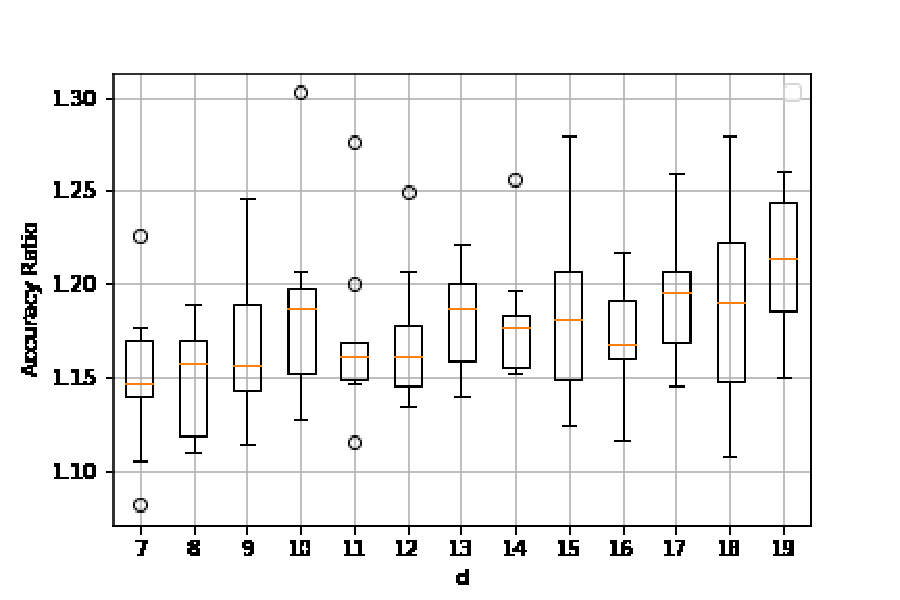}\includegraphics[scale=0.5]{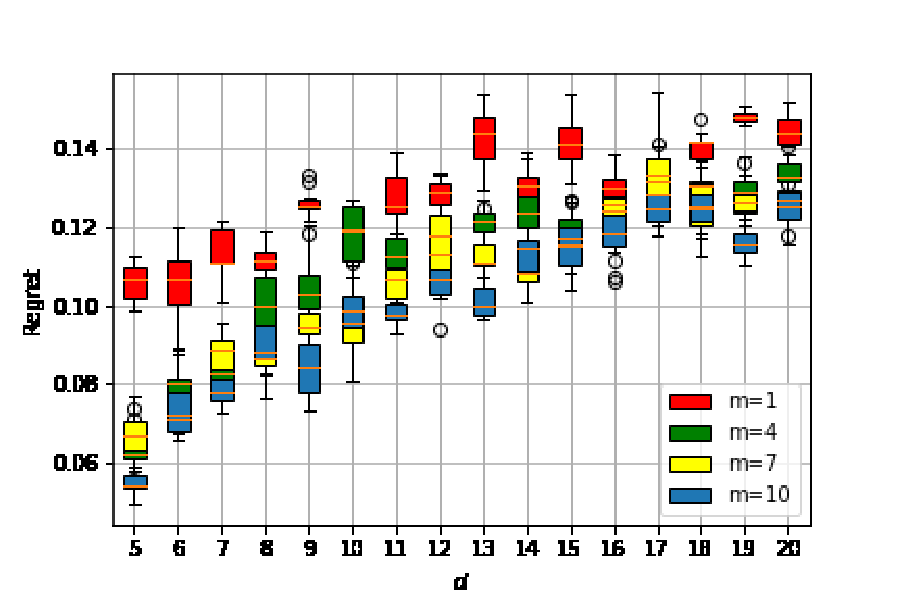}
\par\end{centering}
\caption{Results of Algorithm \ref{alg: Iterative algorithm}. Left: $r=5$,
varying $d$. The ratio between the regret achieved by Algorithm \ref{alg: Iterative algorithm}
and the theoretical regret in the linear MSE setting. Right: $r=3$,
varying $d$. The regret achieved by Algorithm \ref{alg: Iterative algorithm}
in the linear cross-entropy setting, various $m$. \label{fig: theory-alg}}
\end{figure}

\begin{figure}
\centering{}\includegraphics[scale=0.5]{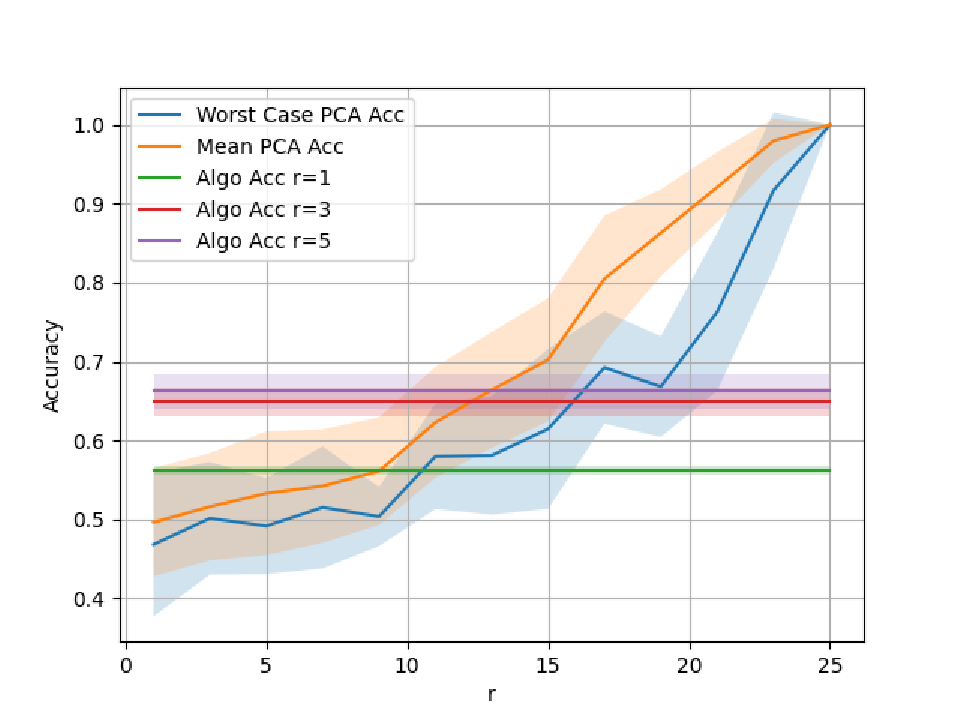}\includegraphics[scale=0.5]{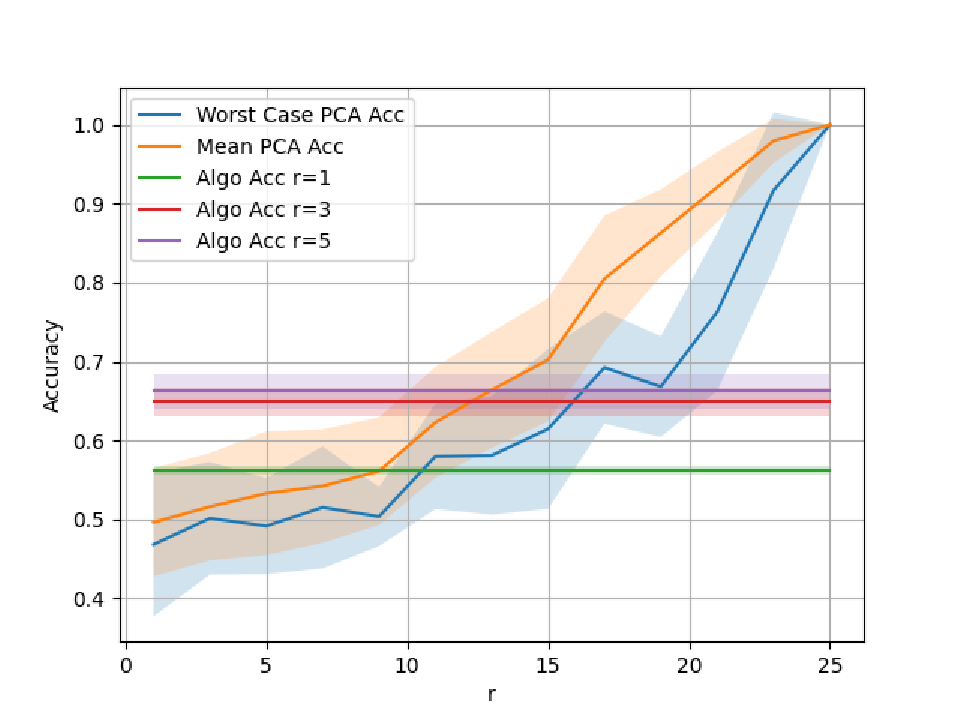}\caption{Results on the dataset of images. Comparison between optimized minimax
representation (simplified version of Algorithm \ref{alg: Iterative algorithm})
vs. PCA. Worst-case function in blue, and average-case function in
orange. Left: Cross entropy loss. Right: Accuracy. \label{fig: Optimized representation vs PCA}}
\end{figure}

\section{Conclusion \label{sec:Conclusion}}

We proposed a game-theoretic formulation for learning feature representations
when prior knowledge on the class of downstream prediction tasks is
available. We derived the optimal solution for the fundamental linear
MSE setting, which quantify the gain of prior knowledge and the possibility
to randomize the representation. We then proposed an iterative algorithm
suitable for general classes of functions and losses, and exemplified
its effectiveness. Interestingly, our formulation links between the
problem of finding jointly efficient representations for multiple
prediction tasks, and the problem of finding saddle points of non-convex/non-concave
games. The latter is a challenging problem and is under active research
(see Appendix \ref{sec:Additional-related-work}). So, any advances
in that domain can be translated to improve representation learning.
Similarly to our problem, foundation models also adapt to wide range
of downstream prediction tasks, however, efficient learning is achieved
therein without explicit prior knowledge on the class of downstream
tasks, albeit using fine-tuning. It is interesting to incorporate
minimax formulations into the training procedure of foundation models,
or to explain them using similar game formulations. 

For future research it would be interesting: (1) To generalize the
class ${\cal F}_{S}=\{f\colon\|f\|_{S}\leq1\}$ used for linear functions
to the class ${\cal F}_{S_{x}}\dfn\E\left[\|\nabla_{x}f(\boldsymbol{x})\|_{S_{\boldsymbol{x}}}^{2}\right]\leq1$
for non-linear functions, where $\{S_{x}\}_{x\in\mathbb{R}^{d}}$
is now locally specified (somewhat similarly to the regularization
term used in contractive AE \citep{rifai2011contractive}, though
for different reasons). (2) To efficiently learn $S$ from previous
experience, e.g., improving $S$ from one episode to another in a
meta-learning setup \citep{hospedales2021meta}. (3) To evaluate the
effectiveness of the learned representation in our formulation, as
an initialization for further optimization when labeled data is collected.
One may postulate that our learned representation may serve as a \emph{universal}
initialization for such training.

\appendix
The outline of the appendix follows. In Appendix \ref{sec:Classes-of-response}
we discuss in detail the origin of prior knowledge of classes of prediction
problems. In Appendix \ref{sec:Additional-related-work} we review
additional related work. In Appendix \ref{sec:Notation-conventions}
we set our notation conventions. In Appendix \ref{sec:Useful-mathematical-results}
we summarize a few mathematical results that are used in later proofs.
In Section \ref{sec:The-linear-MSE} we show that PCA can be cast
as a degenerate setting of our formulation, and provide the proofs
of the main theorems in the paper (the linear MSE setting). In Appendix
\ref{sec:The-Hilbert-space} we generalize these results to an infinite
dimensional Hilbert space. In Appendix \ref{sec:Iterative-algorithms-for}
we provide two algorithms for solving the Phase 1 and Phase 2 problems
in Algorithm \ref{alg: Iterative algorithm}. In Appendix \ref{sec:Details-for-the}
we provide details on the examples for the experiments with the iterative
algorithm, and additional experiments.

\section{Classes of response functions \label{sec:Classes-of-response}}

Our approach to optimal representation is based on the assumption
that a class ${\cal F}$ of downstream prediction tasks is known.
This assumption may represent prior knowledge or constraints on the
response function, and can stem from various considerations. To begin,
it might be hypothesized that some features are less relevant than
others. As a simple intuitive example, the outer pixels in images
are typically less relevant to the classification of photographed
objects, regardless of their variability (which may stem from other
affects, such as lighting conditions). Similarly, non-coding regions
of the genotype are irrelevant for predicting phenotype. The prior
knowledge may encode softer variations in relevance. Moreover, such
prior assumption may be imposed on the learned function, e.g., it
may be assumed that the response function respects the privacy of
some features, or only weakly depends on features which provide an
unfair advantage. In domain adaptation \citep{mansour2009domain},
one may solve the prediction problem for feature distribution $P_{\boldsymbol{x}}$
obtaining a optimal response function $f_{1}$. Then, after a change
of input distribution to $Q_{\boldsymbol{x}}$, the response function
learned for this feature distribution $f_{2}$ may be assumed to belong
to functions which are ``compatible'' with $f_{1}$. For example,
if $P_{\boldsymbol{x}}$ and $Q_{\boldsymbol{x}}$ are supported on
different subsets of $\mathbb{R}^{d}$, the learned response function
$f_{1}(x)$ and $f_{2}(x)$ may be assumed to satisfy some type of
continuity assumptions. Similar assumptions may hold for the more
general setting of transfer learning \citep{ben2010theory}. Furthermore,
such assumptions may hold in a \emph{continual learning} setting \citep{zenke2017continual,nguyen2017variational,van2019three,aljundi2019task},
in which a sequence of response functions is learned one task at a
time. Assuming that \emph{catastrophic forgetting} is aimed to be
avoided, then starting from the second task, the choice of representation
may assume that the learned response function is accurate for all
previously learned tasks. 

\section{Additional related work \label{sec:Additional-related-work}}

\paragraph*{The information bottleneck principle}

The IB principle is a prominent approach to feature relevance in the
design of representations \citep{tishby2000information,chechik2003information,slonim2006multivariate,harremoes2007information},
and proposes to optimize the representation in order to maximize its
relevance to the response $\boldsymbol{y}$. Letting $I(\boldsymbol{z};\boldsymbol{y})$
and $I(\boldsymbol{x};\boldsymbol{\boldsymbol{z}})$ denote the corresponding
mutual information terms \citep{Cover:2006:EIT:1146355}, the IB principle
aims to maximize the former while constraining the latter from above,
and this is typically achieved via a Lagrangian formulation \citep{boyd2004convex}.
The resulting representation, however, is tailored to the joint distribution
of $(\boldsymbol{x},\boldsymbol{y})$, i.e., to a specific prediction
task. In practice, this is achieved using a labeled dataset (generalization
bounds were derived by \citet{shamir2010learning}). As in our mixed
representation approach, the use of randomized representation dictated
by a probability kernel $\boldsymbol{z}\sim f(\cdot\mid\boldsymbol{x}=x)$
is inherent to the IB principle. 

A general observation made from exploring deep neural networks (DNNs)
\citep{goodfellow2016deep} used for classification, is that practically
good predictors in fact first learn an efficient representation $\boldsymbol{z}$
of the features, and then just train a simple predictor from $\boldsymbol{z}$
to the response $\boldsymbol{y}=f(\boldsymbol{x})$. This is claimed
to be quantified by the IB where low mutual information $I(\boldsymbol{x};\boldsymbol{z})$
indicates an efficient representation, and high mutual information
$I(\boldsymbol{z};\boldsymbol{y})$ indicates that the representation
allows for efficient prediction. This idea has been further developed
using the IB principle, by hypothesizing that modern prediction algorithms
must intrinsically include learning an efficient representation \citep{tishby2015deep,shwartz2017opening,shwartz2022information,achille2018emergence,achille2018information}
(this spurred a debate, see, e.g., \citep{saxe2019information,geiger2021information}). 

However, this approach is inadequate in our setting, since the prediction
task is not completely specified to the learner, and in the IB formulation
the optimal representation depends on the response variable (so that
labeled data should be provided to the learner). In addition, as explained
by \citet{dubois2020learning}, while the resulting IB solution provides
a fundamental limit for the problem, it also suffers from multiple
theoretical and practical issues. The first main issue is that the
mutual information terms are inherently difficult to estimate from
finite samples \citep{shamir2010learning,nguyen2010estimating,poole2018variational,wu2020learnability,mcallester2020formal},
especially at high dimensions, and thus require resorting to approximations,
e.g., variational bounds \citep{chalk2016relevant,alemi2016deep,belghazi2018mutual,razeghi2022bottlenecks}.
The resulting generalization bounds \citep{shamir2010learning,vera2018role}
are still vacuous for modern settings \citep{rodriguez2019information}.
The second main issue is that the IB formulation does not constrain
the complexity of the representation and the prediction rule, which
can be arbitrarily complex. These issues were addressed by \citet{dubois2020learning}
using the notion of \emph{usable information}, previously introduced
by \citet{xu2020theory}: The standard mutual information $I(\boldsymbol{z};\boldsymbol{y})$
can be described as the log-loss difference between a predictor for
$\boldsymbol{z}$ which does not use or does use $\boldsymbol{y}$
(or vice-versa, since mutual information is symmetric). Usable information,
or ${\cal F}$-information $I_{{\cal F}}(\boldsymbol{z}\to\boldsymbol{y})$,
restricts the predictor to a class ${\cal F}$, which is computationally
constrained. Several desirable properties were established in \citet{xu2020theory}
for the ${\cal F}$-information, e.g., probably approximate correct
(PAC) bounds via Rademacher-complexity based bounds \citep{bartlett2002rademacher}
\citep[Chapter 5]{wainwright2019high}\citep[Chapters 26-28]{shalev2014understanding}. 

\citet{dubois2020learning} used the notion of ${\cal F}$-information
to define the \emph{decodable IB }problem, with the goal of assessing
the generalization capabilities of this IB problem, and shedding light
on the necessity of efficient representation for generalization. To
this end, a game was proposed, in which the data available to the
learner are feature-response pairs $(\boldsymbol{x},\boldsymbol{y})$
(in our notation $\boldsymbol{y}=f(\boldsymbol{x})$). In the game
proposed by \citet{dubois2020learning}, Alice chooses a prediction
problem $f(\cdot)$, i.e., a feature-response pair $(\boldsymbol{x};\boldsymbol{y})$,
and Bob chooses a representation $\boldsymbol{z}=\rep(\boldsymbol{x})$.
For comparison, in this paper, we assume that the learner is given
features $\boldsymbol{x}$ and a \emph{class }${\cal F}$ of prediction
problems. We ask how to choose the representation $\rep$ if it is
only known that the response function $Y=f(X)$ can be chosen adversarially
from ${\cal F}$. In our formulated game, the order of plays is thus
different. First, the representation player chooses $\rep$, and then
the adversarial function player chooses $f\in{\cal F}$. Beyond those
works, the IB framework has drawn a significant recent attention,
and a remarkable number of extensions and ramifications have been
proposed \citep{sridharan2008information,amjad2019learning,kolchinsky2019nonlinear,strouse2019information,pensia2020extracting,asoodeh2020bottleneck,ngampruetikorn2021perturbation,yu2021deep,ngampruetikorn2022information,gunduz2022beyond,razeghi2022bottlenecks,ngampruetikorn2023generalized}.
An IB framework for self-supervised learning was recently discussed
in \citep{ngampruetikorn2020information}. 

\paragraph*{Multitask learning and learning-to-learn}

In the problem of multitask learning and learning-to-learn \citep{baxter2000model,argyriou2006mulfti,maurer2016benefit,du2020few,tripuraneni2020theory,tripuraneni2021provable},
the goal of the learner is to jointly learn multiple downstream prediction
tasks. The underlying implicit assumption is that the tasks are similar
in some way, specifically, that they share a common low dimensional
representation. In the notation of this paper, a class ${\cal F}=\{f_{1},f_{2},\cdots,f_{t}\}$
of $t$ prediction task is given, and a predictor from $\rep\colon\mathbb{R}^{d}\to{\cal Y}$
is decomposed as $\pre_{i}(\rep(\boldsymbol{x}))$, where the representation
$\boldsymbol{z}=\rep(\boldsymbol{x})\in\mathbb{R}^{r}$ is common
to all tasks. The learner is given a dataset for each of the tasks,
and its goal is to learn the common representation, as well as the
$t$ individual predictors in the multitask setting. In the learning-to-learn
setting, the learner will be presented with a new prediction tasks,
and so only the representation is retained. 

\citet{maurer2016benefit} assumed that the representation is chosen
from a class ${\cal R}$ and the predictors from a class ${\cal Q}$
(in our notation), and generalization bounds on the average excess
risk for an empirical risk minimization (ERM) learning algorithm were
derived, highlighting the different complexity measures associated
with the problem. In the context of learning-to-learn, the bound in
\citep[Theorem 2]{maurer2016benefit} scales as $O(1/\sqrt{t})+O(1/\sqrt{m})$,
where $m$ is the number of samples provided to the learner from the
new task. In practice, it is often the case that learning-to-learn
is efficient even for small $t$, which implies that this bound is
loose. It was then improved by \citet{tripuraneni2020theory}, whose
main statement is that a proper notation of task diversity is crucial
for generalization. They then propose an explicit notion of task diversity,
controlled by a term $\nu$ and obtained a bound of the form $\tilde{O}(\frac{1}{\nu}\left(\sqrt{\frac{C({\cal R})+tC({\cal Q})}{nt}}+\sqrt{\frac{C({\cal Q})}{m}}\right))$
(in our notation), where $C({\cal R})$ and $C({\cal Q})$ are complexity
measures for the representation class and the prediction class. For
comparison, in this paper we focus on finding the optimal representation,
either theoretically (in the fundamental linear MSE setting) or algorithmically,
rather than relying on a generic ERM. To this end we side-step the
generalization error, and so the regret we define can be thought of
as an \emph{approximation error}. One direct consequence of this difference
is that in our case task diversity (rich ${\cal F}$) leads to a large
regret, whereas in \citep{tripuraneni2020theory} task diversity leads
to low generalization bound. From this aspect, our results complement
those of \citet{tripuraneni2020theory} for the non-realizable case
(that is, when the prediction tasks cannot be decomposed as a composition
$f(\boldsymbol{x})=\pre(\rep(\boldsymbol{x}))$.

\paragraph*{Randomization in representation learning }

Randomization is classically used in data representation, most notably,
utilizing the seminal Johnson-Lindenstrauss Lemma \citet{johnson1984extensions}
or more generally, \emph{sketching} algorithms (e.g., \citep{vempala2005random,mahoney2011randomized,woodruff2014sketching,yang2021reduce}).
Our use of randomization is different and is inspired by the classical
Nash equilibrium \citep{nash1950equilibrium}. Rather than using a
single deterministic representation that was randomly chosen, we consider
randomizing multiple representation rules. Such randomization is commonly
used in the face of future uncertainty, which in our setting is the
downstream prediction task. 

\paragraph*{Game-theoretic formulations in statistics and machine-learning }

The use of game-theoretic formulations in statistics, between a player
choosing a prediction algorithm and an adversary choosing a prediction
problem (typically Nature), was established by \citet{wald1939contributions}
in his classical statistical decision theory (see, e.g., \citep[Chapter 12]{wasserman2004all}).
It is a common approach both in classic statistics and learning theory
\citep{yang1999information,Grunwald2004Game,haussler1997mutual,farnia2016minimax},
as well as in modern high-dimensional statistics \citep{wainwright2019high}.
The effect of the representation (quantizer) on the consistency of
learning algorithms when a surrogate convex loss function replaces
the loss function of interest was studied in \citep{Nguyen2009OnSurrogate,Duchi2018Multiclass,Grunwald2004Game}
(for binary and multiclass classification, respectively). A relation
between information loss and minimal error probability was recently
derived by \citet{silva2022interplay}. 

Iterative algorithms for the solution of minimax games have drawn
much attention in the last few years due to their importance in optimizing
GANs \citep{goodfellow2020generative,creswell2018generative}, adversarial
training \citep{madry2017towards}, and robust optimization \citep{ben2009robust}.
The notion of convergence is rather delicate, even for the basic convex-concave
two-player setting \citep{salimans2016improved}. While the value
output by the MWU algorithm \citep{freund1999adaptive}, or improved
versions \citep{daskalakis2011near,rakhlin2013optimization} converges
to a no-regret solution, the actual strategies used by the players
are, in fact, repelled away from the equilibrium point to the boundary
of the probability simplex \citep{bailey2018multiplicative}. For
general games, the gradient descent ascent (GDA) is a natural and
practical choice, yet despite recent advances, its theory is still
partial \citep{zhang2022near}. Various other algorithms have been
proposed, e.g., \citep{schafer2019competitive,mescheder2017numerics,letcher2019differentiable,gidel2019negative,zhang2021suboptimality}. 

\paragraph*{Incremental learning of mixture models }

Our proposed Algorithm \ref{alg: Iterative algorithm} operates iteratively,
and each main iteration adds a representation rule as a new component
to the existing mixture of representation rules. More broadly, the
efficiency of algorithms that follow this idea is based on two principles:
(i) A powerful model can be obtained from a mixture of a few weak
models; (ii) Mixture models can be efficiently learned by a gradual
addition of components to the mixture, if the new component aims to
address the most challenging problem instance. As a classic example,
this idea is instantiated by the boosting method \citep{Schapire2012Boosting}
for classification, and was adapted for generative models by \citet{tolstikhin2017adagan}
for GANs. In boosting, the final classifier is a mixture of simpler
classifiers. Large weights are put on data points which are wrongly
classified with the current mixture of classifiers, and the new component
(classifier) is trained to cope with these samples. In GANs, the generated
distribution is a mixture is of generative models. Large weights are
put on examples which are easily discerned by the discriminator of
the true and generated distributions, and the new component (generative
distribution) optimizes the GAN objective on this weighted data. In
our setting, the final representation is a mixture of representation
rules. Weights are put on adversarial functions that cannot be accurately
predicted with the current representation matrices. The new representation
component aims to allow for accurate prediction of these functions.
Overall, the common intuitive idea is very natural: The learning algorithm
sees what is most lacking in the current mixture, and adds a new component
that directly aims to minimize this shortage. We refer the reader
to \citep{tolstikhin2017adagan} for a more in-depth exposition of
this idea. As mentioned by \citet{tolstikhin2017adagan}, this idea
dates back to the use of boosting for density estimation \citep{welling2002self}. 

\paragraph*{Unsupervised pretraining }

From a broader perspective, our method is essentially an \emph{unsupervised
pretraining} method, similar to the methods which currently enable
the recent success in natural language processing. Our model is much
simplified compared to transformer architecture \citep{vaswani2017attention},
but the unsupervised training aspect used for prediction tasks \citep{devlin2018bert}
is common, and our results may shed light on these methods. For example,
putting more weight on some words compared to others during training
phase that uses the masked-token prediction objective. 

\section{Notation conventions \label{sec:Notation-conventions}}

For an integer $d$, $[d]\dfn\{1,2,\ldots,d\}$. For $p\geq1$, $\|x\|_{p}\dfn(\sum_{i=1}^{d}|x_{i}|^{p})^{1/p}$
is the $\ell_{p}$ norm of $x\in\mathbb{R}^{d}$. The Frobenius norm
of the matrix $A$ is denoted by $\|A\|_{F}=\sqrt{\Tr[A^{T}A]}$ .
The non-negative (resp. positive) definite cone of symmetric matrices
is given by $\mathbb{S}_{+}^{d}$ (resp. $\mathbb{S}_{++}^{d}$).
For a given positive-definite matrix $S\in\mathbb{S}_{++}^{d}$, the
Mahalanobis norm of $x\in\mathbb{R}^{d}$ is given by $\|x\|_{S}\dfn\|S^{-1/2}x\|_{2}=(x^{\top}S^{-1}x)^{1/2}$,
where $S^{1/2}$ is the symmetric square root of $S$. The matrix
$W\dfn[w_{1},\ldots,w_{r}]\in\mathbb{R}^{d\times r}$ is comprised
from the column vectors $\{w_{i}\}_{i\in[r]}\subset\mathbb{R}^{d}$.
For a real symmetric matrix $S\in\mathbb{S}^{d}$, $\lambda_{i}(S)$
is the $i$th largest eigenvalue, so that $\lambda_{\text{max}}(S)\equiv\lambda_{1}(S)\geq\lambda_{2}(S)\geq\cdots\geq\lambda_{d}(S)=\lambda_{\text{min}}(S)$,
and in accordance, $v_{i}(S)$ denote an eigenvector corresponding
to $\lambda_{i}(S)$ (these are unique if there are no two equal eigenvalues,
and otherwise arbitrarily chosen, while satisfying orthogonality $v_{i}^{\top}v_{j}=\langle v_{i},v_{j}\rangle=\delta_{ij}$).
Similarly, $\Lambda(S)\dfn\diag(\lambda_{1}(S),\lambda_{2}(S),\cdots,\lambda_{d}(S))$
and $V(S)\dfn[v_{1}(S),v_{2}(S),\cdots,v_{d}(S)]$, so that $S=V(S)\Lambda(S)V^{\top}(S)$
is an eigenvalue decomposition. For $j\geq i$, $V_{i:j}\dfn[v_{i},\ldots,v_{j}]\in\mathbb{R}^{(j-i+1)\times d}$
is the matrix comprised of the columns indexed by $\{i,\ldots,j\}$.
The vector $e_{i}\in\mathbb{R}^{d}$ is the $i$th standard basis
vector, that is, $e_{i}\dfn[\text{\ensuremath{\underbrace{0,\ldots0}_{i-1\text{ terms}}}},1,\underbrace{0,\ldots0}_{d-i\text{ terms}}]^{\top}$.
Random quantities (scalars, vectors, matrices, etc.) are denoted by
boldface letters. For example, $\boldsymbol{x}\in\mathbb{R}^{d}$
is a random vector that takes values $x\in\mathbb{R}^{d}$ and $\boldsymbol{R}\in\mathbb{R}^{d\times r}$
is a random matrix. The probability law of a random element $\boldsymbol{x}$
is denoted by $\pl(\boldsymbol{x})$. The probability of the event
${\cal E}$ in some given probability space is denoted by $\P[{\cal E}]$
(typically understood from context). The expectation operator is denoted
by $\E[\cdot]$. The indicator function is denoted by $\I\{\cdot\}$,
and the Kronecker delta is denoted by $\delta_{ij}\dfn\I\{i=j\}$.
We do not make a distinction between minimum and infimum (or maximum
and supremum) as arbitrarily accurate approximation is sufficient
for the description of the results in this paper. The binary KL divergence
between $p_{1},p_{2}\in(0,1)$ is denoted as

\[
D_{\text{KL}}(p_{1}\mid\mid p_{2})\dfn p_{1}\log\frac{p_{1}}{p_{2}}+(1-p_{1})\log\frac{1-p_{1}}{1-p_{2}}.
\]

\section{Useful mathematical results \label{sec:Useful-mathematical-results}}

In this section we provide several simplified versions of mathematical
results that are used in the proofs. The following well-known result
is about the optimal low-rank approximation to a given matrix:
\begin{thm}[{\emph{Eckart-Young-Mirsky} \citep[Example 8.1]{wainwright2019high}
\citep[Section 4.1.4]{vershynin2018high}}]
\emph{}For a symmetric matrix $S\in\mathbb{S}^{d}$
\[
\left\Vert S_{k}-S\right\Vert _{F}\leq\min_{S'\in\mathbb{S}^{d}\colon\rank(S')\leq k}\left\Vert S-S'\right\Vert _{F}
\]
where 
\[
S_{k}=\sum_{i\in[k]}\lambda_{i}(S)\cdot v_{i}(S)v_{i}^{\top}(S)
\]
(more generally, this is true for any unitarily invariant norm).
\end{thm}

We next review a simplified version of variational characterizations
of eigenvalues of symmetric matrices:
\begin{thm}[{\emph{Rayleigh quotient} \citep[Theorem 4.2.2]{horn2012matrix}}]
\emph{}For a symmetric matrix $S\in\mathbb{S}^{d}$
\[
\lambda_{1}(S)=\max_{x\neq0}\frac{x^{\top}Sx}{\|x\|_{2}^{2}}.
\]
\end{thm}

\begin{thm}[{\emph{Courant--Fisher variational characterization} \citep[Theorem 4.2.6]{horn2012matrix}}]
\emph{}For a symmetric matrix $S\in\mathbb{S}^{d}$, $k\in[d]$,
and a subspace $T$ of $\mathbb{R}^{d}$ 
\[
\lambda_{k}(S)=\min_{T\colon\dim(T)=k}\max_{x\in T\backslash\{0\}}\frac{x^{\top}Sx}{\|x\|_{2}^{2}}=\max_{T\colon\dim(T)=d-k+1}\min_{x\in T\backslash\{0\}}\frac{x^{\top}Sx}{\|x\|_{2}^{2}}.
\]
\end{thm}

\begin{thm}[{\emph{Fan's variational characterization \citep[Corollary 4.3.39.]{horn2012matrix}}}]
For a symmetric matrix $S\in\mathbb{S}^{d}$ and $k\in[d]$
\[
\lambda_{1}(S)+\cdots+\lambda_{k}(S)=\min_{U\in\mathbb{R}^{d\times k}\colon U^{\top}U=I_{k}}\Tr[U^{\top}SU]
\]
and 
\[
\lambda_{d-k+1}(S)+\cdots+\lambda_{d}(S)=\max_{U\in\mathbb{R}^{d\times k}\colon U^{\top}U=I_{k}}\Tr[U^{\top}SU].
\]
\end{thm}

We will use the following celebrated result from convex analysis.
\begin{thm}[{\emph{Carath\'{e}odory's theorem \citep[Prop. 1.3.1]{bertsekas2003convex}}}]
Let ${\cal A}\subset\mathbb{R}^{d}$ be non-empty. Then, any point
$a$ in the convex hull of ${\cal A}$ can be written as a convex
combination of at most $d+1$ points from ${\cal A}$. 
\end{thm}

\section{The linear MSE setting: additions and proofs \label{sec:The-linear-MSE}}

\subsection{The standard principal component setting \label{subsec:Standard-principle-component}}

In order to highlight the formulation proposed in this paper, we show,
as a starting point, that the well known PCA solution of representing
$\boldsymbol{x}\in\mathbb{R}^{d}$ with the top $r$ eigenvectors
of the covariance matrix of $\boldsymbol{x}$ can be obtained as a
specific case of the regret formulation. In this setting, we take
${\cal F}=\{I_{d}\}$, and so $\boldsymbol{y}=\boldsymbol{x}$ with
probability $1$. In addition, the predictor class ${\cal Q}$ is
a linear function from the representation dimension $r$ back to the
features dimension $d$. 
\begin{prop}
\label{prop: regret PCA}Consider the linear MSE setting, with the
difference that the response is $\boldsymbol{y}\in\mathbb{R}^{d}$,
the loss function is the squared Euclidean norm $\loss(y_{1},y_{2})=\|y_{1}-y_{2}\|^{2}$,
and the predictor is $\pre(z)=Q^{\top}z\in\mathbb{R}^{d}$ for  $Q\in\mathbb{R}^{r\times d}$.
Assume ${\cal F}=\{I_{d}\}$ so that $\boldsymbol{y}=\boldsymbol{x}$
with probability $1$. Then, 
\[
\regret_{\pure}({\cal F}\mid\Sigma_{\boldsymbol{x}})=\regret_{\mix}({\cal F}\mid\Sigma_{\boldsymbol{x}})=\sum_{i=r+1}^{d}\lambda_{i}(\Sigma_{\boldsymbol{x}}),
\]
and an optimal representation is $R=V_{1:r}(\Sigma_{\boldsymbol{x}})$.
\end{prop}

The result of Proposition \ref{prop: regret PCA} verifies that the
minimax and maximin formulations indeed generalize the standard PCA
formulation. The proof is standard and follows from the \emph{Eckart-Young-Mirsky
theorem}, which determines the best rank $r$ approximation in the
Frobenius norm. 
\begin{proof}[Proof of Proposition \ref{prop: regret PCA}]
Since ${\cal F}=\{I_{d}\}$ is a singleton, there is no distinction
between pure and mixed minimax regret. It holds that 
\[
\regret(R,f)=\E\left[\|\boldsymbol{x}-Q^{\top}R^{\top}\boldsymbol{x}\|^{2}\right]
\]
where $A=Q^{\top}R^{\top}\in\mathbb{R}^{d\times d}$ is a rank $r$
matrix. For any $A\in\mathbb{R}^{d\times d}$ 
\begin{align}
\E\left[\|\boldsymbol{x}-A\boldsymbol{x}\|^{2}\right] & =\E\left[\boldsymbol{x}^{\top}\boldsymbol{x}-\boldsymbol{x}^{\top}A\boldsymbol{x}-\boldsymbol{x}^{\top}A^{\top}\boldsymbol{x}+\boldsymbol{x}^{\top}A^{\top}A\boldsymbol{x}\right]\\
 & =\Tr\left[\Sigma_{\boldsymbol{x}}-A\Sigma_{\boldsymbol{x}}-A^{\top}\Sigma_{\boldsymbol{x}}+A^{\top}A\Sigma_{\boldsymbol{x}}\right]\\
 & =\left\Vert \Sigma_{\boldsymbol{x}}^{1/2}-\Sigma_{\boldsymbol{x}}^{1/2}A\right\Vert _{F}^{2}\\
 & =\left\Vert \Sigma_{\boldsymbol{x}}^{1/2}-B\right\Vert _{F}^{2},
\end{align}
where $B\dfn\Sigma_{\boldsymbol{x}}^{1/2}A$. By the classic \emph{Eckart-Young-Mirsky
theorem} \citep[Example 8.1]{wainwright2019high} \citep[Section 4.1.4]{vershynin2018high}
(see Appendix \ref{sec:Useful-mathematical-results}), the best rank
$r$ approximation in the Frobenius norm is obtained by setting
\[
B^{*}=\sum_{i=1}^{r}\lambda_{i}(\Sigma_{\boldsymbol{x}}^{1/2})\cdot v_{i}v_{i}^{\top}=\sum_{i=1}^{r}\sqrt{\lambda_{i}(\Sigma_{\boldsymbol{x}})}\cdot v_{i}v_{i}^{\top}
\]
where $v_{i}\equiv v_{i}(\Sigma_{\boldsymbol{x}}^{1/2})=v_{i}(\Sigma_{\boldsymbol{x}})$
is the $i$th eigenvector of $\Sigma_{\boldsymbol{x}}^{1/2}$ (or
$\Sigma_{\boldsymbol{x}}$). Then, the optimal $A$ is 
\[
A^{*}=\sum_{i=1}^{r}\sqrt{\lambda_{i}(\Sigma_{\boldsymbol{x}})}\cdot\Sigma_{\boldsymbol{x}}^{-1/2}v_{i}v_{i}^{\top}=\sum_{i=1}^{r}\sqrt{\lambda_{i}(\Sigma_{\boldsymbol{x}})}\cdot\Sigma_{\boldsymbol{x}}^{-1/2}v_{i}v_{i}^{\top}=\sum_{i=1}^{r}v_{i}v_{i}^{\top},
\]
since $v_{i}$ is also an eigenvector of $\Sigma_{\boldsymbol{x}}^{-1/2}$.
Letting $R=U(R)\Sigma(R)V^{\top}(R)$ and $Q=U(Q)\Sigma(Q)V^{\top}(Q)$
be the singular value decomposition of $R$ and $Q$, respectively,
it holds that 
\[
Q^{\top}R^{\top}=V(Q)\Sigma^{\top}(Q)V(Q)V(R)\Sigma^{\top}(R)U^{\top}(R).
\]
Setting $V(Q)=V(R)=I_{r}$, and $\Sigma^{\top}(Q)=\Sigma(R)\in\mathbb{R}^{d\times r}$
to have $r$ ones on the diagonal (and all other entries are zero),
as well as $U(Q)=U(R)$ to be an orthogonal matrix whose first $r$
columns are $\{v_{i}\}_{i\in[r]}$ results that $Q^{\top}R^{\top}=A^{*}$,
as required.
\end{proof}

\subsection{Proofs of pure and mixed minimax representations \label{subsec:Analysis-of-pure-mixed}}

Before the proof of Theorem \ref{thm: pure minimax regret linear quadratic constraint},
we state a simple and useful lemma, which provides the pointwise value
of the regret and the optimal linear predictor for a given representation
and response.
\begin{lem}
\label{lem: LS solution for predictor}Consider the representation
$\boldsymbol{z}=R^{\top}\boldsymbol{x}\in\mathbb{R}^{r}$. It then
holds that
\begin{align}
\regret(R,f\mid P_{\boldsymbol{x}}) & =\min_{q\in\mathbb{R}^{r}}\E\left[\left(f^{\top}\boldsymbol{x}+\boldsymbol{n}-q^{\top}\boldsymbol{z}\right)^{2}\right]\\
 & =\E\left[\E[\boldsymbol{n}^{2}\mid\boldsymbol{x}]\right]+f^{\top}\left(\Sigma_{\boldsymbol{x}}-\Sigma_{\boldsymbol{x}}R(R^{\top}\Sigma_{\boldsymbol{x}}R)^{-1}R^{\top}\Sigma_{\boldsymbol{x}}\right)f.
\end{align}
\end{lem}

\begin{proof}
The orthogonality principle states that 
\begin{equation}
\E\left[\left(f^{\top}\boldsymbol{x}+\boldsymbol{n}-q^{\top}\boldsymbol{z}\right)\cdot\boldsymbol{z}^{\top}\right]=0\label{eq: orthogonality principle}
\end{equation}
 must hold for the optimal linear estimator. Using $\boldsymbol{z}=R^{\top}\boldsymbol{x}$
and taking expectations leads to the standard least-squares (LS) solution
\begin{equation}
q_{\text{LS}}=(R^{\top}\Sigma_{\boldsymbol{x}}R)^{-1}R^{T}\Sigma_{\boldsymbol{x}}f,\label{eq: LS solution for h in minimax}
\end{equation}
assuming that $R^{\top}\Sigma_{\boldsymbol{x}}R$ is invertible (which
we indeed assume as if this is not the case, the representation can
be reduced to a dimension lower than $r$ in a lossless manner). The
resulting regret of $R$ is thus given by 
\begin{align}
\regret(R,f\mid P_{\boldsymbol{x}}) & =\E\left[\left(f^{\top}\boldsymbol{x}+\boldsymbol{n}-q_{\text{LS}}^{\top}\boldsymbol{z}\right)^{2}\right]\\
 & \trre[=,a]\E\left[\left(f^{\top}\boldsymbol{x}+\boldsymbol{n}\right)^{\top}\left(f^{\top}\boldsymbol{x}+\boldsymbol{n}-q_{\text{LS}}^{\top}\boldsymbol{z}\right)\right]\\
 & =\E\left[\left(f^{\top}\boldsymbol{x}+\boldsymbol{n}\right)^{2}-\left(f^{\top}\boldsymbol{x}+\boldsymbol{n}\right)^{\top}q_{\text{LS}}^{\top}\boldsymbol{z}\right]\\
 & \trre[=,b]\E\left[\E[\boldsymbol{n}^{2}\mid\boldsymbol{x}]\right]+f^{\top}\Sigma_{\boldsymbol{x}}f-\E\left[\boldsymbol{x}^{\top}fq_{\text{LS}}^{\top}R^{\top}\boldsymbol{x}\right]\\
 & =\E\left[\E[\boldsymbol{n}^{2}\mid\boldsymbol{x}]\right]+f^{\top}\Sigma_{\boldsymbol{x}}f-\Tr\left[fq_{\text{LS}}^{\top}R^{\top}\Sigma_{\boldsymbol{x}}\right]\\
 & =\E\left[\E[\boldsymbol{n}^{2}\mid\boldsymbol{x}]\right]+f^{\top}\Sigma_{\boldsymbol{x}}f-q_{\text{LS}}^{\top}R^{\top}\Sigma_{\boldsymbol{x}}f\\
 & \trre[=,c]\E\left[\E[\boldsymbol{n}^{2}\mid\boldsymbol{x}]\right]+f^{\top}\left(\Sigma_{\boldsymbol{x}}-\Sigma_{\boldsymbol{x}}R(R^{\top}\Sigma_{\boldsymbol{x}}R)^{-1}R^{\top}\Sigma_{\boldsymbol{x}}\right)f,
\end{align}
where $(a)$ follows from the orthogonality principle in \eqref{eq: orthogonality principle},
$(b)$ follows from the tower property of conditional expectation
and since $\E[\boldsymbol{x}\boldsymbol{n}]=\E[\boldsymbol{x}\cdot\E[\boldsymbol{n}\mid\boldsymbol{x}]]=0$,
and $(c)$ follows by substituting $q_{\text{LS}}$ from \eqref{eq: LS solution for h in minimax}.
\end{proof}
We may now prove Theorem \ref{thm: pure minimax regret linear quadratic constraint}. 
\begin{proof}[Proof of Theorem \ref{thm: pure minimax regret linear quadratic constraint}]
For any given $f$, the optimal predictor based on $x\in\mathbb{R}^{d}$
achieves average loss of
\[
\regret(R=I_{d},f\mid P_{\boldsymbol{x}})=\min_{q\in\mathbb{R}^{d}}\E\left[\left(f^{\top}\boldsymbol{x}+\boldsymbol{n}-q^{\top}\boldsymbol{x}\right)^{2}\right]=\E\left[\E[\boldsymbol{n}^{2}\mid\boldsymbol{x}]\right]
\]
(obtained by setting $R=I_{d}$ in Lemma \ref{lem: LS solution for predictor}
so that $\boldsymbol{z}=\boldsymbol{x}$). Hence, the resulting regret
of $R$ over an adversarial choice of $f\in{\cal F}_{S}$ is 
\begin{align}
\max_{f\in{\cal F}_{S}}\regret(R,f)= & \max_{f\in{\cal F}}\E\left[\left|f^{\top}\boldsymbol{x}+\boldsymbol{n}-q_{\text{LS}}^{\top}\boldsymbol{z}\right|^{2}\right]-\E\left[\E[\boldsymbol{n}^{2}\mid\boldsymbol{x}]\right]\nonumber \\
 & \trre[=,a]\max_{f\in{\cal F}_{S}}f^{\top}\left(\Sigma_{\boldsymbol{x}}-\Sigma_{\boldsymbol{x}}R(R^{\top}\Sigma_{\boldsymbol{x}}R)^{-1}R^{\top}\Sigma_{\boldsymbol{x}}\right)f\\
 & \trre[=,b]\max_{\tilde{f}\colon\|\tilde{f}\|_{2}^{2}\leq1}\tilde{f}^{\top}\left(S^{1/2}\Sigma_{\boldsymbol{x}}S^{1/2}-S^{1/2}\Sigma_{\boldsymbol{x}}R(R^{\top}\Sigma_{\boldsymbol{x}}R)^{-1}R^{\top}\Sigma_{\boldsymbol{x}}S^{1/2}\right)\tilde{f}\\
 & \trre[=,c]\lambda_{1}\left(S^{1/2}\Sigma_{\boldsymbol{x}}S^{1/2}-S^{1/2}\Sigma_{\boldsymbol{x}}R(R^{\top}\Sigma_{\boldsymbol{x}}R)^{-1}R^{\top}\Sigma_{\boldsymbol{x}}S^{1/2}\right)\\
 & =\lambda_{1}\left[S^{1/2}\Sigma_{\boldsymbol{x}}^{1/2}\left(I_{d}-\Sigma_{\boldsymbol{x}}^{1/2}R(R^{\top}\Sigma_{\boldsymbol{x}}R)^{-1}R^{\top}\Sigma_{\boldsymbol{x}}^{1/2}\right)\Sigma_{\boldsymbol{x}}^{1/2}S^{1/2}\right]\\
 & \trre[=,d]\lambda_{1}\left[S^{1/2}\Sigma_{\boldsymbol{x}}^{1/2}\left(I_{d}-\tilde{R}(\tilde{R}^{\top}\tilde{R})^{-1}\tilde{R}^{\top}\right)\Sigma_{\boldsymbol{x}}^{1/2}S^{1/2}\right]\\
 & \trre[=,e]\lambda_{1}\left[\left(I_{d}-\tilde{R}(\tilde{R}^{\top}\tilde{R})^{-1}\tilde{R}^{\top}\right)\Sigma_{\boldsymbol{x}}^{1/2}S\Sigma_{\boldsymbol{x}}^{1/2}\left(I_{d}-\tilde{R}(\tilde{R}^{\top}\tilde{R})^{-1}\tilde{R}^{\top}\right)\right],\label{eq: minimax regret linear as maxeig}
\end{align}
where $(a)$ follows from Lemma \ref{lem: LS solution for predictor},
$(b)$ follows by letting $\tilde{f}\dfn S^{-1/2}f$ and recalling
that any $f\in{\cal F}$ must satisfy $\|f\|_{S}^{2}\leq1$, $(c)$
follows from the \emph{Rayleigh quotient theorem} \citep[Theorem 4.2.2]{horn2012matrix}
(see Appendix \ref{sec:Useful-mathematical-results}), $(d)$ follows
by letting $\tilde{R}\dfn\Sigma_{\boldsymbol{x}}^{1/2}R$, and $(e)$
follows since $I_{d}-\tilde{R}(\tilde{R}^{\top}\tilde{R})^{-1}\tilde{R}^{\top}$
is an orthogonal projection (idempotent and symmetric matrix) of rank
$d-r$. 

Now, to find the minimizer of $\max_{f\in{\cal F}_{S}}\regret(R,f)$
over $R$, we note that 
\begin{align}
 & \lambda_{1}\left[\left(I_{d}-\tilde{R}(\tilde{R}^{\top}\tilde{R})^{-1}\tilde{R}^{\top}\right)\Sigma_{\boldsymbol{x}}^{1/2}S\Sigma_{\boldsymbol{x}}^{1/2}\left(I_{d}-\tilde{R}(\tilde{R}^{\top}\tilde{R})^{-1}\tilde{R}^{\top}\right)\right]\nonumber \\
 & \trre[=,a]\max_{u\colon\|u\|_{2}=1}u^{\top}\left(I_{d}-\tilde{R}(\tilde{R}^{\top}\tilde{R})^{-1}\tilde{R}^{\top}\right)\Sigma_{\boldsymbol{x}}^{1/2}S\Sigma_{\boldsymbol{x}}^{1/2}\left(I_{d}-\tilde{R}(\tilde{R}^{\top}\tilde{R})^{-1}\tilde{R}^{\top}\right)u\\
 & \trre[=,b]\max_{u\colon\|u\|_{2}=1,\;\tilde{R}^{\top}u=0}u^{\top}\Sigma_{\boldsymbol{x}}^{1/2}S\Sigma_{\boldsymbol{x}}^{1/2}u\\
 & \trre[\geq,c]\min_{{\cal S}\colon\dim({\cal S})=d-r}\max_{u\colon\|u\|_{2}=1,\;u\in{\cal S}}u^{\top}\Sigma_{\boldsymbol{x}}^{1/2}S\Sigma_{\boldsymbol{x}}^{1/2}u\\
 & \trre[=,d]\lambda_{r+1}\left(\Sigma_{\boldsymbol{x}}^{1/2}S\Sigma_{\boldsymbol{x}}^{1/2}\right),
\end{align}
where $(a)$ follows again from the \emph{Rayleigh quotient theorem}
\citep[Theorem 4.2.2]{horn2012matrix}, $(b)$ follows since $I_{d}-\tilde{R}(\tilde{R}^{\top}\tilde{R})^{-1}\tilde{R}^{\top}$
is an orthogonal projection matrix, and so we may write $u=u_{\perp}+u_{\Vert}$
so that $\|u_{\perp}\|^{2}+\|u_{\Vert}\|^{2}=1$ and $\tilde{R}^{\top}u_{\perp}=0$;
Hence replacing $u$ with $u_{\perp}$ only increases the value of
the maximum, $(c)$ follows by setting ${\cal S}$ to be a $d-r$
dimensional subspace of $\mathbb{R}^{d}$, and $(d)$ follows by the
\emph{Courant--Fischer variational characterization} \citep[Theorem 4.2.6]{horn2012matrix}
(see Appendix \ref{sec:Useful-mathematical-results}). The lower bound
in $(c)$ can be achieved by setting the $r$ columns of $\tilde{R}\in\mathbb{R}^{d\times r}$
to be the top eigenvectors $\{v_{i}(\Sigma_{\boldsymbol{x}}^{1/2}S\Sigma_{\boldsymbol{x}}^{1/2})\}_{i\in[r]}$.
This leads to the minimax representation $\tilde{R}^{*}$. From \eqref{eq: minimax regret linear as maxeig},
the worst case $\tilde{f}$ is the top eigenvector of 
\[
\left(I_{d}-\tilde{R}^{*}((\tilde{R})^{*\top}\tilde{R}^{*})^{-1}(\tilde{R})^{*\top}\right)\Sigma_{\boldsymbol{x}}^{1/2}S\Sigma_{\boldsymbol{x}}^{1/2}\left(I_{d}-\tilde{R}^{*}((\tilde{R})^{*\top}\tilde{R}^{*})^{-1}(\tilde{R})^{*\top}\right).
\]
This is a symmetric matrix, whose top eigenvector is the $(r+1)$th
eigenvector $v_{r+1}(\Sigma_{\boldsymbol{x}}^{1/2}S\Sigma_{\boldsymbol{x}}^{1/2})$. 
\end{proof}
We next prove Theorem \ref{thm: mixed minimax regret linear quadratic constraint}. 
\begin{proof}[Proof of Theorem \ref{thm: mixed minimax regret linear quadratic constraint}]
We follow the proof strategy mentioned after the statement of the
theorem. We assume that $\boldsymbol{n}\equiv0$ with probability
$1$, since, as for the pure minimax regret, this unavoidable additive
term of $\E\left[\E[\boldsymbol{n}^{2}\mid\boldsymbol{x}]\right]$
to the loss does not affect the regret.

\emph{\uline{The minimax problem -- a direct computation:}}\emph{
}As in the derivations leading to \eqref{eq: minimax regret linear as maxeig},
the minimax regret in \eqref{eq: minimax regret mixed} is given by
\begin{align}
 & \regret_{\mix}({\cal F}_{S}\mid\Sigma_{\boldsymbol{x}})\nonumber \\
 & =\min_{\pl(\boldsymbol{R})\in{\cal P}({\cal R})}\max_{f\in{\cal {\cal F}_{S}}}\E\left[\regret(\boldsymbol{R},f\mid\Sigma_{\boldsymbol{x}})\right]\\
 & =\min_{\pl(\boldsymbol{R})\in{\cal P}({\cal R})}\max_{\tilde{f}\colon\|\tilde{f}\|_{2}^{2}\leq1}\E\left[\tilde{f}^{\top}\left(S^{1/2}\Sigma_{\boldsymbol{x}}S^{1/2}-S^{1/2}\Sigma_{\boldsymbol{x}}\boldsymbol{R}(\boldsymbol{R}^{\top}\Sigma_{\boldsymbol{x}}\boldsymbol{R})^{-1}\boldsymbol{R}^{\top}\Sigma_{\boldsymbol{x}}S^{1/2}\right)\tilde{f}\right]\\
 & =\min_{\pl(\Sigma_{\boldsymbol{x}}^{-1/2}\tilde{\boldsymbol{R}})\in{\cal P}({\cal R})}\max_{\tilde{f}\colon\|\tilde{f}\|_{2}^{2}\leq1}\tilde{f}^{\top}S^{1/2}\Sigma_{\boldsymbol{x}}^{1/2}\E\left[I_{d}-\tilde{\boldsymbol{R}}(\tilde{\boldsymbol{R}}^{\top}\tilde{\boldsymbol{R}})^{-1}\tilde{\boldsymbol{R}}^{\top}\right]\Sigma_{\boldsymbol{x}}^{1/2}S^{1/2}\tilde{f}\\
 & =\min_{\pl(\Sigma_{\boldsymbol{x}}^{-1/2}\tilde{\boldsymbol{R}})\in{\cal P}({\cal R})}\lambda_{1}\left(S^{1/2}\Sigma_{\boldsymbol{x}}^{1/2}\E\left[I_{d}-\tilde{\boldsymbol{R}}(\tilde{\boldsymbol{R}}^{\top}\tilde{\boldsymbol{R}})^{-1}\tilde{\boldsymbol{R}}^{\top}\right]\Sigma_{\boldsymbol{x}}^{1/2}S^{1/2}\right),\label{eq: minimax mixed value linear MSE quadratic response direct}
\end{align}
where $\tilde{\boldsymbol{R}}=\Sigma_{\boldsymbol{x}}^{1/2}\boldsymbol{R}$.
Determining the optimal distribution of the representation directly
from this expression seems to be intractable. We thus next solve the
maximin problem, and then return to the maximin problem \eqref{eq: minimax mixed value linear MSE quadratic response direct},
set a specific random representation, and show that it achieves the
maximin value. This, in turn, establishes the optimality of this choice. 

\emph{\uline{The maximin problem:}} Let an arbitrary $\pl(\boldsymbol{f})$
be given. Then, taking the expectation of the regret over the random
choice of $\boldsymbol{f}$, for any given $R\in{\cal R}$,
\begin{align}
\E\left[\regret(R,\boldsymbol{f})\right] & \trre[=,a]\E\left[\Tr\left[\left(S^{1/2}\Sigma_{\boldsymbol{x}}S^{1/2}-S^{1/2}\Sigma_{\boldsymbol{x}}R(R^{\top}\Sigma_{\boldsymbol{x}}R)^{-1}R^{\top}\Sigma_{\boldsymbol{x}}S^{1/2}\right)\tilde{\boldsymbol{f}}\tilde{\boldsymbol{f}}^{\top}\right]\right]\\
 & \trre[=,b]\Tr\left[\left(S^{1/2}\Sigma_{\boldsymbol{x}}S^{1/2}-S^{1/2}\Sigma_{\boldsymbol{x}}R(R^{\top}\Sigma_{\boldsymbol{x}}R)^{-1}R^{\top}\Sigma_{\boldsymbol{x}}S^{1/2}\right)\tilde{\Sigma}_{\boldsymbol{f}}\right]\\
 & =\Tr\left[\tilde{\Sigma}_{\boldsymbol{f}}^{1/2}\left(S^{1/2}\Sigma_{\boldsymbol{x}}S^{1/2}-S^{1/2}\Sigma_{\boldsymbol{x}}R(R^{\top}\Sigma_{\boldsymbol{x}}R)^{-1}R^{\top}\Sigma_{\boldsymbol{x}}S^{1/2}\right)\tilde{\Sigma}_{\boldsymbol{f}}^{1/2}\right]\\
 & \trre[=,c]\Tr\left[\tilde{\Sigma}_{\boldsymbol{f}}^{1/2}S^{1/2}\Sigma_{\boldsymbol{x}}^{1/2}\left(I_{d}-\tilde{R}(\tilde{R}^{\top}\tilde{R})^{-1}\tilde{R}^{\top}\right)\Sigma_{\boldsymbol{x}}^{1/2}S^{1/2}\tilde{\Sigma}_{\boldsymbol{f}}^{1/2}\right]\\
 & =\Tr\left[\left(I-\tilde{R}(\tilde{R}^{\top}\tilde{R})^{-1}\tilde{R}^{\top}\right)\Sigma_{\boldsymbol{x}}^{1/2}S^{1/2}\tilde{\Sigma}_{\boldsymbol{f}}S^{1/2}\Sigma_{\boldsymbol{x}}^{1/2}\right]\\
 & \trre[=,d]\Tr\left[\left(I-\tilde{R}(\tilde{R}^{\top}\tilde{R})^{-1}\tilde{R}^{\top}\right)\Sigma_{\boldsymbol{x}}^{1/2}S^{1/2}\tilde{\Sigma}_{\boldsymbol{f}}S^{1/2}\Sigma_{\boldsymbol{x}}^{1/2}\left(I-\tilde{R}(\tilde{R}^{\top}\tilde{R})^{-1}\tilde{R}^{\top}\right)\right]\\
 & \trre[\geq,e]\min_{W\in\mathbb{R}^{d\times(d-r)}\colon W^{\top}W=I_{d-r}}\Tr\left[W^{\top}\Sigma_{\boldsymbol{x}}^{1/2}S^{1/2}\tilde{\Sigma}_{\boldsymbol{f}}S^{1/2}\Sigma_{\boldsymbol{x}}^{1/2}W\right]\\
 & \trre[=,f]\sum_{i=r+1}^{d}\lambda_{i}(\Sigma_{\boldsymbol{x}}^{1/2}S^{1/2}\tilde{\Sigma}_{\boldsymbol{f}}S^{1/2}\Sigma_{\boldsymbol{x}}^{1/2})\\
 & =\sum_{i=r+1}^{d}\lambda_{i}\left(\tilde{\Sigma}_{\boldsymbol{f}}S^{1/2}\Sigma_{\boldsymbol{x}}S^{1/2}\right),
\end{align}
where $(a)$ follows from Lemma \ref{lem: LS solution for predictor}
and setting $\tilde{\boldsymbol{f}}\dfn S^{-1/2}\boldsymbol{f}$,
$(b)$ follows by setting $\tilde{\Sigma}_{\boldsymbol{f}}\equiv\Sigma_{\tilde{\boldsymbol{f}}}=\E[\tilde{\boldsymbol{f}}\tilde{\boldsymbol{f}}^{\top}]$,
$(c)$ follows by setting $\tilde{R}\dfn\Sigma_{\boldsymbol{x}}^{1/2}R$,
$(d)$ follows since $I-\tilde{R}(\tilde{R}^{\top}\tilde{R})^{-1}\tilde{R}^{\top}$
is an orthogonal projection (idempotent and symmetric matrix) of rank
$d-r$, $(e)$ follows since any orthogonal projection can be written
as $WW^{\top}$ where $W\in\mathbb{R}^{d\times(d-r)}$ is an orthogonal
matrix $W^{\top}W=I_{d-r}$, $(f)$ follows from \emph{\citet{fan1949theorem}'s
variational characterization} \citep[Corollary 4.3.39.]{horn2012matrix}
(see Appendix \ref{sec:Useful-mathematical-results}). Equality in
$(e)$ can be achieved by letting $\tilde{R}$ be the top $r$ eigenvectors
of $\Sigma_{\boldsymbol{x}}^{1/2}S^{1/2}\tilde{\Sigma}_{\boldsymbol{f}}S^{1/2}\Sigma_{\boldsymbol{x}}^{1/2}$. 

The next step of the derivation is to maximize the expected regret
over the probability law of $\boldsymbol{f}$ (or $\tilde{\boldsymbol{f}}$).
Evidently, $\E[\regret(R,\boldsymbol{f})]=\sum_{i=r+1}^{d}\lambda_{i}(\tilde{\Sigma}_{\boldsymbol{f}}S^{1/2}\Sigma_{\boldsymbol{x}}S^{1/2})$
only depends on the random function $\tilde{\boldsymbol{f}}$ via
$\tilde{\Sigma}_{\boldsymbol{f}}$. The covariance matrix $\tilde{\Sigma}_{\boldsymbol{f}}$
is constrained as follows. Recall that $\boldsymbol{f}$ is supported
on ${\cal F}_{S}\dfn\{f\in\mathbb{R}^{d}\colon\|f\|_{S}^{2}\leq1\}$
(see \eqref{eq: quadratic constraints}), and let $\Sigma_{\boldsymbol{f}}=\E[\boldsymbol{f}\boldsymbol{f}^{\top}]$
be its covariance matrix. Then, it must hold that $\Tr[S^{-1}\Sigma_{\boldsymbol{f}}]\le1$.
Then, it also holds that 
\begin{align}
1\geq\Tr[S^{-1}\Sigma_{\boldsymbol{f}}] & =\Tr\left[\E[S^{-1}\boldsymbol{f}\boldsymbol{f}^{\top}]\right]\\
 & =\E\left[\boldsymbol{f}^{\top}S^{-1}\boldsymbol{f}\right]\\
 & =\E\left[\tilde{\boldsymbol{f}}^{\top}\tilde{\boldsymbol{f}}\right]\\
 & =\Tr\left[\tilde{\Sigma}_{\boldsymbol{f}}\right]
\end{align}
where $\tilde{\Sigma}_{\boldsymbol{f}}=S^{-1/2}\Sigma_{\boldsymbol{f}}S^{-1/2}$.
Conversely, given any covariance matrix $\tilde{\Sigma}_{\boldsymbol{f}}\in\mathbb{S}_{++}^{d}$
such that $\Tr[\tilde{\Sigma}_{\boldsymbol{f}}]\leq1$ there exists
a random vector $\boldsymbol{f}$ supported on ${\cal F}_{S}$ such
that 
\[
\E[\boldsymbol{f}\boldsymbol{f}^{\top}]=S^{1/2}\tilde{\Sigma}_{\boldsymbol{f}}S^{-1/2}.
\]
We show this by an explicit construction. Let $\tilde{\Sigma}_{\boldsymbol{f}}=\tilde{V}_{\boldsymbol{f}}\tilde{\Lambda}_{\boldsymbol{f}}\tilde{V}_{\boldsymbol{f}}^{\top}$
be the eigenvalue decomposition of $\tilde{\Sigma}_{\boldsymbol{f}}$,
and, for brevity, denote by $\tilde{\lambda}_{i}\equiv\lambda_{i}(\tilde{\Sigma}_{\boldsymbol{f}})$
the diagonal elements of $\tilde{\Lambda}_{\boldsymbol{f}}$. Let
$\{\boldsymbol{q}_{i}\}_{i\in[d]}$ be a set of independent and identically
(IID) distributed  random variables, so that $\boldsymbol{q}_{i}$
is Rademacher, that is $\P[\boldsymbol{q}_{i}=1]=\P[\boldsymbol{q}_{i}=-1]=1/2$.
Define the random vector 
\[
\boldsymbol{g}\dfn\left(\boldsymbol{q}_{1}\cdot\sqrt{\tilde{\lambda}_{1}},\cdots,\boldsymbol{q}_{d}\cdot\sqrt{\tilde{\lambda}_{d}}\right)^{\top}.
\]
The constraint $\Tr[\tilde{\Sigma}_{\boldsymbol{f}}]\leq1$ implies
that $\sum\tilde{\lambda}_{i}\leq1$ and so $\|\boldsymbol{g}\|^{2}=\sum_{i=1}^{d}\tilde{\lambda}_{i}\leq1$
with probability $1$. Then, letting $\tilde{\boldsymbol{f}}=\tilde{V}_{\boldsymbol{f}}\boldsymbol{g}$
it also holds that $\|\tilde{\boldsymbol{f}}\|_{2}^{2}=\|\boldsymbol{g}\|_{2}^{2}\leq1$
with probability $1$, and furthermore, 
\[
\E\left[\tilde{\boldsymbol{f}}\tilde{\boldsymbol{f}}^{\top}\right]=\tilde{V}_{\boldsymbol{f}}\E\left[\boldsymbol{g}\boldsymbol{g}^{\top}\right]\tilde{V}_{\boldsymbol{f}}^{\top}=\tilde{V}_{\boldsymbol{f}}\tilde{\Lambda}_{\boldsymbol{f}}\tilde{V}_{\boldsymbol{f}}^{\top}=\tilde{\Sigma}_{\boldsymbol{f}}.
\]
Consequently, letting $\boldsymbol{f}=S^{1/2}\boldsymbol{f}$ assures
that $\|\boldsymbol{f}\|_{S}=\|\tilde{\boldsymbol{f}}\|_{2}\leq1$
and $\E[\boldsymbol{f}\boldsymbol{f}^{\top}]=S^{1/2}\tilde{\Sigma}_{\boldsymbol{f}}S^{-1/2}$,
as was required to obtain. Therefore, instead of maximizing over probability
laws on ${\cal P}({\cal F}_{S})$, we may equivalently maximize over
$\tilde{\Sigma}_{\boldsymbol{f}}\in\mathbb{S}_{++}^{d}$ such that
$\Tr[\tilde{\Sigma}_{\boldsymbol{f}}]\leq1$, i.e., to solve\textbf{
\begin{equation}
\regret_{\mix}({\cal F}_{S}\mid\Sigma_{\boldsymbol{x}})=\max_{\tilde{\Sigma}_{\boldsymbol{f}}\colon\Tr[\tilde{\Sigma}_{\boldsymbol{f}}]\leq1}\sum_{i=r+1}^{d}\lambda_{i}(\tilde{\Sigma}_{\boldsymbol{f}}S^{1/2}\Sigma_{\boldsymbol{x}}S^{1/2}).\label{eq: maximin regret optimization problem}
\end{equation}
}The optimization problem in \eqref{eq: maximin regret optimization problem}
is solved in Lemma \ref{lem: minimal eigs maximization under modified trace constraints},
and is provided after this proof. Setting \textbf{$\Sigma=S^{1/2}\Sigma_{\boldsymbol{x}}S^{1/2}$}
in Lemma \ref{lem: minimal eigs maximization under modified trace constraints},
and letting $\lambda_{i}\equiv\lambda_{i}(S^{1/2}\Sigma_{\boldsymbol{x}}S^{1/2})$,
the solution is given by 
\begin{equation}
\frac{\ell^{*}-r}{\sum_{i=1}^{\ell^{*}}\frac{1}{\lambda_{i}}}\label{eq: mixed regret linear quadratic proof}
\end{equation}
where \textbf{$\ell^{*}\in[d]\backslash[r]$} satisfies 
\begin{equation}
\frac{\ell^{*}-r}{\lambda_{\ell^{*}}}\leq\sum_{i=1}^{\ell^{*}}\frac{1}{\lambda_{i}}\leq\frac{\ell^{*}-r}{\lambda_{\ell^{*}+1}}.\label{eq: condition on maximin rank mixed regret linear quadratic proof}
\end{equation}
Lemma \ref{lem: minimal eigs maximization under modified trace constraints}
also directly implies that an optimal $\tilde{\Sigma}_{\boldsymbol{f}}$
is given as in \eqref{eq: maximin Sigma_f}. The value in \eqref{eq: mixed regret linear quadratic proof}
is exactly $\regret_{\mix}({\cal F}_{S}\mid\Sigma_{\boldsymbol{x}})$
claimed by the theorem, and we next show it is indeed achievable by
a properly constructed random representation. 

\emph{\uline{The minimax problem -- a solution via the maximin
certificate:}}\emph{ }Given the value of the regret game in mixed
strategies found in \eqref{eq: mixed regret linear quadratic proof},
we may also find a minimax representation in mixed strategies. To
this end, we return to the minimax expression in \eqref{eq: minimax mixed value linear MSE quadratic response direct},
and propose a random representation which achieves the maximin value
in \eqref{eq: mixed regret linear quadratic proof}. Note that for
any given $\tilde{R}$, the matrix $I_{d}-\tilde{R}(\tilde{R}^{\top}\tilde{R})^{-1}\tilde{R}^{\top}$
is an orthogonal projection, that is, a symmetric matrix whose eigenvalues
are all either $0$ or $1$, and it has at most $r$ eigenvalues equal
to zero. We denote its eigenvalue decomposition by $I_{d}-\tilde{R}(\tilde{R}^{\top}\tilde{R})^{-1}\tilde{R}^{\top}=U\Omega U^{\top}$.
Then, any probability law on $\tilde{\boldsymbol{R}}$ induces a probability
law on $\boldsymbol{U}$ and $\boldsymbol{\Omega}$ (and vice-versa).
To find the mixed minimax representation, we propose setting $\boldsymbol{U}=V(\Sigma_{\boldsymbol{x}}^{1/2}S\Sigma_{\boldsymbol{x}}^{1/2})\equiv V$
with probability $1$, that is, to be deterministic, and thus only
randomize $\boldsymbol{\Omega}$. With this choice, and by denoting,
for brevity, $\Lambda\equiv\Lambda(\Sigma_{\boldsymbol{x}}^{1/2}S\Sigma_{\boldsymbol{x}}^{1/2})=\Lambda(S^{1/2}\Sigma_{\boldsymbol{x}}S^{1/2})$,
the value of the objective function in \eqref{eq: minimax mixed value linear MSE quadratic response direct}
is given by 
\begin{align}
 & \lambda_{1}\left(S^{1/2}\Sigma_{\boldsymbol{x}}^{1/2}V\cdot\E\left[\boldsymbol{\Omega}\right]\cdot V^{\top}\Sigma_{\boldsymbol{x}}^{1/2}S^{1/2}\right)\nonumber \\
 & =\lambda_{1}\left(\E\left[\boldsymbol{\Omega}\right]\cdot V^{\top}\Sigma_{\boldsymbol{x}}^{1/2}S\Sigma_{\boldsymbol{x}}^{1/2}V\right)\\
 & =\lambda_{1}\left(\E\left[\boldsymbol{\Omega}\right]\cdot\Lambda\right).
\end{align}
Now, the distribution of $\boldsymbol{\Omega}$ is equivalent to a
distribution on its diagonal, which is supported on the finite set
${\cal A}\dfn\{a\in\{0,1\}^{d}\colon\|a\|_{1}\geq d-r\}$. Our goal
is thus to find a probability law on $\boldsymbol{a}$, supported
on ${\cal A}$, which solves
\[
\min_{\pl(\boldsymbol{\Omega})}\max_{i\in[d]}\lambda_{1}\left(\E\left[\boldsymbol{\Omega}\right]\cdot\Lambda\right)=\min_{\pl(\boldsymbol{a})}\max_{i\in[d]}\E[\boldsymbol{a}_{i}]\lambda_{i}
\]
where $\lambda_{i}\equiv\lambda_{i}(S^{1/2}\Sigma_{\boldsymbol{x}}S^{1/2})$
are the diagonal elements of $\Lambda$. Consider $\ell^{*}$, the
optimal dimension of the maximin problem, which satisfies \eqref{eq: condition on maximin rank mixed regret linear quadratic proof}.
We then set $\boldsymbol{a}_{\ell^{*}+1}=\cdots=\boldsymbol{a}_{d}=1$
to hold with probability $1$, and so it remains to determine the
probability law of $\overline{\boldsymbol{a}}\dfn(\boldsymbol{a}_{1},\ldots,\boldsymbol{a}_{\ell^{*}})$,
supported on $\tilde{{\cal A}}\dfn\{a\in\{0,1\}^{\ell^{*}}\colon\|a\|_{1}\geq\ell^{*}-r\}$.
Clearly, reducing $\|a\|_{1}$ only reduces the objective function
$\max_{i\in[d]}\E[\boldsymbol{a}_{i}\lambda_{i}]$, and so we may
in fact assume that $\overline{\boldsymbol{a}}$ is supported on $\overline{{\cal A}}\dfn\{\overline{a}\in\{0,1\}^{\ell^{*}}\colon\|\overline{a}\|_{1}=\ell^{*}-r\}$,
a finite subset of cardinality ${\ell^{*} \choose r}$. Suppose that
we find a probability law $\pl(\overline{\boldsymbol{a}})$ supported
on $\overline{{\cal A}}$ such that 
\begin{equation}
\E[\boldsymbol{a}_{i}]=(\ell^{*}-r)\cdot\frac{1/\lambda_{i}}{\sum_{i=1}^{\ell^{*}}1/\lambda_{i}}\dfn b_{i},\label{eq: condition on the expected eigenvalues}
\end{equation}
for all $i\in[\ell^{*}]$. Then, since $\E[\boldsymbol{a}_{i}]=1$
for $i\in[d]\backslash[\ell^{*}]$
\begin{align}
\max_{i\in[d]}\E[\boldsymbol{a}_{i}]\lambda_{i} & =\max\left\{ \frac{\ell^{*}-r}{\sum_{i=1}^{\ell^{*}}\frac{1}{\lambda_{i}}},\lambda_{\ell^{*}+1},\cdots,\lambda_{d}\right\} \\
 & =\max\left\{ \frac{\ell^{*}-r}{\sum_{i=1}^{\ell^{*}}\frac{1}{\lambda_{i}}},\lambda_{\ell^{*}+1}\right\} \\
 & \trre[=,*]\frac{\ell^{*}-r}{\sum_{i=1}^{\ell^{*}}\frac{1}{\lambda_{i}}},
\end{align}
where $(*)$ follows from the condition on $\ell^{*}$ in the right
inequality of \eqref{eq: condition on maximin rank mixed regret linear quadratic proof}.
This proves that such probability law achieves the minimax regret
in mixed strategies. This last term is $\regret_{\mix}({\cal F}_{S}\mid\Sigma_{\boldsymbol{x}})$
claimed by the theorem. It remains to construct $\pl(\overline{\boldsymbol{a}})$
which satisfies \eqref{eq: condition on the expected eigenvalues}.
To this end, note that the set 
\[
{\cal C}\dfn\left\{ c\in[0,1]^{\ell^{*}}\colon\|c\|_{1}=\ell^{*}-r\right\} 
\]
is convex and compact, and $\overline{{\cal A}}$ is the set of its
\emph{extreme points} (${\cal C}$ is the convex hull of $\overline{{\cal A}}$).
Letting $\overline{b}=(b_{1},\ldots,b_{\ell^{*}})^{\top}$ as denoted
in \eqref{eq: condition on the expected eigenvalues}, it holds that
$\overline{b}_{i}\geq0$ and $\{\overline{b}_{i}\}_{i=1}^{\ell^{*}}$
is a non-decreasing sequence. Using the condition on $\ell^{*}$ in
the left inequality of \eqref{eq: condition on maximin rank mixed regret linear quadratic proof},
it then holds that 
\[
\overline{b}_{1}\leq\cdots\leq\overline{b}_{\ell^{*}}=(\ell^{*}-r)\cdot\frac{1/\lambda_{\ell^{*}}}{\sum_{i=1}^{\ell^{*}}1/\lambda_{i}}\leq1.
\]
Hence, $\overline{b}\in{\cal C}$. By Carath\'{e}odory's theorem
\citep[Prop. 1.3.1]{bertsekas2003convex} (see Appendix \ref{sec:Useful-mathematical-results}),
any point inside a convex compact set in $\mathbb{R}^{\ell^{*}}$
can be written as a convex combination of at most $\ell^{*}+1$ extreme
points. Thus, there exists $\{p_{\overline{a}}\}_{\overline{a}\in\overline{{\cal A}}}$
such that $p_{\overline{a}}\in[0,1]$ and $\sum_{\overline{a}\in\overline{{\cal A}}}p_{\overline{a}}=1$
so that $\overline{b}=\sum_{\overline{a}\in\overline{{\cal A}}}p_{\overline{a}}\cdot\overline{a}$,
and moreover the support of $p_{\overline{a}}$ has cardinality at
most $\ell^{*}+1$. Let $\overline{A}\in\{0,1\}^{\ell^{*}\times|\overline{{\cal A}}|}$
be such that its $j$th column is given by the $j$th member of ${\cal \overline{A}}$
(in an arbitrary order). Let $p\in[0,1]^{|\overline{{\cal A}}|}$
be a vector whose $j$th element corresponds to the $j$th member
of ${\cal \overline{A}}$. Then, $p$ is the solution to $\overline{A}p=\overline{b}$,
and as claimed above, such a solution with at most $\ell^{*}+1$ nonzero
entries always exists. Setting $\boldsymbol{a}=(\overline{a},\underbrace{1\ldots,1}_{d-\ell^{*}\text{ terms}})$
with probability $p_{\overline{a}}$ then assures that \eqref{eq: condition on the expected eigenvalues}
holds, as was required to be proved.

Given the above, we observe that setting $\tilde{R}$ as in the theorem
induces a distribution on $\boldsymbol{\Omega}$ for which the random
entries of its diagonal $\boldsymbol{a}$ satisfy \eqref{eq: condition on the expected eigenvalues},
and thus achieve $\regret_{\mix}({\cal F}_{S}\mid\Sigma_{\boldsymbol{x}})$. 
\end{proof}
We next turn to complete the proof of Theorem \ref{thm: mixed minimax regret linear quadratic constraint}
by solving the optimization problem in \eqref{eq: maximin regret optimization problem}.
Assume that $\Sigma\in\mathbb{S}_{++}^{d}$ is a strictly positive
covariance matrix $\Sigma\succ0$, and consider the optimization problem
\begin{align}
v_{r}^{*}=\max_{\tilde{\Sigma}_{\boldsymbol{f}}\in\mathbb{S}_{+}^{d}} & \sum_{i=r+1}^{d}\lambda_{i}(\tilde{\Sigma}_{\boldsymbol{f}}\Sigma)\nonumber \\
\st & \Tr[\tilde{\Sigma}_{\boldsymbol{f}}]\leq1\label{eq: sum of smallest eig optimization}
\end{align}
for some $r\in[d-1]$. Note that the objective function refers to
the maximization of the $d-r$ minimal eigenvalues of $\Sigma^{1/2}\tilde{\Sigma}_{\boldsymbol{f}}\Sigma^{1/2}$. 
\begin{lem}
\label{lem: minimal eigs maximization under modified trace constraints}Let
\[
a_{\ell}\dfn\frac{\ell-r}{\sum_{i=1}^{\ell}\frac{1}{\lambda_{i}(\Sigma)}}.
\]
The optimal value of \eqref{eq: sum of smallest eig optimization}
is $v^{*}=\max_{[d]\backslash[r]}a_{\ell}$ and $\ell^{*}\in\argmax_{[d]\backslash[r]}a_{\ell}$
iff 
\begin{equation}
\frac{\ell^{*}-r}{\lambda_{\ell^{*}}(\Sigma)}\leq\sum_{i=1}^{\ell^{*}}\frac{1}{\lambda_{i}(\Sigma)}\leq\frac{\ell^{*}-r}{\lambda_{\ell^{*}+1}(\Sigma)}.\label{eq: condition on optuimal index}
\end{equation}
An optimal solution is 
\begin{equation}
\tilde{\Sigma}_{\boldsymbol{f}}^{*}=\left[\sum_{i=1}^{\ell^{*}}\frac{1}{\lambda_{i}(\Sigma)}\right]^{-1}\cdot V(\Sigma)\diag\left(\frac{1}{\lambda_{1}(\Sigma)},\ldots,\frac{1}{\lambda_{\ell^{*}}(\Sigma)},0,\cdots,0\right)\cdot V(\Sigma)^{\top}.\label{eq: optimal prior maximin regert mse case}
\end{equation}
\end{lem}

\begin{proof}
Let $\overline{\Sigma}_{\boldsymbol{f}}=\Sigma^{1/2}\tilde{\Sigma}_{\boldsymbol{f}}\Sigma^{1/2}$,
let $\overline{\Sigma}_{\boldsymbol{f}}=\overline{U}_{\boldsymbol{f}}\overline{\Lambda}_{\boldsymbol{f}}\overline{U}_{\boldsymbol{f}}^{\top}$
be its eigenvalue decomposition, and, for brevity, denote $\overline{\lambda}_{i}\equiv\lambda_{i}(\overline{\Sigma}_{\boldsymbol{f}})$.
Then, the trace operation appearing in the constraint of \eqref{eq: sum of smallest eig optimization}
can be written as
\begin{align}
\Tr[\tilde{\Sigma}_{\boldsymbol{f}}] & =\Tr\left[\Sigma^{-1/2}\overline{\Sigma}_{\boldsymbol{f}}\Sigma^{-1/2}\right]\\
 & =\Tr\left[\Sigma^{-1/2}\overline{U}_{\boldsymbol{f}}\overline{\Lambda}_{\boldsymbol{f}}\overline{U}_{\boldsymbol{f}}^{\top}\Sigma^{-1/2}\right]\\
 & =\Tr\left[\Sigma^{-1/2}\left(\sum_{i=1}^{d}\lambda_{i}\overline{u}_{i}\overline{u}_{i}^{\top}\right)\Sigma^{-1/2}\right]\\
 & =\sum_{i=1}^{d}\overline{\lambda}_{i}\cdot\left(\overline{u}_{i}^{\top}\Sigma^{-1}\overline{u}_{i}\right)\\
 & =\sum_{i=1}^{d}c_{i}\overline{\lambda}_{i},
\end{align}
where $\overline{u}_{i}=v_{i}(\overline{U}_{\boldsymbol{f}})$ (that
is, the $i$th column of $\overline{U}_{\boldsymbol{f}}$), and $c_{i}\dfn\overline{u}_{i}^{\top}\Sigma^{-1}\overline{u}_{i}$
(which satisfies $c_{i}>0$). Thus, the optimization problem in \eqref{eq: sum of smallest eig optimization}
over $\tilde{\Sigma}_{\boldsymbol{f}}$ is equivalent to an optimization
problem over $\{\overline{\lambda}_{i},\overline{u}_{i}\}_{i\in[d]},$
given by
\begin{align}
v_{r}^{*}=\max_{\{\overline{u}_{i},\overline{\lambda}_{i}\}_{i\in[d]}} & \sum_{i=r+1}^{d}\overline{\lambda}_{i}\nonumber \\
\st & \sum_{i=1}^{d}c_{i}\overline{\lambda}_{i}\leq1,\nonumber \\
 & c_{i}=\overline{u}_{i}^{\top}\Sigma^{-1}\overline{u}_{i},\nonumber \\
 & \overline{u}_{i}^{\top}\overline{u}_{j}=\delta_{ij},\nonumber \\
 & \overline{\lambda}_{1}\geq\overline{\lambda}_{2}\geq\cdots\geq\overline{\lambda}_{d}\geq0.\label{eq: sum of smallest eig optimization alternative}
\end{align}
To solve the optimization problem \eqref{eq: sum of smallest eig optimization alternative},
let us fix feasible $\{\overline{u}_{i}\}_{i\in[d]}$, so that $\{c_{i}\}_{i\in[d]}$
are fixed too. This results the problem
\begin{align}
v_{r}^{*}(\{\overline{u}_{i}\})\equiv v_{r}^{*}(\{c_{i}\})=\max_{\{\overline{\lambda}_{i}\}_{i\in[d]}} & \sum_{i=r+1}^{d}\overline{\lambda}_{i}\nonumber \\
\st & \sum_{i=1}^{d}c_{i}\overline{\lambda}_{i}\leq1,\nonumber \\
 & \overline{\lambda}_{1}\geq\overline{\lambda}_{2}\geq\cdots\geq\overline{\lambda}_{d}\geq0.\label{eq: sum of smallest eig optimization alternative - only lambda}
\end{align}
The objective function of \eqref{eq: sum of smallest eig optimization alternative - only lambda}
is linear in $\{\overline{\lambda}_{i}\}_{i\in[d]}$ and its constraint
set is a convex bounded polytope. So the solution to \eqref{eq: sum of smallest eig optimization alternative - only lambda}
must be obtained on the boundary of the constraint set. Clearly, the
optimal value satisfies $v_{r}^{*}(\{c_{i}\})\geq0$, and thus the
solution $\{\overline{\lambda}_{i}^{*}\}_{i\in[d]}$ must be obtained
when the constraint $\sum_{i=1}^{d}c_{i}\overline{\lambda}_{i}\leq1$
is satisfied with equality. Indeed, if this is not the case then one
may scale all $\overline{\lambda}_{i}^{*}$ by a constant larger than
$1$, and obtain larger value of the objective, while still satisfying
the constraint. 

To find the optimal solution to \eqref{eq: sum of smallest eig optimization alternative - only lambda},
we consider feasible points for which $\ell\dfn\max\{i\in[d]\colon\overline{\lambda}_{i}>0)$
is fixed. Let $\{\overline{\lambda}_{i}^{*}\}_{i\in[d]}$ be the optimal
solution of \eqref{eq: sum of smallest eig optimization alternative - only lambda},
under the additional constraint that $\overline{\lambda}_{\ell+1}=\cdots=\overline{\lambda}_{d}=0$.
We next prove that $\overline{\lambda}_{1}^{*}=\cdots=\overline{\lambda}_{\ell}^{*}$
must hold. To this end, assume by contradiction that there exists
$j\in[\ell]$ so that $\overline{\lambda}_{j-1}^{*}>\overline{\lambda}_{j}^{*}>0$.
There are two cases to consider, to wit, whether $j-1<r+1$ and so
only $\overline{\lambda}_{j}$ appears in the objective of \eqref{eq: sum of smallest eig optimization alternative - only lambda},
or, otherwise, $j-1\geq r+1$ and then $\overline{\lambda}_{j-1}+\overline{\lambda}_{j}$
appears in the objective of \eqref{eq: sum of smallest eig optimization alternative - only lambda}.
Assuming the first case, let $\alpha=\overline{\lambda}_{j-1}^{*}c_{j-1}+\overline{\lambda}_{j}^{*}c_{j}$
and consider the optimization problem 
\begin{align}
\max_{\hat{\lambda}_{j-1},\hat{\lambda}_{j}} & \hat{\lambda}_{j}\nonumber \\
\st & \hat{\lambda}_{j-1}c_{j-1}+\hat{\lambda}_{j}c_{j}=\alpha,\nonumber \\
 & \hat{\lambda}_{j-1}\geq\hat{\lambda}_{j}>0.
\end{align}
It is easy to verify that the optimum of this problem is $\hat{\lambda}_{j-1}^{*}=\hat{\lambda}_{j}^{*}=\frac{\alpha}{c_{j-1}+c_{j}}$.
Thus, if $\overline{\lambda}_{j-1}^{*}>\overline{\lambda}_{j}^{*}$
then one can replace this pair with $\overline{\lambda}_{j-1}^{*}=\overline{\lambda}_{j}^{*}=\hat{\lambda}_{j-1}^{*}=\hat{\lambda}_{j}^{*}$
so that the value of the constraint $\sum_{i=1}^{d}\overline{\lambda}_{i}c_{i}$
remains the same, and thus $(\overline{\lambda}_{1}^{*},\cdots,\hat{\lambda}_{j-1}^{*},\hat{\lambda}_{j}^{*},\overline{\lambda}_{j+1}^{*},\ldots\overline{\lambda}_{d}^{*})$
is a feasible point, while the objective function value of \eqref{eq: sum of smallest eig optimization alternative - only lambda}
is smaller; a contradiction. Therefore, it must hold for the first
case that $\overline{\lambda}_{j-1}^{*}=\overline{\lambda}_{j}^{*}$.
For the second case, in a similar fashion, let now $\alpha=\overline{\lambda}_{j-1}^{*}c_{j-1}+\overline{\lambda}_{j}^{*}c_{j}$,
and consider the optimization problem 
\begin{align}
\max_{\hat{\lambda}_{j-1},\hat{\lambda}_{j}} & \hat{\lambda}_{j}+\hat{\lambda}_{j-1}\nonumber \\
\st & \hat{\lambda}_{j-1}c_{j-1}+\hat{\lambda}_{j}c_{j}=\alpha,\nonumber \\
 & \hat{\lambda}_{j-1}\geq\hat{\lambda}_{j}>0.
\end{align}
The solution for this optimization problem is at one of the two extreme
points of the feasible interval for $\hat{\lambda}_{j}$. Since $\lambda_{j}^{*}>0$
was assumed it therefore must hold that $\hat{\lambda}_{j-1}^{*}=\hat{\lambda}_{j}^{*}$,
and hence also $\overline{\lambda}_{j-1}^{*}=\overline{\lambda}_{j}^{*}$.
Thus, $\lambda_{j-1}^{*}<\lambda_{j}^{*}$ leads to a contradiction.
From the above, we deduce that the optimal solution of \eqref{eq: sum of smallest eig optimization alternative - only lambda}
under the additional constraint that $\overline{\lambda}_{\ell+1}=\cdots=\overline{\lambda}_{d}=0$
is
\begin{align}
 & \overline{\lambda}_{1}^{*}=\cdots=\overline{\lambda}_{\ell}^{*}=\frac{1}{\sum_{i=1}^{\ell}c_{i}}\\
 & \overline{\lambda}_{\ell+1}^{*}=\cdots=\overline{\lambda}_{d}^{*}=0,
\end{align}
and that the optimal value is $\frac{\ell-r}{\sum_{i=1}^{\ell}c_{i}}$.
Since $\ell\in[d]\backslash[r]$ can be arbitrarily chosen, we deduce
that the value of \eqref{eq: sum of smallest eig optimization alternative - only lambda}
is
\begin{equation}
v^{*}(\{c_{i}\})=\max_{\ell\in[d]\backslash[r]}\frac{\ell-r}{\sum_{i=1}^{\ell}c_{i}}.\label{eq: optimal value after eigenvalue optimization}
\end{equation}
For any given $\ell\in[d]\backslash[r]$, we may now optimize over
$\{\overline{u}_{i}\}$, which from \eqref{eq: optimal value after eigenvalue optimization}
is equivalent to minimizing $\sum_{i=1}^{\ell}c_{i}$. It holds that
\begin{align}
\min_{\{\overline{u}_{i}\}}\sum_{i=1}^{\ell}c_{i} & =\min_{\{\overline{u}_{i}\colon\overline{u}_{i}^{\top}\overline{u}_{j}=\delta_{ij}\}}\sum_{i=1}^{\ell}\overline{u}_{i}^{\top}\Sigma^{-1}\overline{u}_{i}\\
 & =\min_{\{\overline{u}_{i}\colon\overline{u}_{i}^{\top}\overline{u}_{j}=\delta_{ij}\}}\Tr\left[\Sigma^{-1}\sum_{i=1}^{\ell}\overline{u}_{i}\overline{u}_{i}^{\top}\right]\\
 & \trre[=,a]\min_{\grave{U}\in\mathbb{R}^{d\times\ell}\colon\grave{U}^{\top}\grave{U}=I_{\ell}}\Tr\left[\Sigma^{-1}\grave{U}\grave{U}^{\top}\right]\\
 & =\min_{\grave{U}\in\mathbb{R}^{d\times\ell}\colon\grave{U}^{\top}\grave{U}=I_{\ell}}\Tr\left[\grave{U}^{\top}\Sigma^{-1}\grave{U}\right]\\
 & \trre[=,b]\sum_{i=1}^{\ell}\frac{1}{\lambda_{i}(\Sigma)},\label{eq: optimization of sum of orthogonal quadratic forms}
\end{align}
where in $(a)$ $\grave{U}\in\mathbb{R}^{d\times\ell}$ whose $\ell$
columns are $\{\overline{u}_{i}\}_{i\in[\ell]}$ and $\grave{U}^{\top}\grave{U}=I_{\ell}$,
and in $(b)$ we have used \emph{\citet{fan1949theorem}'s variational
characterization} \citep[Corollary 4.3.39.]{horn2012matrix} (see
Appendix \ref{sec:Useful-mathematical-results}). Substituting back
to \eqref{eq: optimal value after eigenvalue optimization} results
that 
\[
v_{r}^{*}=\max_{\ell\in[d]\backslash[r]}\frac{\ell-r}{\sum_{i=1}^{\ell}\frac{1}{\lambda_{i}(\Sigma)}}=\max_{\ell\in[d]\backslash[r]}a_{\ell}.
\]
Let us denote that maximizer index by $\ell^{*}$. Then, Fan's characterization
is achieved by setting $\overline{U}_{\boldsymbol{f}}=V$ (so that
the $\ell^{*}$ columns of $\grave{U}$ are the $\ell^{*}$ eigenvectors
$v_{i}(\Sigma)$, corresponding to the $\ell^{*}$ largest eigenvalues
of $\Sigma$), so that 
\[
\overline{\Sigma}_{\boldsymbol{f}}^{*}=\left[\sum_{i=1}^{\ell^{*}}\frac{1}{\lambda_{i}(\Sigma)}\right]^{-1}\cdot V\cdot\diag\left(\underbrace{1,\ldots,1}_{\ell^{*}\,\text{terms}},0,\cdots,0\right)\cdot V^{\top},
\]
and then 
\begin{align}
\tilde{\Sigma}_{\boldsymbol{f}}^{*} & =\Sigma^{-1/2}\overline{\Sigma}_{\boldsymbol{f}}^{*}\Sigma^{-1/2}\\
 & =\left[\sum_{i=1}^{\ell^{*}}\frac{1}{\lambda_{i}(\Sigma)}\right]^{-1}\cdot V\Lambda^{-1/2}V^{\top}V\cdot\diag\left(1,\ldots,1,0,\cdots,0\right)V^{\top}V\Lambda^{-1/2}V^{\top}\\
 & =\left[\sum_{i=1}^{\ell^{*}}\frac{1}{\lambda_{i}(\Sigma)}\right]^{-1}\cdot V\cdot\diag\left(\frac{1}{\lambda_{1}(\Sigma)},\ldots,\frac{1}{\lambda_{\ell^{*}}(\Sigma)},0,\cdots,0\right)\cdot V^{\top}
\end{align}
as claimed in \eqref{eq: optimal prior maximin regert mse case}. 

To complete the proof, it remains to characterize $\ell^{*}$, which
belongs to the set possible indices maximizing $\{a_{\ell}\}_{\ell\in[d]\backslash[r]}$.
Since $\ell^{*}$ maximizes $a_{\ell}$ it must be a local maximizer,
that is, it must hold that $a_{\ell^{*}-1}\leq a_{\ell^{*}}\geq a_{\ell^{*}+1}$.
By simple algebra, these conditions are equivalent to those in \eqref{eq: condition on optuimal index}.
It remains to show that any $\ell\in[d]\backslash[r]$ which satisfies
\eqref{eq: condition on optuimal index} has the same value, and thus
any local maxima is a global maxima. We will show this by proving
that the sequence $\{a_{\ell}\}_{\ell=r}^{d}$ is \emph{unimodal},
as follows. Let $\Delta_{\ell}\dfn a_{\ell+1}-a_{\ell}$ be the discrete
derivative of $\{a_{\ell}\}_{\ell\in[d]}$, and consider the sequence
$\{\Delta_{\ell}\}_{\ell\in[d]\backslash[r]}$. We show that as $\ell$
increases from $r$ to $d$, $\{\Delta_{\ell}\}_{\ell\in[d]\backslash[r]}$
is only changing its sign at most once. To this end, we first note
that 
\begin{equation}
\Delta_{\ell}=\frac{\ell+1-r}{\sum_{i=1}^{\ell+1}\frac{1}{\lambda_{i}(\Sigma)}}-\frac{\ell-r}{\sum_{i=1}^{\ell}\frac{1}{\lambda_{i}(\Sigma)}}=\frac{\sum_{i=1}^{\ell}\frac{1}{\lambda_{i}(\Sigma)}-(\ell-r)\frac{1}{\lambda_{\ell+1}(\Sigma)}}{\left[\sum_{i=1}^{\ell+1}\frac{1}{\lambda_{i}(\Sigma)}\right]\left[\sum_{i=1}^{\ell}\frac{1}{\lambda_{i}(\Sigma)}\right]}.\label{eq: a fractional expression for the discrete derivative}
\end{equation}
Since the denominator of \eqref{eq: a fractional expression for the discrete derivative}
is strictly positive, it suffices to prove that the sequence comprised
of the numerator of \eqref{eq: a fractional expression for the discrete derivative},
to wit $\{\zeta_{\ell}\}_{\ell\in[d]\backslash[r]}$ with 
\[
\zeta_{\ell}\dfn\sum_{i=1}^{\ell}\frac{1}{\lambda_{i}(\Sigma)}-(\ell-r)\frac{1}{\lambda_{\ell+1}(\Sigma)},
\]
is only changing its sign at most once. Indeed, this claim is true
because $\zeta_{r}=\sum_{i=1}^{\ell}\frac{1}{\lambda_{i}(\Sigma)}>0$
and because $\{\zeta_{\ell}\}_{\ell\in[d]\backslash[r]}$ is a monotonic
non-increasing sequence, 
\[
\zeta_{\ell}-\zeta_{\ell+1}=(\ell-r+1)\left[\frac{1}{\lambda_{\ell+2}(\Sigma)}-\frac{1}{\lambda_{\ell+1}(\Sigma)}\right]\geq0.
\]
Therefore, $\{\zeta_{\ell}\}_{\ell\in[d]\backslash[r]}$ has at most
a single sign change (its has a positive value at $\ell=r$ and is
monotonically non-increasing with $\ell$ up to $\ell=d$), and so
is $\{\Delta_{\ell}\}_{\ell=r}^{d}$. The single sign change property
of the finite difference $\{\Delta_{\ell}\}_{\ell=r}^{d}$ is equivalent
to the fact that $\{a_{\ell}\}_{\ell=r}^{d}$ is \emph{unimodal}.
Thus, any local maximizer of $a_{\ell}$ is also a global maximizer. 
\end{proof}

\section{The Hilbert space MSE setting\label{sec:The-Hilbert-space}}

In this section, we show that the regret expressions in Section \ref{sec:The-linear-setting}
can be easily generalized to an infinite dimensional Hilbert space,
for responses with noise that is statistically independent of the
features. We still assume the MSE loss function (${\cal Y}=\mathbb{R}$
, and $\loss(y_{1},y_{2})=(y_{1}-y_{2})^{2}$), and that the predictor
is a linear function. However, we allow the the representation and
response function to be functions in a Hilbert space. As will be evident,
the resulting regret is not very different from the finite-dimensional
case. Formally, this is defined as follows:
\begin{defn}[The Hilbert space MSE setting]
\label{def: Hilbert MSE}Assume that $\boldsymbol{x}\sim P_{\boldsymbol{x}}$
is supported on a compact subset ${\cal X}\subset\mathbb{R}^{d}$,
and let $L_{2}(P_{\boldsymbol{x}})$ be the Hilbert space of functions
from ${\cal X}\to\mathbb{R}$ such that $\E[f^{2}(\boldsymbol{x})]=\int_{{\cal X}}f^{2}(\boldsymbol{x})\cdot\d P_{\boldsymbol{x}}<\infty$,
with the inner product, 
\[
\langle f,g\rangle\dfn\int_{{\cal X}}f(\boldsymbol{x})g(\boldsymbol{x})\cdot\d P_{\boldsymbol{x}}
\]
for $f,g\in L_{2}(P_{\boldsymbol{x}})$. Let $\{\phi_{j}(x)\}_{j=1}^{\infty}$
be an orthonormal basis for $L_{2}(P_{\boldsymbol{x}})$. 

A representation is comprised of a set of functions $\{\psi_{i}\}_{i\in[r]}\subset L_{2}(P_{\boldsymbol{x}})$,
$\psi_{i}\colon{\cal X}\to\mathbb{R}$, so that 
\[
{\cal R}\dfn\{\rep(x)=(\psi_{1}(x),\ldots,\psi_{r}(x))^{\top}\in\mathbb{R}^{r}\}.
\]
Let $\{\lambda_{j}\}_{j\in\mathbb{N}}$ be a positive monotonic non-increasing
sequence for which $\lambda_{j}\downarrow0$ as $j\to\infty$, and
let ${\cal F}$ be the set of functions from ${\cal X}\to\mathbb{R}$
such that given $f\in{\cal F}$, the response is given by 
\[
\boldsymbol{y}=f(\boldsymbol{x})+\boldsymbol{n}\in\mathbb{R}
\]
where
\[
f\in{\cal F}_{\{\lambda_{j}\}}\dfn\left\{ f(x)=\sum_{j=1}^{\infty}f_{j}\phi_{j}(x)\colon\{f_{j}\}_{j\in\mathbb{N}}\in\ell_{2}(\mathbb{N}),\quad\sum_{j=1}^{\infty}\frac{f_{j}^{2}}{\lambda_{j}}\leq1\right\} ,
\]
where $\boldsymbol{n}\in\mathbb{R}$ is a homoscedastic noise that
is statistically independent of $\boldsymbol{x}$ and satisfies $\E[\boldsymbol{n}]=0$.
Infinite-dimensional ellipsoids such as ${\cal F}_{\{\lambda_{j}\}}$
naturally arise in reproducing kernel Hilbert spaces (RKHS) \citep[Chapter 12]{wainwright2019high}
\citep[Chapter 16]{shalev2014understanding}, in which $\{\lambda_{j}\}$
is the eigenvalues of the kernel. In this case, the set ${\cal F}_{\{\lambda_{i}\}}=\{f\colon\|f\|_{{\cal H}}\leq1\}$
where $\|\cdot\|_{{\cal H}}$ is the norm of the RKHS ${\cal H}$.
For example, ${\cal H}$ could be the first-order Sobolev space of
functions with finite first derivative energy. 

Let the set of predictor functions be the set of linear functions
from $\mathbb{R}^{d}\to\mathbb{R}$, that is
\[
{\cal Q}\dfn\{\pre(z)=q^{\top}z=\sum_{i=1}^{r}q_{i}\cdot\psi_{i}(x),\;q\in\mathbb{R}^{r}\}.
\]
We denote the pure (resp. mixed) minimax regret as $\regret_{\pure}({\cal R},{\cal F}_{\{\lambda_{j}\}}\mid P_{\boldsymbol{x}})$
(resp. $\regret_{\mix}({\cal R},{\cal F}_{\{\lambda_{j}\}}\mid P_{\boldsymbol{x}})$).
We begin with pure strategies. 
\end{defn}

\begin{thm}
\label{thm: pure minimax regret orthogonal}For the Hilbert space
MSE setting (Definition \ref{def: Hilbert MSE})
\begin{equation}
\regret_{\pure}({\cal R},{\cal F}_{\{\lambda_{j}\}}\mid P_{\boldsymbol{x}})=\lambda_{r+1}.\label{eq: Orthogonal MSE pure minimax regret}
\end{equation}
A minimax representation is 
\[
\rep^{*}(x)=(\phi_{1}(x),\ldots,\phi_{r}(x))^{\top},
\]
and the worst case response function is $f^{*}=\sqrt{\lambda_{r+1}}\cdot\phi_{r+1}$. 
\end{thm}

We now turn to the minimax representation in mixed strategies.
\begin{thm}
\label{thm: mixed minimax regret orthogonal}For the Hilbert space
MSE setting (Definition \ref{def: Hilbert MSE})

\begin{equation}
\regret_{\mix}({\cal R},{\cal F}_{\{\lambda_{j}\}}\mid P_{\boldsymbol{x}})=\frac{\ell^{*}-r}{\sum_{i=1}^{\ell^{*}}\frac{1}{\lambda_{i}}},\label{eq: Orthogonal MSE mixed minimax regret}
\end{equation}
where $\ell^{*}$ is defined as \eqref{eq: condition on optimal rank of least favorable covariance matrix}
of Theorem \ref{thm: mixed minimax regret linear quadratic constraint}
(with the replacement $d\to\mathbb{N}_{+}$). Let $\{\boldsymbol{b}_{j}\}_{i=1}^{\infty}$
be an IID sequence of Rademacher random variables, $\P[\boldsymbol{b}_{i}=1]=\P[\boldsymbol{b}_{i}=-1]=1/2$.
Then, a least favorable prior $\boldsymbol{f}^{*}$ is 
\[
\boldsymbol{f}_{i}^{*}=\begin{cases}
\boldsymbol{b}_{i}\cdot\frac{1}{\sqrt{\sum_{i=1}^{\ell^{*}}\frac{1}{\lambda_{i}}}}, & 1\leq i\leq\ell_{*}\\
0, & i\geq\ell_{*}+1
\end{cases},
\]
and a law of minimax representation is to choose 
\[
\boldsymbol{\rep}^{*}(x)=\{\phi_{{\cal I}_{j}}(x)\}_{j=1}^{r}
\]
with probability $p_{j}$ , $j\in[{\ell^{*} \choose r}]$, defined
as in Theorem \ref{thm: mixed minimax regret linear quadratic constraint}.
\end{thm}

\paragraph*{Discussion}

Despite having countably infinite possible number of representations,
the optimal representation only utilizes a \emph{finite} set of orthogonal
functions, as determined by the radius of ${\cal F}_{\{s_{i}\}}$.
The proof of Theorems \ref{thm: pure minimax regret orthogonal} and
\ref{thm: mixed minimax regret orthogonal} is obtained by reducing
the infinite dimensional problem to a $d$-dimensional problem via
an approximation argument, then showing the the finite dimensional
case is similar to the problem of Section \ref{sec:The-linear-setting},
and then taking limit $d\uparrow\infty$. 

\subsection{Proofs}

Let us denote the $d$-dimensional \emph{slice} of ${\cal F}_{\{\lambda_{j}\}}$
by 
\[
{\cal F}_{\{\lambda_{j}\}}^{(d)}\dfn\left\{ f(x)\in{\cal F}_{\{\lambda_{j}\}}\colon f_{j}=0\text{ for all }j\geq d+1\right\} .
\]
Further, let us consider the restricted representation class, in which
the representation functions $\psi_{i}(t)$ belong to the span of
the first $d$ basis functions, that is 
\[
{\cal R}^{(d)}\dfn\{\rep(x)\in{\cal R}\colon=\psi_{i}(x)\in\Span(\{\phi_{i}\}_{i\in[d]})\text{ for all }i\in[r]\}.
\]
The following proposition implies that the regret in the infinite-dimensional
Hilbert space is obtained as the limit of finite-dimensional regrets,
as the one characterized in Section \ref{sec:The-linear-setting}:
\begin{prop}
\label{prop: approximation of orthogonal expansion with finite dimension}It
holds that 
\[
\regret_{\pure}({\cal R},{\cal F}_{\{\lambda_{j}\}}\mid P_{\boldsymbol{x}})=\lim_{d\uparrow\infty}\regret_{\pure}({\cal R}^{(d)},{\cal F}_{\{\lambda_{j}\}}^{(d)}\mid P_{\boldsymbol{x}})
\]
and
\[
\regret_{\mix}({\cal R},{\cal F}_{\{\lambda_{j}\}}\mid P_{\boldsymbol{x}})=\lim_{d\uparrow\infty}\regret_{\mix}({\cal R}^{(d)},{\cal F}_{\{\lambda_{j}\}}^{(d)}\mid P_{\boldsymbol{x}}).
\]
\end{prop}

\begin{proof}
Let $\{c_{ij}\}_{j\in\mathbb{N}}$ be the coefficients of the orthogonal
expansion of $\psi_{i}$, $i\in[r]$, that is, $\psi_{i}=\sum_{j=1}^{\infty}c_{ij}\phi_{j}$.
With a slight abuse of notation, we also let $c_{i}\dfn(c_{i1},c_{i2}\ldots)\in\ell_{2}(\mathbb{N})$.
We use a sandwich argument. On one hand, 
\begin{align}
\regret_{\pure}({\cal R},{\cal F}_{\{\lambda_{j}\}}\mid P_{\boldsymbol{x}}) & =\min_{\rep\in{\cal R}}\max_{f\in{\cal F}_{\{\lambda_{j}\}}}\regret(\rep,f)\\
 & \geq\min_{\rep\in{\cal R}}\max_{f\in{\cal F}_{\{\lambda_{j}\}}^{(d)}}\regret(\rep,f)\\
 & \trre[=,*]\min_{\rep\in{\cal R}^{(d)}}\max_{f\in{\cal F}_{\{\lambda_{j}\}}^{(d)}}\regret(\rep,f)\\
 & =\regret_{\pure}({\cal R}^{(d)},{\cal F}_{\{\lambda_{j}\}}^{(d)}\mid P_{\boldsymbol{x}}),\label{eq: orthogonal sandwich lower bound}
\end{align}
where $(*)$ follows from the following reasoning: For any $(\rep\in{\cal R},f\in{\cal F}_{\{\lambda_{j}\}}^{(d)})$,
\begin{align}
\regret(\rep,f)= & \min_{q\in\mathbb{R}^{r}}\E\left[\left(\sum_{j=1}^{d}f_{j}\phi_{j}(\boldsymbol{x})+\boldsymbol{n}-\sum_{j=1}^{\infty}\sum_{i=1}^{r}q_{i}c_{ij}\phi_{j}(\boldsymbol{x})\right)^{2}\right]-\E\left[\boldsymbol{n}^{2}\right]\\
 & \trre[=,a]\min_{q\in\mathbb{R}^{r}}\E\left[\left(\sum_{j=1}^{d}f_{j}\phi_{j}(\boldsymbol{x})-\sum_{j=1}^{\infty}\sum_{i=1}^{r}q_{i}c_{ij}\phi_{j}(\boldsymbol{x})\right)\right]\\
 & \trre[=,b]\min_{q\in\mathbb{R}^{r}}\sum_{j=1}^{d}\left(f_{j}-\sum_{i=1}^{r}q_{i}c_{ij}\right)^{2}+\sum_{j=d+1}^{\infty}\left(\sum_{i=1}^{r}q_{i}c_{ij}\right)^{2},
\end{align}
where here $(a)$ follows since the noise $\boldsymbol{n}$ is independent
of $\boldsymbol{x}$, and since, similarly to the finite-dimensional
case (Section \ref{sec:The-linear-setting}), the prediction loss
based on the features $x\in{\cal X}$ is $\E[\boldsymbol{n}^{2}]$,
for any given $f\in{\cal F}$, $(b)$ follows from Parseval's identity
and the orthonormality of $\{\phi_{j}\}_{j\in\mathbb{N}}$. So, 
\begin{align}
 & \min_{\rep\in{\cal R}}\max_{f\in{\cal F}_{\{\lambda_{j}\}}^{(d)}}\regret(\rep,f)\nonumber \\
 & =\min_{\{c_{ij}\}_{i\in[r],j\in\mathbb{N}}}\max_{f\in{\cal F}_{\{\lambda_{j}\}}^{(d)}}\min_{q\in\mathbb{R}^{r}}\sum_{j=1}^{d}\left(f_{j}-\sum_{i=1}^{r}q_{i}c_{ij}\right)^{2}+\sum_{j=d+1}^{\infty}\left(\sum_{i=1}^{r}q_{i}c_{ij}\right)^{2}.
\end{align}
Evidently, since $\sum_{j=d+1}^{\infty}(\sum_{i=1}^{r}q_{i}c_{ij})^{2}\geq0$,
an optimal representation may satisfy that $c_{ij}=0$ for all $j\geq d+1$.
Thus, the optimal representation belongs to ${\cal R}^{(d)}$. 

On the other hand,
\begin{align}
\regret_{\pure}({\cal R},{\cal F}_{\{\lambda_{j}\}}\mid P_{\boldsymbol{x}}) & =\min_{\rep\in{\cal R}}\max_{f\in{\cal F}_{\{\lambda_{j}\}}}\regret(\rep,f)\\
 & \leq\min_{\rep\in{\cal R}^{(d)}}\max_{f\in{\cal F}_{\{\lambda_{j}\}}}\regret(\rep,f)\\
 & \trre[\leq,*]\min_{\rep\in{\cal R}^{(d)}}\max_{f\in{\cal F}_{\{\lambda_{j}\}}^{(d)}}\regret(\rep,f)+\lambda_{d+1}\\
 & =\regret_{\pure}({\cal R}^{(d)},{\cal F}_{\{\lambda_{j}\}}^{(d)}\mid P_{\boldsymbol{x}})+\lambda_{d+1},\label{eq: orthogonal sandwich upper bound}
\end{align}
where $(*)$ follows from the following reasoning: For any $(\rep\in{\cal R}^{(d)},f\in{\cal F}_{\{\lambda_{j}\}})$,
\begin{align}
\regret(\rep,f)= & \min_{q\in\mathbb{R}^{r}}\E\left[\left(\sum_{j=1}^{\infty}f_{j}\phi_{j}(\boldsymbol{x})+\boldsymbol{n}-\sum_{j=1}^{\infty}\sum_{i=1}^{r}q_{i}c_{ij}\phi_{j}(\boldsymbol{x})\right)^{2}\right]-\E[\boldsymbol{n}^{2}]\\
 & \trre[=,a]\min_{q\in\mathbb{R}^{r}}\sum_{j=1}^{d}\left(f_{j}-\sum_{i=1}^{r}q_{i}c_{ij}\right)^{2}+\sum_{j=d+1}^{\infty}f_{j}^{2}\\
 & \trre[\leq,b]\min_{q\in\mathbb{R}^{r}}\sum_{j=1}^{d}\left(f_{j}-\sum_{i=1}^{r}q_{i}c_{ij}\right)^{2}+\lambda_{d+1},
\end{align}
where $(a)$ follows similarly to the analysis made in the previous
step, and $(b)$ follows since for any $f\in{\cal F}_{\{\lambda_{j}\}}$
it holds that 
\[
\sum_{j=d+1}^{\infty}f_{j}^{2}\leq\lambda_{d+1}\sum_{j=d+1}^{\infty}\frac{f_{j}^{2}}{\lambda_{j}}\leq\lambda_{d+1}\sum_{j=1}^{\infty}\frac{f_{j}^{2}}{\lambda_{j}}\leq\lambda_{d+1}.
\]
Combining \eqref{eq: orthogonal sandwich lower bound} and \eqref{eq: orthogonal sandwich upper bound}
and using $\lambda_{d+1}\downarrow0$ completes the proof for the
pure minimax regret. The proof for the mixed minimax is analogous
and thus is omitted. 
\end{proof}
We also use the following simple and technical lemma.
\begin{lem}
\label{lem: LS solution for predictor orthogonal expansion}For $\rep\in{\cal R}^{(d)}$
and $f\in{\cal F}^{(d)}$\footnote{Note that any $f\in{\cal F}^{(d)}$ may be uniquely identified with
a $d$-dimensional vector $f\in\mathbb{R}^{d}$. With a slight abuse
of notation we do not distinguish between the two. }
\begin{align}
\regret(R,f) & =f^{\top}\left(I_{d}-R^{\top}(RR^{\top})^{-1}R\right)f,
\end{align}
where $R\in\mathbb{R}^{r\times d}$ is the matrix of coefficients
of the orthogonal expansion of $\psi_{i}=\sum_{j=1}^{d}c_{ij}\phi_{j}$
for $i\in[r]$, so that $R(i,j)=c_{ij}$. 
\end{lem}

\begin{proof}
It holds that
\begin{align}
\regret(\rep,f) & =\min_{q\in\mathbb{R}^{r}}\E\left[\left(\sum_{j=1}^{d}f_{j}\phi_{j}(\boldsymbol{x})+\boldsymbol{n}-\sum_{i=1}^{r}q_{i}\sum_{j=1}^{d}c_{ij}\phi_{j}(\boldsymbol{x})\right)^{2}\right]-\E\left[\boldsymbol{n}^{2}\right]\\
 & =\min_{q\in\mathbb{R}^{r}}\E\left[\sum_{j=1}^{d}\left(f_{j}-\sum_{i=1}^{r}q_{i}c_{ij}\right)\phi_{j}(\boldsymbol{x})\right]\\
 & =\min_{q\in\mathbb{R}^{r}}\sum_{j=1}^{d}\left(f_{j}-\sum_{i=1}^{r}q_{i}c_{ij}\right)^{2}\\
 & =\min_{q\in\mathbb{R}^{r}}\sum_{j=1}^{d}\left[f_{j}^{2}-2f_{j}\sum_{i=1}^{r}q_{i}c_{ij}+\sum_{i_{1}=1}^{r}\sum_{i_{2}=1}^{r}q_{i_{1}}c_{i_{1}j}q_{i_{2}}c_{i_{2}j}\right]\\
 & =\min_{q\in\mathbb{R}^{r}}f^{\top}f-2q^{\top}Rf+q^{\top}RR^{\top}q\\
 & =f^{\top}\left(I_{d}-R^{\top}(RR^{\top})^{-1}R\right)f,
\end{align}
where the last equality is obtained by the minimizer $q^{*}=(RR^{\top})^{-1}Rf$.
\end{proof}
\begin{proof}[Proof of Theorems \ref{thm: pure minimax regret orthogonal} and \ref{thm: mixed minimax regret orthogonal}]
 By Proposition \ref{prop: approximation of orthogonal expansion with finite dimension},
we may first consider the finite dimensional case, and then take the
limit $d\uparrow\infty$. By Lemma \ref{lem: LS solution for predictor orthogonal expansion},
in the $d$-dimensional case (for both the representation and the
response function), the regret is formally as in the linear setting
under the MSE of Theorem \ref{thm: pure minimax regret linear quadratic constraint},
by setting therein $\Sigma_{\boldsymbol{x}}=I_{d}$, and $S=\diag(\lambda_{1},\ldots,\lambda_{d})$
(c.f. Lemma \ref{lem: LS solution for predictor}). The claim of the
Theorem \ref{thm: pure minimax regret orthogonal} then follows by
taking $d\uparrow\infty$ and noting that $\lambda_{d+1}\downarrow0$.
The proof of Theorem \ref{thm: mixed minimax regret orthogonal} is
analogous and thus omitted. 
\end{proof}

\section{Iterative algorithms for the Phase 1 and Phase 2 problems \label{sec:Iterative-algorithms-for}}

In this section, we describe our proposed algorithms for solving the
Phase 1 and Phase 2 problems of Algorithm \ref{alg: Iterative algorithm}.
Those algorithms are general, and only require providing gradients
of the regret function \eqref{eq: pointwise regret} and an initial
representation and a set of adversarial functions. These are individually
determined for each setting. See Section \ref{sec:Details-for-the}
for the way these are determined in Examples \ref{exa: Algorithm linear MSE}
and \ref{exa: Algorithm linear logistic}. 

\subsection{Phase 1: finding a new adversarial function }

We propose an algorithm to solve the Phase 1 problem \eqref{eq: phase 1 iterative alg},
which is again based on an iterative algorithm. We denote the function's
value at the $t$th iteration by $f_{(t)}$. The proposed Algorithm
\ref{alg: phase 1 sol} operates as follows. At initialization, the
function $f_{(1)}\in{\cal F}$ is arbitrarily initialized (say at
random), and then the optimal predictor $\pre^{(j)}$ is found for
each of the $k$ possible representations $\rep^{(j)}$, $j\in[k]$.
Then, the algorithm iteratively repeats the following steps, starting
with $t=2$: (1) Updating the function from $f_{(t-1)}$ to $f_{(t)}$
based on a gradient step of 
\[
\sum_{j\in[k]}p^{(j)}\cdot\E\left[\loss(f_{(t-1)}(\boldsymbol{x}),\pre^{(j)}(\rep^{(j)}(\boldsymbol{x})))\right],
\]
that is, the weighted loss function of the previous iteration function,
which is then followed by a projection to the feasible class of functions
${\cal F}$, denoted as $\Pi_{{\cal F}}(\cdot)$ (2) Finding the optimal
predictor $\pre^{(j)}$ for the current function $f_{(t)}$ and the
given representations $\{\rep^{(j)}\}_{j\in[k]}$, and computing the
respective loss for each representation, 
\[
L^{(j)}:=\E\left[\loss(f_{(t)}(\boldsymbol{x}),\pre^{(j)}(\rep^{(j)}(\boldsymbol{x})))\right].
\]
This loop iterates for $T_{f}$ iterations, or until convergence. 

\begin{algorithm*}
\begin{algorithmic}[1]

\Procedure{Phase 1 solver}{$\{\rep^{(j)},p^{(j)}\}_{j\in[k]},{\cal F},{\cal Q},d,r,P_{\boldsymbol{x}}$}

\State  \textbf{begin}

\State  Initialize\textbf{ }$T_{f}$ \Comment{Number of iterations
parameters}

\State  Initialize\textbf{ }$\eta_{f}$ \Comment{Step size parameter}

\State  Initialize $f_{(1)}\in{\cal F}$ \Comment{Function initialization,
e.g., at random}

\For{$j=1$ to $k$}

\State \textbf{set} $\pre^{(j)}\leftarrow\argmin_{\pre\in{\cal Q}}\E\left[\loss(f_{(1)}(\boldsymbol{x}),\pre(\rep^{(j)}(\boldsymbol{x})))\right]$

\EndFor 

\For{$t=2$ to $T_{f}$}

\State \textbf{update} $f_{(t-1/2)}=f_{(t-1)}+\eta_{f}\cdot\sum_{j\in[k]}p_{(t-1)}^{(j)}\cdot\nabla_{f}\E\left[\loss(f_{(t-1)}(\boldsymbol{x}),\pre^{(j)}(\rep^{(j)}(\boldsymbol{x})))\right]$
\newline\Comment{A gradient update of the function}

\State  \textbf{project }$f_{(t)}=\Pi_{{\cal F}}(f_{(t-1/2)})$ \Comment{Projection
on the class ${\cal F}$}

\For{$j=1$ to $k$}

\State \textbf{set} $\pre^{(j)}\leftarrow\argmin_{\pre\in{\cal Q}}\E\left[\loss(f_{(t)}(\boldsymbol{x}),\pre(\rep^{(j)}(\boldsymbol{x})))\right]$
\Comment{Update of predictors}

\State \textbf{set} $L^{(j)}\leftarrow\E\left[\loss(f_{(t)}(\boldsymbol{x}),\pre^{(j)}(\rep^{(j)}(\boldsymbol{x})))\right]$\Comment{Compute
loss of each representation}

\EndFor

\EndFor

\State \textbf{return $f_{(T)}$, and the regret $\sum_{j\in[k]}p^{(j)}\cdot L^{(j)}$}

\EndProcedure

\end{algorithmic}

\caption{A procedure for finding a new function via the solution of \eqref{eq: phase 1 iterative alg}
\label{alg: phase 1 sol}}
\end{algorithm*}

\paragraph*{Design choices and possible variants of the basic algorithm}

At initialization, we have chosen a simple random initialization for
$f_{(1)}$, but it may also be initialized based on some prior knowledge
of the adversarial function. For the update of the predictors, we
have specified a full computation of the optimal predictor, which
can be achieved in practice by running another iterative algorithm
such as stochastic gradient descent (SGD) until convergence. If this
is too computationally expensive, the number of gradient steps may
be limited. The update of the function is done via projected SGD with
a constant step size $\eta_{f}$, yet it is also possible to modify
the step size with the iteration, e.g., the common choice $\eta_{f}/\sqrt{t}$
at step $t$ \citep{hazan2016introduction}. Accelerated algorithms,
e.g., moment-based, may also be deployed. 

\paragraph*{Convergence analysis }

A theoretical analysis of the convergence properties of the algorithm
appears to be challenging. Evidently, this is a minimax game between
the response player and a player cooperating with the representation
player, which optimizes the prediction rule in order to minimize the
loss. This is, however, not a concave-convex game. As described in
Appendix \ref{sec:Additional-related-work}, even concave-convex games
are not well understood at this point. We thus opt to validate this
algorithm numerically. 

\paragraph*{Running-time complexity analysis}

Algorithm \ref{alg: phase 1 sol} runs for a fixed number of iterations
$T_{f}$, accepts $k$ representations, and makes $kT_{f}$ updates.
Each update is comprised from a gradient step of for the adversarial
function (cost $C_{1})$, and optimization of the predictor (cost
$C_{2})$. So the total computation complexity is $kT_{f}\cdot(C_{1}+C_{2})$.
The most expensive part is the optimization of the predictor $C_{2}$,
and this can be significantly reduced by running a few gradients steps
of the predictor instead of a full optimization. If we take $g$ gradient
steps then $C_{2}$ is replaced by $C_{1}g$ and the total computational
cost is $kT_{f}(g+1)C_{1}$. 

\subsection{Phase 2: finding a new representation }

We propose an iterative algorithm to solve the Phase 2 problem \eqref{eq: phase 2 iterative alg},
and thus finding a new representation $\rep^{(k+1)}$. To this end,
we first note that the objective function in \eqref{eq: phase 2 iterative alg}
can be separated into a part that depends on existing representations
and a part that depends on the new one, specifically, as
\begin{align}
 & \sum_{j_{1}\in[k]}\sum_{j_{2}\in[m_{0}+k]}p^{(j_{1})}\cdot o^{(j_{2})}\cdot\E\left[\loss(f^{(j_{2})}(\boldsymbol{x}),\pre^{(j_{1},j_{2})}(\rep^{(j_{1})}(\boldsymbol{x})))\right]\nonumber \\
 & \hphantom{===}+\sum_{j_{2}\in[m_{0}+k]}p^{(k+1)}\cdot o^{(j_{2})}\cdot\E\left[\loss(f^{(j_{2})}(\boldsymbol{x}),\pre^{(k+1,j_{2})}(\rep^{(k+1)}(\boldsymbol{x})))\right]\nonumber \\
 & =\sum_{j_{1}\in[k]}\sum_{j_{2}\in[m_{0}+k]}p^{(j_{1})}\cdot o^{(j_{2})}\cdot L^{(j_{1},j_{2})}\nonumber \\
 & +\sum_{j_{2}\in[m_{0}+k]}p^{(k+1)}\cdot o^{(j_{2})}\cdot\E\left[\loss(f^{(j_{2})}(\boldsymbol{x}),\pre^{(k+1,j_{2})}(\rep^{(k+1)}(\boldsymbol{x})))\right],\label{eq: phase 2 objective decomposition to old and new}
\end{align}
where 
\[
L^{(j_{1},j_{2})}:=\E\left[\loss(f^{(j_{2})}(\boldsymbol{x}),\pre^{(j_{1},j_{2})}(\rep^{(j_{2})}(\boldsymbol{x})))\right],
\]
and the predictors $\{\pre^{(j_{1},j_{2})}\}_{j_{1}\in[k],j_{2}\in[m_{0}+k]}$
can be optimized independently of the new representation $\rep^{(k+1)}$.
We propose an iterative algorithm for this problem, and denote the
new representation at the $t$th iteration of the algorithm by $\rep_{(t)}^{(k+1)}$.
The algorithm's input is a set of $m_{0}+k$ adversarial functions
$\{f^{(i)}\}_{i\in[m_{0}+k]}$, and the current set of representations
$\{\rep^{(j)}\}_{j\in[k]}$. Based on these, the algorithm may find
the optimal predictor for $f^{(j_{2})}$ based on the representation
$\rep^{(j_{1})}$, and thus compute the loss 
\[
L_{*}^{(j_{1},j_{2})}:=\min_{\pre\in{\cal Q}}\E\left[\loss(f^{(j_{2})}(\boldsymbol{x}),\pre(\rep^{(j_{1})}(\boldsymbol{x})))\right]
\]
for $j_{1}\in[k]$ and $j_{2}\in[m_{0}+k]$. In addition, the new
representation is arbitrarily initialized (say, at random) as $\rep_{(1)}^{(k+1)}$,
and the predictors $\{\pre_{(1)}^{(k+1,j_{2})}\}_{j_{2}\in[m_{0}+k]}$
are initialized as the optimal predictors for $f^{(j_{2})}$ given
the representation $\rep_{(1)}^{(k+1)}$. The algorithm keeps track
of weights for the representations (including the new one), which
are initialized uniformly, i.e., $p_{(1)}^{(j_{1})}=\frac{1}{k+1}$
for $j_{1}\in[k+1]$ (including a weight for the new representation).
The algorithm also keeps track of weights for the functions, which
are also initialized uniformly as $o_{(1)}^{(j_{2})}=\frac{1}{m_{0}+k}$
for $j_{2}\in[m_{0}+k]$. Then, the algorithm iteratively repeats
the following steps, starting with $t=2$: (1) Updating the new representation
from $\rep_{(t-1)}^{(k+1)}$ to $\rep_{(t)}^{(k+1)}$ based on a gradient
step of the objective function \eqref{eq: phase 2 iterative alg}
as a function of $\rep^{(k+1)}$. Based on the decomposition in \eqref{eq: phase 2 objective decomposition to old and new}
the term of the objective which depends on $\rep^{(k+1)}$ is 
\[
p_{(t-1)}^{(k+1)}\sum_{j_{2}\in[m_{0}+k]}o_{(t-1)}^{(j_{2})}\cdot\E\left[\loss(f^{(j_{2})}(\boldsymbol{x}),\pre^{(k+1,j_{2})}(\rep^{(k+1)}(\boldsymbol{x})))\right],
\]
that is, the loss function of the previous iteration new representation,
weighted according to the current function weights $o_{(t-1)}^{(j_{2})}$.
Since the multiplicative factor $p_{(t-1)}^{(k+1)}$ is common to
all terms, it is removed from the gradient computation (this aids
in the choice of the gradient step). This gradient step is then possibly
followed by normalization or projection, which we denote by the operator
$\Pi_{{\cal R}}(\cdot)$. For example, in the linear case, it make
sense to normalize $\rep^{(k+1)}$ to have unity norm (in some matrix
norm of choice). After updating the new representation to $\rep_{(t)}^{(k+1)}$,
optimal predictors are found for each function, the loss is computed
\[
L_{(t)}^{(k+1,j_{2})}:=\min_{\pre\in{\cal Q}}\E\left[\loss(f^{(j_{2})}(\boldsymbol{x}),\pre(\rep_{(t)}^{(k+1)}(\boldsymbol{x})))\right]
\]
for all $j_{2}\in[m_{0}+k]$, and the optimal predictor is updated
to $\{\pre_{(t)}^{(k+1,j_{2})}\}_{j_{2}\in[m_{0}+k]}$ based on this
solution. (2) Given the current new representation $\rep_{(t)}^{(k+1)}$,
the loss matrix 
\[
\{L_{(t)}^{(j_{1},j_{2})}\}_{j_{1}\in[k],j_{2}\in[m_{0}+k]}
\]
is constructed where for $j_{1}\in[k]$ it holds that $L_{(t)}^{(j_{1},j_{2})}=L^{(j_{1},j_{2})}$
for all $t$ (i.e., the loss of previous representations and functions
is kept fixed). This is considered to be the loss matrix of a two-player
zero-sum game between the representation player and the function player,
where the representation player has $k+1$ possible strategies and
the function player has $m_{0}+k$ strategies. The weights $\{p_{(t)}^{(j_{1})}\}_{j_{1}\in[k+1]}$
and $\{o_{(t)}^{(j_{2})}\}_{j_{2}\in[m_{0}+k]}$ are then updated
according to the MWU rule. Specifically, for an \emph{inverse temperature
parameter} $\beta$ (or a \emph{regularization parameter}), the update
is given by 
\begin{equation}
p_{(t)}^{(j)}=\frac{p_{(t-1)}^{(j)}\cdot\beta^{L^{(j)}}}{\sum_{\tilde{j}\in[k]}p_{(t-1)}^{(\tilde{j})}\cdot\beta^{L^{(\tilde{j})}}}\label{eq: MUW for phase 1}
\end{equation}
for the representation weights and, analogously, by 
\[
o_{(t)}^{(j)}=\frac{o_{(t-1)}^{(j)}\cdot\beta^{-L^{(j)}}}{\sum_{\tilde{j}\in[k]}o_{(t-1)}^{(\tilde{j})}\cdot\beta^{-L^{(\tilde{j})}}}
\]
for the function weights (as the function player aims to maximize
the loss). This can be considered as a regularized gradient step on
the probability simplex, or more accurately, a \emph{follow-the-regularized-leader}
\citep{hazan2016introduction}. The main reasoning of this algorithm
is that at each iteration the weights $\{p^{(j)}\}_{j\in[k+1]}$ and
$\{o^{(j)}\}_{j\in[m_{0}+k]}$ are updated towards the solution of
the two-player zero-sum game with payoff matrix $\{-L_{(t)}^{(j_{1},j_{2})}\}_{j_{1}\in[k+1],j_{2}\in[m_{0}+k]}$.
In turn, based only on the function weights $\{o^{(j)}\}_{j\in[m_{0}+k]}$,
the new representation is updated to $\rep_{(t)}^{(k+1)}$, which
then changes the pay-off matrix at the next iteration. It is well
known that the MWU solved two-player zero-sum game \citep{freund1999adaptive},
in which the representation player can choose the weights and the
function player can choose the function. 

This loop iterates for $T_{\text{stop}}$ iterations, and then the
optimal weights are given by the average over the last $T_{\text{avg}}$
iterations \citep{freund1999adaptive}, i.e., 
\[
p_{*}^{(j)}=\frac{1}{T_{\text{avg}}}\sum_{t=T_{\text{stop}}-T_{\text{avg}}+1}^{T_{\text{stop}}}p_{(t)}^{(j)},
\]
and 
\[
o_{*}^{(j)}=\frac{1}{T_{\text{avg}}}\sum_{t=T_{\text{stop}}-T_{\text{avg}}+1}^{T_{\text{stop}}}o_{(t)}^{(j)}.
\]
In the last $T_{\rep}-T_{\text{stop}}$ iterations, only the representation
$R_{(t)}^{(k+1)}$ and the predictors are updated. The algorithm then
outputs $\rep_{(T)}^{(k+1)}$ as the new representation and the weights
$\{p_{*}^{(j)}\}_{j\in[k+1]}$. 

\begin{algorithm*}
\begin{algorithmic}[1]

\small

\Procedure{Phase 2 solver}{$\{\rep^{(j_{1})}\}_{j\in[k]},\{f^{(j_{2})}\}_{j_{2}\in[m_{0}+k]},{\cal R},{\cal F},{\cal Q},d,r,P_{\boldsymbol{x}}$}

\State  \textbf{Begin}

\State  Initialize\textbf{ }$T_{\rep},T_{\text{stop}},T_{\text{avg}}$
\Comment{Number of iterations parameters}

\State  Initialize\textbf{ }$\eta_{\rep}$ \Comment{Step size parameter}

\State  Initialize $\beta\in(0,1)$ \Comment{Inverse temperature
parameter}

\State  Initialize $f_{(1)}\in{\cal F}$ \Comment{Function initialization,
e.g., at random}

\State  Initialize $p_{(1)}^{(j)}\leftarrow0$ for $j\in[k]$ and
$p_{(1)}^{(k+1)}\leftarrow0$ \Comment{A uniform weight initialization
for the representations}

\State  Initialize $o_{(1)}^{(j_{2})}\leftarrow\frac{1}{m_{0}+k}$
for $j_{2}\in[k]$ \Comment{A uniform weight initialization for the
functions}

\For{$j_{1}=1$ to $k$}

\For{$j_{2}=1$ to $m_{0}+k$}

\State \textbf{Set} $\pre^{(j_{1},j_{2})}\leftarrow\argmin_{\pre\in{\cal Q}}\E\left[\loss(f^{(j_{2})}(\boldsymbol{x}),\pre(\rep^{(j_{1})}(\boldsymbol{x})))\right]$\newline\Comment{Optimal
predictors for existing representations and input functions }

\State \textbf{Set} $L^{(j_{1},j_{2})}\leftarrow\min_{\pre\in{\cal Q}}\E\left[\loss(f^{(j_{2})}(\boldsymbol{x}),\pre^{(j_{1},j_{2})}(\rep^{(j_{1})}(\boldsymbol{x})))\right]$\Comment{The
minimal loss}

\EndFor 

\EndFor 

\For{$j_{2}=1$ to $m_{0}+k$}

\State \textbf{Initialize} $\rep_{(1)}^{(k+1)}$ \Comment{Arbitrarily,
e.g., at random }

\State \textbf{Set} $\pre_{(1)}^{(k+1,j_{2})}\leftarrow\argmin_{\pre\in{\cal Q}}\E\left[\loss(f^{(j_{2})}(\boldsymbol{x}),\pre(\rep^{(k+1)}(\boldsymbol{x})))\right]$
for $j_{2}\in[m_{0}+k]$\newline\Comment{Optimal predictors for new
representation and input functions}

\EndFor 

\For{$t=2$ to $T_{\rep}$}

\State \textbf{Update} \Comment{A gradient update of the new representation}
\[
\rep_{(t-1/2)}^{(k+1)}=\rep_{(t-1)}^{(k+1)}+\eta_{\rep}\cdot\sum_{j_{2}\in[m_{0}+k]}o_{(t-1)}^{(j_{2})}\cdot\nabla_{\rep^{(k+1)}}\E\left[\loss(f^{(j_{2})}(\boldsymbol{x}),\pre^{(k+1,j_{2})}(\rep_{(t-1)}^{(k+1)}(\boldsymbol{x})))\right]
\]

\State \textbf{Project }$\rep_{(t)}^{(k+1)}=\Pi_{{\cal R}}(\rep_{(t-1/2)}^{(k+1)})$
\Comment{Standardization based on the class ${\cal R}$}

\For{$j=1$ to $k$}

\State \textbf{Set} $\pre^{(k+1,j_{2})}\leftarrow\argmin_{\pre\in{\cal Q}}\E\left[\loss(f^{(j_{2})}(\boldsymbol{x}),\pre(\rep_{(t)}^{(k+1)}(\boldsymbol{x})))\right]$
\newline\Comment{Update of predictors for the new representation}

\State $L_{(t)}^{(k+1,j_{2})}\leftarrow\E\left[\loss((f^{(j_{2})}(\boldsymbol{x}),\pre^{(k+1,j_{2})}(\rep_{(t)}^{(k+1)}(\boldsymbol{x})))\right]$
\Comment{Compute loss}

\EndFor

\State \textbf{Set $L_{(t)}^{(j_{1},j_{2})}\leftarrow L^{(j_{1},j_{2})}$
}for $j_{1}\in[k]$ and $j_{2}\in[m_{0}+k]$

\If{$t<T_{\text{stop}}$}

\State \textbf{Update} $p_{(t)}^{(j)}\leftarrow\nicefrac{p_{(t-1)}^{(j)}\cdot\beta^{L^{(j)}}}{\sum_{\tilde{j}\in[k]}p_{(t-1)}^{(\tilde{j})}\cdot\beta^{L^{(\tilde{j})}}}$
for $j\in[k]$ \Comment{A MWU}

\State \textbf{Update} $o_{(t)}^{(j)}\leftarrow\nicefrac{o_{(t-1)}^{(j)}\cdot\beta^{-L^{(j)}}}{\sum_{\tilde{j}\in[m_{0}+k]}o_{(t-1)}^{(\tilde{j})}\cdot\beta^{-L^{(\tilde{j})}}}$
for $j\in[m_{0}+k]$ \Comment{A MWU}

\ElsIf{$t=T_{\text{stop}}$}

\State \textbf{Update} $p_{(t)}^{(j)}=p_{(t)}^{(j)}\leftarrow\frac{1}{T_{\text{avg}}}\sum_{t=T_{\text{stop}}-T_{\text{avg}}+1}^{T_{\text{stop}}}p_{(t)}^{(j)}$
for $j\in[k]$ \newline\Comment{Optimal weights by averaging last
$T_{\text{avg}}$ iterations}

\State \textbf{Update} $o_{(t)}^{(j)}\leftarrow\frac{1}{T_{\text{avg}}}\sum_{t=T_{\text{stop}}-T_{\text{avg}}+1}^{T_{\text{stop}}}o_{(t)}^{(j)}$
for $j\in[m_{0}+k]$ \newline\Comment{Optimal weights by averaging
last $T_{\text{avg}}$ iterations}

\Else 

\State \textbf{Update} $p_{(t)}^{(j)}\leftarrow p_{(t-1)}^{(j)}$
for $j\in[k]$ \Comment{No update for the last $T-T_{\text{stop}}$
iterations}

\State \textbf{Update} $o_{(t)}^{(j)}\leftarrow o_{(t-1)}^{(j)}$
for $j\in[m_{0}+k]$ \Comment{No update for the last $T-T_{\text{stop}}$
iterations}

\EndIf

\State \textbf{Return $\rep_{(T)}^{(k+1)}$ and $\{p_{(T_{\rep})}^{(j)}\}_{j\in[k+1]}$}

\EndFor

\EndProcedure

\end{algorithmic}

\caption{A procedure for finding a new representation $\rep^{(k+1)}$ via the
solution of \eqref{eq: phase 2 iterative alg} \label{alg: phase 2 sol}}
\end{algorithm*}

\paragraph*{Design choices and possible variants of the basic algorithm}

At initialization, we have chosen a simple random initialization for
$\rep_{(1)}^{(k+1)}$, but it may also be initialized based on some
prior knowledge of the desired new representation. The initial predictors
$\{\pre_{(1)}^{(k+1,j_{2})}\}_{j_{2}\in[m_{0}+k]}$ will then be initialized
as the optimal predictors for $\rep_{(1)}^{(k+1)}$ and $\{f^{(j_{2})}\}_{j_{2}\in[m_{0}+k]}$.
We have initialized the representation and function weights uniformly.
A possibly improved initialization for the function weights is to
put more mass on the more recent functions, that is, for large values
of $j_{2}$, or to use the minimax strategy of the function player
in the two-player zero-sum game with payoff matrix $\{-L_{(t)}^{(j_{1},j_{2})}\}_{j_{1}\in[k],j_{2}\in[m_{0}+k]}$
(that is, a game which does not include the new representation). As
in the Phase 1 algorithm, the gradient update of the new representation
can be replaced by a more sophisticated algorithm, the computation
of the optimal predictors can be replaced with (multiple) update steps,
and the step size may also be adjusted. For the MWU update, we use
the proposed scaling proposed by \citet{freund1999adaptive}
\[
\beta=\frac{1}{1+\sqrt{\frac{c\ln m}{T}}}
\]
for some constant $c$. It is well known that using the last iteration
of a MWU algorithm may fail \citep{bailey2018multiplicative}, while
averaging the weights value of all iterations provides the optimal
value of a two-player zero-sum games \citep{freund1999adaptive}.
For improved accuracy, we compute the average weights over the last
$T_{\text{avg}}$ iterations (thus disregarding the initial iterations).
We then halt the weights update and let the function and predictor
update to run for $T-T_{\text{stop}}$ iterations in order to improve
the convergence of $R^{(k+1})$. Finally, the scheduling of the steps
may be more complex, e.g., it is possible that running multiple gradient
steps follows by multiple MWU steps may improve the result. 

\paragraph*{Running-time complexity analysis}

Algorithm \ref{alg: phase 2 sol} is more complicated than Algorithm
\ref{alg: phase 1 sol}, but the computational complexity analysis
is similar. It runs for $T_{R}$ iterations and the total cost is
roughly on the order of $T_{R}k^{2}gC_{1}$ (taking $g$ gradient
steps for the predictor optimization; $k^{2}$ is the number of representations,
and is controlled by the learner; $C_{1}$ is determined by the computer,
and $g$ should be large enough to assure quality results). 

\section{Details for the examples of Algorithm \ref{alg: Iterative algorithm}
\label{sec:Details-for-the} and additional experiments}

As mentioned, the solvers of the Phase 1 and Phase 2 problems of Algorithm
\ref{alg: Iterative algorithm} require the gradients of the regret
\eqref{eq: pointwise regret} as inputs, as well as initial representation
and set of adversarial functions. We next provide these details for
the examples in Section \ref{sec:An-iterative-algorithm}. The code
for the experiments was written in \texttt{Python 3.6} and is available
at \href{https://github.com/Hannibal96/paper-optimization-algorithm/tree/main}{this link}.
The optimization of hyperparameters was done using the \texttt{Optuna}
library. The hardware used is standard and detailed appear in Table
\ref{tab:Hardware-details}. 

\begin{table}

\caption{Hardware details \label{tab:Hardware-details}}

\begin{centering}
\begin{tabular}{ccc}
\hline 
CPU & RAM & GPU\tabularnewline
\hline 
\hline 
Intel i9 13900k & 64GB & RTX 3090 Ti\tabularnewline
\hline 
\end{tabular}
\par\end{centering}
\end{table}

\subsection{Details for Example \ref{exa: Algorithm linear MSE}: the linear
MSE setting}

In this setting, the expectation over the feature distribution can
be carried out analytically, and the regret is given by 
\begin{align}
\regret(\rep,f\mid\Sigma_{\boldsymbol{x}}) & =\E\left[\left(f^{\top}\boldsymbol{x}-q^{\top}R^{\top}\boldsymbol{x}\right)^{2}\right]\\
 & =f^{\top}\Sigma_{\boldsymbol{x}}f-2q^{\top}R^{\top}\Sigma_{\boldsymbol{x}}f+q^{\top}R^{\top}Rq.
\end{align}
The regret only depends on the feature distribution $P_{\boldsymbol{x}}$
via $\Sigma_{\boldsymbol{x}}$. For each run of the algorithm, the
covariance matrix $\Sigma_{\boldsymbol{x}}$ was chosen to be diagonal
with elements drawn from a log-normal distribution, with parameters
$(0,\sigma_{0})$, and $S=I_{d}$. 

\paragraph{Regret gradients }

The gradient of the regret w.r.t. the function $f$ is given by
\[
\nabla_{f}\E\left[\left(f^{\top}\boldsymbol{x}-q^{\top}R^{\top}\boldsymbol{x}\right)^{2}\right]=2f^{\top}\Sigma_{\boldsymbol{x}}-2q^{\top}R^{\top}\Sigma_{\boldsymbol{x}}
\]
and the projection on ${\cal F}_{S}$ is 
\[
\Pi_{{\cal F}}(f)=\begin{cases}
\frac{f}{\|f\|_{S}}, & \|f\|_{S}\geq1\\
f, & \|f\|_{S}<1
\end{cases}.
\]
However, we may choose to normalize by $\frac{f}{\|f\|_{S}}$ even
if $\|f\|_{S}\leq1$ since in this case the regret is always larger
if $f$ is replaced by $\frac{f}{\|f\|_{S}}$ (in other words, the
worst case function is obtained on the boundary of ${\cal F}_{S}$).
The gradient w.r.t. the predictor $q$ is given by 
\[
\nabla_{q}\E\left[\left(f^{\top}\boldsymbol{x}-q^{\top}R^{\top}\boldsymbol{x}\right)^{2}\right]=\left[-2f^{\top}\Sigma_{\boldsymbol{x}}R+2q^{\top}R^{\top}\Sigma_{\boldsymbol{x}}R\right].
\]
Finally, to derive the gradient w.r.t. $R$, let us denote $R\dfn[R_{1},R_{2},\ldots,R_{r}]\in\mathbb{R}^{d\times r}$
where $R_{i}\in\mathbb{R}^{d}$ is the $i$th column ($i\in[r]$),
and $q^{\top}=(q_{1},q_{2},\ldots,q_{r})$. Then, $q^{\top}R^{\top}\boldsymbol{x}=\sum_{i\in[d]}q_{i}R_{i}^{\top}\boldsymbol{x}$
and the loss function is 
\begin{align}
\E\left[\left(f^{\top}\boldsymbol{x}-q^{\top}R^{\top}\boldsymbol{x}\right)^{2}\right] & =\E\left[\left(f^{\top}\boldsymbol{x}-\sum_{i\in[d]}q_{i}\boldsymbol{x}^{\top}R_{i}\right)^{2}\right]\\
 & =f^{\top}\Sigma_{\boldsymbol{x}}f-2q^{\top}R^{\top}\Sigma_{\boldsymbol{x}}f+q^{\top}R^{\top}\Sigma_{\boldsymbol{x}}Rq.
\end{align}
The gradient of the regret w.r.t. $R_{k}$ is then given by
\begin{align}
\nabla_{R_{k}}\left\{ \E\left[\left(f^{\top}\boldsymbol{x}-q^{\top}R^{\top}\boldsymbol{x}\right)^{2}\right]\right\}  & =-2\E\left[\left(f^{\top}\boldsymbol{x}-q^{\top}R^{\top}\boldsymbol{x}\right)\cdot q_{k}\boldsymbol{x}^{\top}\right]\\
 & =-2q_{k}\left(f^{\top}\Sigma_{\boldsymbol{x}}-q^{\top}R^{\top}\Sigma_{\boldsymbol{x}}\right),
\end{align}
hence, more succinctly, the gradient w.r.t. $R$ is
\[
\nabla_{R}\left\{ \E\left[\left(f^{\top}\boldsymbol{x}-q^{\top}R^{\top}\boldsymbol{x}\right)^{2}\right]\right\} =-2q\left(f^{\top}\Sigma_{\boldsymbol{x}}-q^{\top}R^{\top}\Sigma_{\boldsymbol{x}}\right).
\]
We remark that in the algorithm these gradients are multiplied by
weights. We omit this term whenever the weight is common to all terms
in order to keep the effective step size constant. 

\paragraph*{Initialization }

Algorithm \ref{alg: Iterative algorithm} requires an initial representation
$\rep^{(1)}$ and an initial set of functions $\{f^{(j)}\}_{j\in[m_{0}]}$.
In the MSE setting, each function $f\in\mathbb{R}^{d}$ is also a
single column of a representation matrix $R\in\mathbb{R}^{d\times r}$.
A plausible initialization matrix $\rep^{(1)}\in\mathbb{R}^{d\times r}$
is therefore the worst $r$ functions. These, in turn, can be found
by running Algorithm \eqref{alg: Iterative algorithm} to obtain $\tilde{m}=r$
functions, by setting $\tilde{r}=1$. A proper initialization for
this run is simply an all-zero representation $\tilde{\rep}^{(1)}=0\in\mathbb{R}^{d\times1}$.
The resulting output is then $\{\tilde{\rep}_{(T)}^{(j)}\}_{j\in[r]}$
which can be placed as the $r$ columns of $\rep^{(1)}$. This initialization
is then used for Algorithm \ref{alg: Iterative algorithm}.

\paragraph*{Algorithm parameters}

The algorithm parameters used for Example \ref{exa: Algorithm linear MSE}
are shown in Table \ref{tab: MSE experiments hyper parameters}. The
parameters were optimally tuned for $\sigma_{0}=1$. 

\begin{table}
\begin{centering}
\caption{Parameters for linear MSE setting example\label{tab: MSE experiments hyper parameters}}
\par\end{centering}
\centering{}{\small{}}%
\begin{tabular*}{10cm}{@{\extracolsep{\fill}}ccccc}
\toprule 
{\small{}Parameter} & $\beta_{r}$ & $\beta_{f}$ & $\eta_{r}$ & $\eta_{f}$\tabularnewline
\midrule
\midrule 
{\small{}Value} & $0.94$ & $0.653$ & $0.713$ & $0.944$\tabularnewline
\midrule
\midrule 
{\small{}Parameter} & $T_{\rep}$ & $T_{f}$ & $T_{\text{avg}}$ & $T_{\text{stop}}$\tabularnewline
\midrule
\midrule 
{\small{}Value} & $100$ & until convergence & $10$ & $80$\tabularnewline
\bottomrule
\end{tabular*}{\small\par}
\end{table}

\paragraph*{Additional results }

The learning curve for running Algorithm \ref{alg: Iterative algorithm}
for Example \ref{exa: Algorithm linear MSE} is shown in Figure \ref{fig:The-learning-curve-linear-MSE},
which shows the improvement in regret in each iteration, for which
an additional matrix is added to the set of representations. It can
be seen that mixing roughly $10$ matrices suffice to get close to
the minimal regret attained by the algorithm, compared to the potential
number of ${d \choose r}={20 \choose 3}=1140$ representation matrices
determined by $\overline{{\cal A}}$. 

\begin{figure}
\centering{}\includegraphics[scale=0.7]{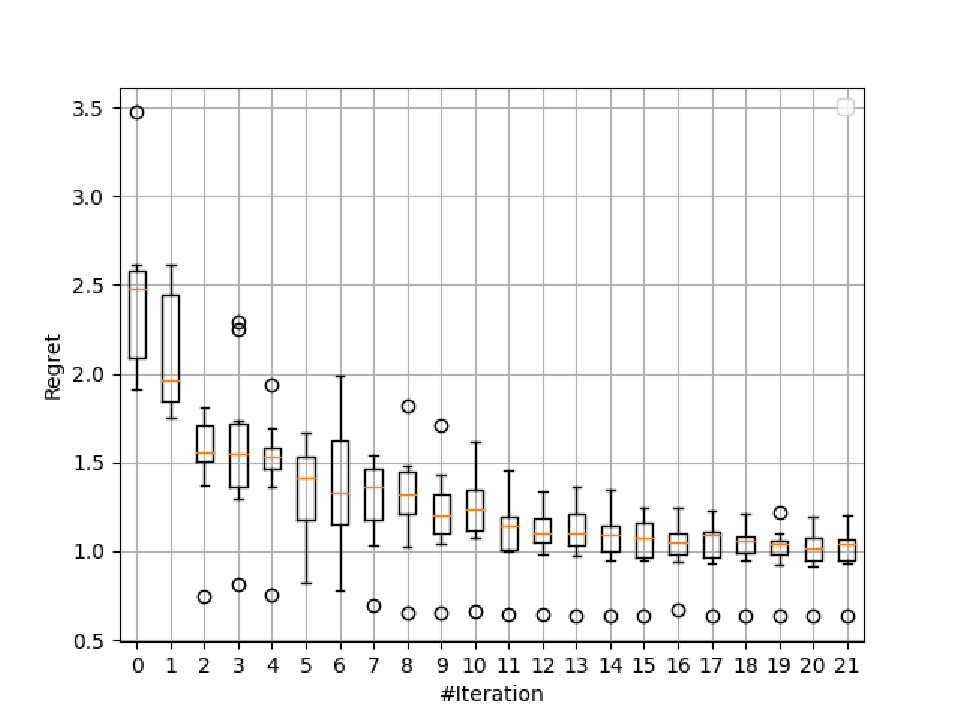}\caption{The learning curve for Algorithm \ref{alg: Iterative algorithm} in
the linear MSE setting: $d=20$, $r=3$, $\sigma=1$. \label{fig:The-learning-curve-linear-MSE}}
\end{figure}

Additional results of the accuracy of the Algorithm \ref{alg: Iterative algorithm}
in the linear MSE setting are displayed in Figure \ref{fig: alg-liner-MSE}.
The left panel of Figure \ref{fig: alg-liner-MSE} shows that the
algorithm output is accurate for small values of $r$, but deteriorates
as $r$ increases. This is because when $r$ increases then so is
$\ell^{*}$ and so is the required number of matrices in the support
of the representation rule (denoted by $m$). Since the algorithm
gradually adds representation matrices to the support, an inaccurate
convergence at an early iteration significantly affects later iterations.
One possible way to remedy this is to run each iteration multiple
times, and choose the best one, before moving on to the next one.
Another reason is that given large number of matrices in the support
(large $m$), it becomes increasingly difficult for the the MWU to
accurately converge. Since the iterations of the MWU do not converge
to the equilibrium point, but rather their average (see discussion
in Appendix \ref{sec:Additional-related-work}) this can only be remedied
by allowing more iterations for convergence (in advance) for large
values of $m$. The right panel of Figure \ref{fig: alg-liner-MSE}
shows that the algorithm output is accurate for a wide range of the
condition number of the covariance matrix. This condition number is
determined by the choice of $\sigma_{0}$, where low values typically
result covariance matrices with condition number that is close to
$1$, while high values will typically result large condition number.
The right panel shows that while the hyperparameters were tuned for
$\sigma_{0}=1$, the result is fairly accurate for a wide range of
$\sigma_{0}$ values, up to $\sigma_{0}\approx5$. Since for $Z\sim N(0,1)$
(standard normal) it holds that $\P[-2<Z<2]\approx95\%$, the typical
condition number of a covariance matrix drawn with $\sigma_{0}=5$
is roughly $\frac{e^{2\sigma_{0}}}{e^{-2\sigma_{0}}}\approx4.85\cdot10^{8}$,
which is a fairly large range. 

\begin{figure}
\begin{centering}
\includegraphics[scale=0.5]{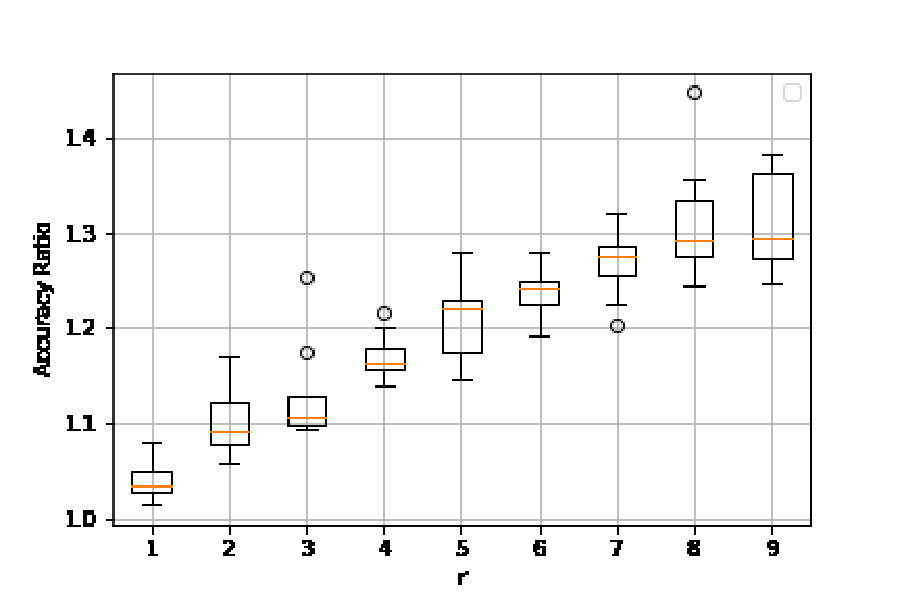}\includegraphics[scale=0.5]{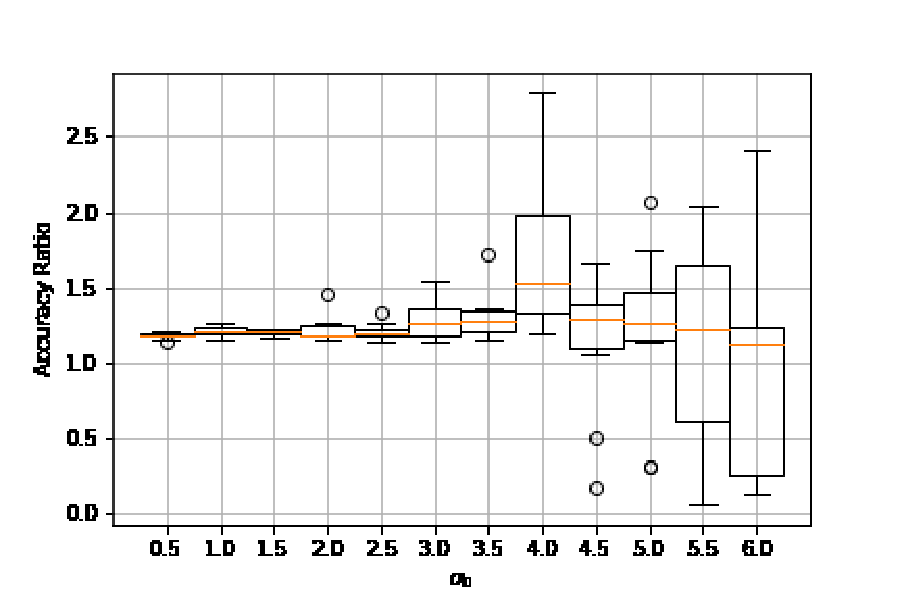}
\par\end{centering}
\caption{The ratio between the regret achieved by Algorithm \ref{alg: Iterative algorithm}
and the theoretical regret in the linear MSE setting. Left: $d=20$,
$\sigma_{0}=1$, varying $r$. Right: $r=5$, $d=20$, varying $\sigma_{0}$.\label{fig: alg-liner-MSE}}
\end{figure}

\subsection{Details for Example \ref{exa: Algorithm linear logistic}: the linear
cross-entropy setting }

In this setting,
\begin{align}
\regret(\rep,f\mid P_{\boldsymbol{x}}) & =\min_{q\in\mathbb{R}^{r}}\E\left[D_{\text{KL}}\left([1+\exp(-f^{\top}\boldsymbol{x})]^{-1}\mid\mid[1+\exp(-q^{\top}R^{\top}\boldsymbol{x})]^{-1}\right)\right],
\end{align}
and the expectation over the feature distribution typically cannot
be carried out analytically. We thus tested Algorithm \ref{alg: Iterative algorithm}
on empirical distributions of samples drawn from a high-dimensional
normal distribution. Specifically, for each run, $B=1000$ feature
vectors were drawn from an isotropic normal distribution of dimension
$d=15$. The expectations of the regret and the corresponding gradients
were then computed with respect to (w.r.t.) the resulting empirical
distributions. 

\paragraph{Regret gradients }

We use the facts that 
\[
\frac{\partial}{\partial p_{1}}D_{\text{KL}}(p_{1}\mid\mid p_{2})=\log\frac{p_{1}(1-p_{2})}{p_{2}(1-p_{1})}
\]
and 
\[
\frac{\partial}{\partial p_{2}}D_{\text{KL}}(p_{1}\mid\mid p_{2})=\frac{p_{2}-p_{1}}{p_{2}(1-p_{2})}.
\]
For brevity, let us next denote 
\[
p_{1}:=\frac{1}{1+\exp(-f^{\top}\boldsymbol{x})}
\]
and 
\[
p_{2}:=\frac{1}{1+\exp(-q^{\top}R^{\top}\boldsymbol{x})}.
\]
We next repeatedly use the chain rule for differentiation. First,
\[
\nabla_{f}p_{1}=\nabla_{f}\left[\frac{1}{1+\exp(-f^{\top}\boldsymbol{x})}\right]=\frac{\exp(-f^{\top}\boldsymbol{x})\cdot\boldsymbol{x}}{\left[1+\exp(-f^{\top}\boldsymbol{x})\right]^{2}}=p_{1}(1-p_{1})\cdot\boldsymbol{x}\cdot
\]
and 
\[
\nabla_{q}p_{2}=\nabla_{q}\left[\frac{1}{1+\exp(-q^{\top}R^{\top}\boldsymbol{x})}\right]=\frac{\exp(-q^{\top}R^{\top}\boldsymbol{x})\cdot R^{\top}\boldsymbol{x}}{\left[1+\exp(-q^{\top}R^{\top}\boldsymbol{x})\right]^{2}}=p_{2}(1-p_{2})\cdot R^{\top}\boldsymbol{x}\cdot
\]
So, assuming that $P_{\boldsymbol{x}}$ is such that the order of
differentiation and expectation may be interchanged (this can be guaranteed
using dominated/monotone convergence theorems), the gradient of the
regret w.r.t. $f$ is 
\begin{align}
\nabla_{f}\regret(\rep,f\mid P_{\boldsymbol{x}}) & =\E\left[\frac{\partial}{\partial p_{1}}D_{\text{KL}}(p_{1}\mid\mid p_{2})\times\nabla_{f}p_{1}\right]\\
 & =\E\left[\log\left(\frac{p_{1}(1-p_{2})}{p_{2}(1-p_{1})}\right)\cdot p_{1}(1-p_{1})\cdot\boldsymbol{x}\right]\\
 & =\E\left[(f^{\top}-q^{\top}R^{\top})\boldsymbol{x}\frac{\exp(-f^{\top}\boldsymbol{x})}{\left[1+\exp(-f^{\top}\boldsymbol{x})\right]^{2}}\cdot\boldsymbol{x}\right]\\
 & =\E\left[\frac{\exp(-f^{\top}\boldsymbol{x})}{\left[1+\exp(-f^{\top}\boldsymbol{x})\right]^{2}}\cdot\boldsymbol{x}^{\top}(f-Rq)\boldsymbol{x}\right].
\end{align}
Next, under similar assumptions, the gradient of the regret w.r.t.
the predictor $q$ is 
\begin{align}
\nabla_{q}\regret(\rep,f\mid P_{\boldsymbol{x}}) & =\E\left[\frac{\partial}{\partial p_{2}}D_{\text{KL}}(p_{1}\mid\mid p_{2})\times\nabla_{q}p_{2}\right]\\
 & =\E\left[\left(\frac{1}{1+\exp(-q^{\top}R^{\top}\boldsymbol{x})}-\frac{1}{1+\exp(-f^{\top}\boldsymbol{x})}\right)\cdot R^{\top}\boldsymbol{x}\right].
\end{align}
Finally, as for the MSE case, to derive the gradient w.r.t. $R$,
we denote $R\dfn[R_{1},R_{2},\ldots,R_{r}]\in\mathbb{R}^{d\times r}$
where $R_{i}\in\mathbb{R}^{d}$ is the $i$th column ($i\in[r]$),
and $q^{\top}=(q_{1},q_{2},\ldots,q_{r})$. Then, $q^{\top}R^{\top}\boldsymbol{x}=\sum_{i\in[d]}q_{i}R_{i}^{\top}\boldsymbol{x}$
and
\[
p_{2}=\frac{1}{1+\exp(-\sum_{i\in[d]}q_{i}R_{i}^{\top}\boldsymbol{x})}.
\]
Then, the gradient of $p_{2}$ w.r.t. $R_{k}$ is then given by 
\[
\nabla_{R_{k}}p_{2}=p_{2}(1-p_{2})\cdot q_{i}\boldsymbol{x},
\]
hence, more succinctly, the gradient w.r.t. $R$ is
\[
\nabla_{R}p_{2}=p_{2}(1-p_{2})\cdot\boldsymbol{x}q^{\top}.
\]
Hence, 
\begin{align}
\nabla_{R}\regret(\rep,f\mid P_{\boldsymbol{x}}) & =\E\left[\frac{\partial}{\partial p_{2}}D_{\text{KL}}(p_{1}\mid\mid p_{2})\times\nabla_{R}p_{2}\right]\\
 & =\E\left[\left(p_{2}-p_{1}\right)\cdot\boldsymbol{x}q^{\top}\right]\\
 & =\E\left[\left(\frac{1}{1+\exp(-q^{\top}R^{\top}\boldsymbol{x})}-\frac{1}{1+\exp(-f^{\top}\boldsymbol{x})}\right)\cdot\boldsymbol{x}q^{\top}\right].
\end{align}

\paragraph*{Initialization }

Here the initialization is similar to the linear MSE setting, except
that since a column of the representation cannot ideally capture even
a single adversarial function, the initialization algorithm only searches
for a single adversarial function ($\tilde{m}=1$). This single function
is then used to produce $\rep^{(1)}$ as the initialization of Algorithm
\ref{alg: Iterative algorithm}.

\paragraph*{Algorithm parameters}

The algorithm parameters used for Example \ref{exa: Algorithm linear logistic}
are shown in Table \ref{tab: logistic experiments hyper parameters}.
\begin{table}
\begin{centering}
\caption{Parameters for linear cross entropy setting example \label{tab: logistic experiments hyper parameters}}
\par\end{centering}
\centering{}{\small{}}%
\begin{tabular*}{10cm}{@{\extracolsep{\fill}}ccccc}
\toprule 
{\small{}Parameter} & $\beta_{r}$ & $\beta_{f}$ & $\eta_{r}$ & $\eta_{f}$\tabularnewline
\midrule
\midrule 
{\small{}Value} & $0.9$ & $0.9$ & $10^{-3}$ & $10^{-1}$\tabularnewline
\midrule
\midrule 
{\small{}Parameter} & $T_{\rep}$ & $T_{f}$ & $T_{\text{avg}}$ & $T_{\text{stop}}$\tabularnewline
\midrule
\midrule 
{\small{}Value} & $100$ & $1000$ & $25$ & $50$\tabularnewline
\bottomrule
\end{tabular*}{\small\par}
\end{table}

\subsection{Details for Example \ref{exa: images with 4 shapes}: An experiment
with a multi-label classification of images and a comparison to PCA
\label{subsec:An-experiment-with}}

We next present the setting of Example \ref{exa: images with 4 shapes},
which shows that large reduction in the representation dimension can
be obtained if the function is known to belong to a finite class. 
\begin{defn}[The multi-label classification setting ]
\label{def:muli-label classification setting} Assume that ${\cal X}=\mathbb{R}^{\sqrt{d}\times\sqrt{d}}$
where $d=625$, and $\boldsymbol{x}$ represents an image. The distribution
$P_{\boldsymbol{x}}$ is such that $\boldsymbol{x}$ contains $4$
shapes selected from a dictionary of $6$ shapes in different locations,
chosen with a uniform probability; see Figure \ref{fig:image with shapes}.
The output is a binary classification ${\cal Y}=\{\pm1\}$ of the
image. Assume that the class of representation is linear $z=\rep(x)=R^{\top}x$
for some $R\in{\cal R}:=\mathbb{R}^{d\times r}$ where $d>r$. The
response function belongs to a class of $6$ different functions ${\cal F}=\{f_{1},\ldots f_{6}\}$,
where $f_{j}:{\cal X}\to{\cal Y}$ indicates whether the $i$th shape
appears in the image or not. Assume the cross-entropy loss function,
where given that the prediction that $\boldsymbol{y}=1$ with probability
$q$ results the loss $\loss(y,q)\dfn-\frac{1}{2}(1+y)\log q-\frac{1}{2}\left(1-y\right)\log(1-q)$.
The set of predictor functions is ${\cal Q}\dfn\left\{ \pre(z)=1/[1+\exp(-q^{\top}\boldsymbol{z})],\;q\in\mathbb{R}^{r}\right\} $,
and the regret is then given by the expected binary Kullback-Leibler
(KL) divergence as in Definition \ref{def: linear logistic log loss }. 
\end{defn}

\paragraph*{Simplifying Algorithm \ref{alg: Iterative algorithm}}

In the the multi-label classification setting of Definition \ref{def:muli-label classification setting},
Algorithm \ref{alg: Iterative algorithm} can be simplified as follows.
First, since the number of response functions in the class ${\cal F}$
is finite, the Phase 1 problem \eqref{eq: mixed optimization for online alg}
in Algorithm \ref{alg: Iterative algorithm} algorithm is simple,
since the adversarial function can be found by a simple maximization
over the $6$ functions. Then, the phase $2$ step simply finds for
each function $f^{(j_{2})}$ in ${\cal F}$, $j_{2}\in[6]$, the best
representation-predictor $(\rep^{(j_{1})},\pre^{(j_{1},j_{2})})$
using gradient descent, where $\rep^{(j_{1})}:{\cal X}\to\mathbb{R}^{r}$
is a linear representation $z=$, and $\pre$ is logistic regression.
This results is a payoff matrix of $\mathbb{R}^{6\times6}$. Then,
the resulting game can be numerically solved as a linear program,
thus obtaining the probability that each representation should be
played. The resulting loss of this minimax rule $\boldsymbol{R}^{*}$
is the loss of our representation. This representation is then compared
with a standard PCA representation, which uses the projections on
the first $r$ principle directions of $R=V_{1:r}(\Sigma_{\boldsymbol{x}})^{\top}$
as the representation (without randomization). The results of the
experiment are shown in Figure \ref{fig: Optimized representation vs PCA}
in the paper.

\begin{figure}
\centering{}\includegraphics[scale=0.5]{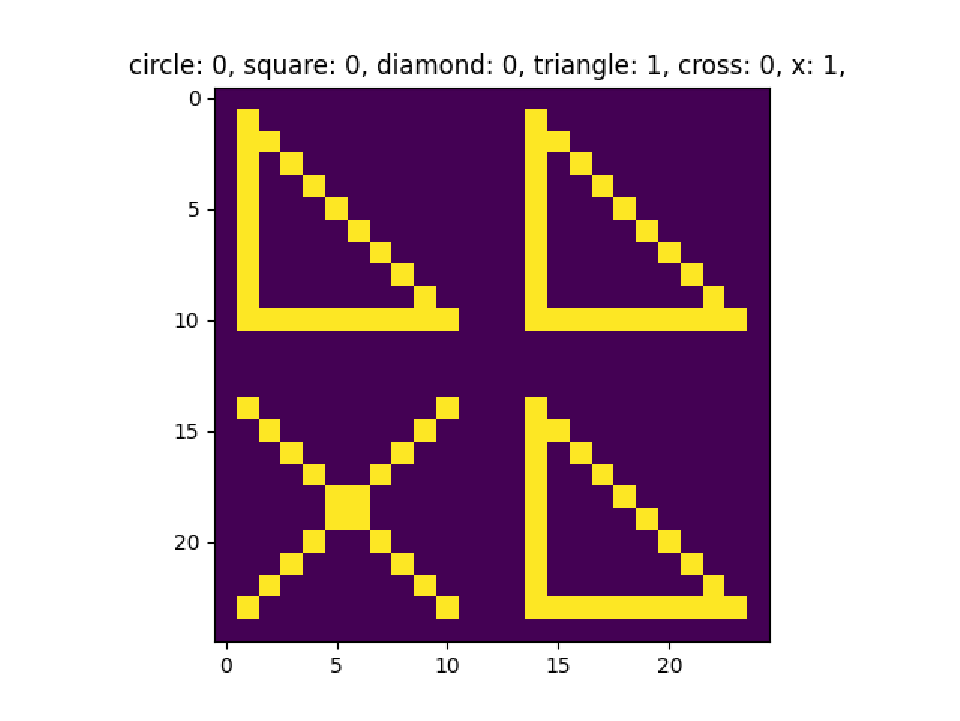}\caption{An image in the dataset for the multi-label classification setting
(Definition \ref{def:muli-label classification setting}). \label{fig:image with shapes}}
\end{figure}

\subsection{An experiment with a NN architecture \label{sec:Additional-experiments}}

In the analysis and the experiments above we have considered basic
linear functions. As mentioned, since the operation of Algorithm \ref{alg: Iterative algorithm}
only depends on the gradients of the loss function, it can be easily
generalized to representations, response functions and predictors
for which such gradients (or sub-gradients) can be provided. In this
section, we exemplify this idea with a simple NN architecture. For
$x\in\mathbb{R}^{d}$, we let the rectifier linear unit (ReLU) be
denoted as $(x)_{+}$.
\begin{defn}[The NN setting]
\label{def: NN logistic log loss}Assume the same setting as in Definitions
\ref{def: linear MSE} and \ref{def: linear logistic log loss },
except that the class of representation, response and predictors are
NN with $c$ hidden layers of sizes $h_{R},h_{f},h_{q}\in\mathbb{N}_{+}$,
respectively, instead of linear functions. Specifically: (1) The representation
is 
\[
\rep(x)=R_{c}^{\top}\left(\cdots\left(R_{1}^{\top}(R_{0}^{\top}x)_{+}\right)_{+}\right)_{+}
\]
for some $(R_{0},R_{1},\cdots R_{c})\in{\cal R}:=\{\mathbb{R}^{d\times h_{R}}\times\mathbb{R}^{h_{R}\times h_{R}}\cdots\mathbb{R}^{h_{R}\times h_{R}}\times\mathbb{R}^{h_{R}\times r}\}$
where $d>r$. (2) The response is determined by 
\[
f(x)=f_{c}^{\top}\left(\cdots\left(F_{1}^{\top}(F_{0}^{\top}x)_{+}\right)_{+}\right)_{+}
\]
where $(F_{0},F_{1},\ldots,f_{c})\in{\cal F}:=\{\mathbb{R}^{d\times h_{f}}\times\mathbb{R}^{h_{f}\times h_{f}}\cdots\mathbb{R}^{h_{f}\times h_{f}}\times\mathbb{R}^{h_{f}}\}$.
(3) The predictor is determined by for some
\[
q(z)=q_{c}^{\top}\left(\cdots\left(Q_{1}^{\top}(Q_{0}^{\top}z)_{+}\right)_{+}\right)_{+}
\]
where $(Q_{0},Q_{1},\ldots,q_{c})\in{\cal Q}:=\{\mathbb{R}^{r\times h_{q}}\times\mathbb{R}^{h_{q}\times h_{q}}\cdots\mathbb{R}^{h_{q}\times h_{q}}\times\mathbb{R}^{h_{q}}\}$.
\end{defn}

\paragraph*{Regret gradients}

Gradients were computed using \texttt{PyTorch} with standard gradients
computation using backpropagation for an SGD optimizer. 

\paragraph*{Initialization }

The initialization algorithm is similar to the initialization algorithm
used in the linear cross-entropy setting. 

\paragraph*{Algorithm parameters}

The algorithm parameters used for the example are shown in Table \ref{tab:Parameters-for-neural-network}.
\begin{table}
\begin{centering}
\caption{Parameters for the NN cross-entropy setting. \label{tab:Parameters-for-neural-network}}
\par\end{centering}
\centering{}{\small{}}%
\begin{tabular*}{10cm}{@{\extracolsep{\fill}}cccccc}
\toprule 
{\small{}Parameter} & $c$ & $h_{R}$ & $h_{f}$ & $h_{q}$ & \tabularnewline
\midrule 
{\small{}Value} & $1$ & $d$ & $d$ & $d$ & \tabularnewline
\midrule
\midrule 
{\small{}Parameter} & $\beta_{r}$ & $\beta_{f}$ & $\eta_{r}$ & $\eta_{f}$ & $\eta_{q}$\tabularnewline
\midrule 
{\small{}Value} & $0.9$ & $0.9$ & $10^{-3}$ & $10^{-1}$ & $10^{-1}$\tabularnewline
\midrule
\midrule 
{\small{}Parameter} & $T_{\rep}$ & $T_{f}$ & $T_{\pre}$ & $T_{\text{avg}}$ & $T_{\text{stop}}$\tabularnewline
\midrule
\midrule 
{\small{}Value} & $100$ & $1000$ & $100$ & $10$ & $80$\tabularnewline
\bottomrule
\end{tabular*}{\small\par}
\end{table}

\paragraph*{Results }

For a single hidden layer, Figure \ref{fig: NN-cross entropy} shows
the reduction of the regret with the iteration for the cross-entropy
loss. 
\begin{figure}
\begin{centering}
\includegraphics[scale=0.7]{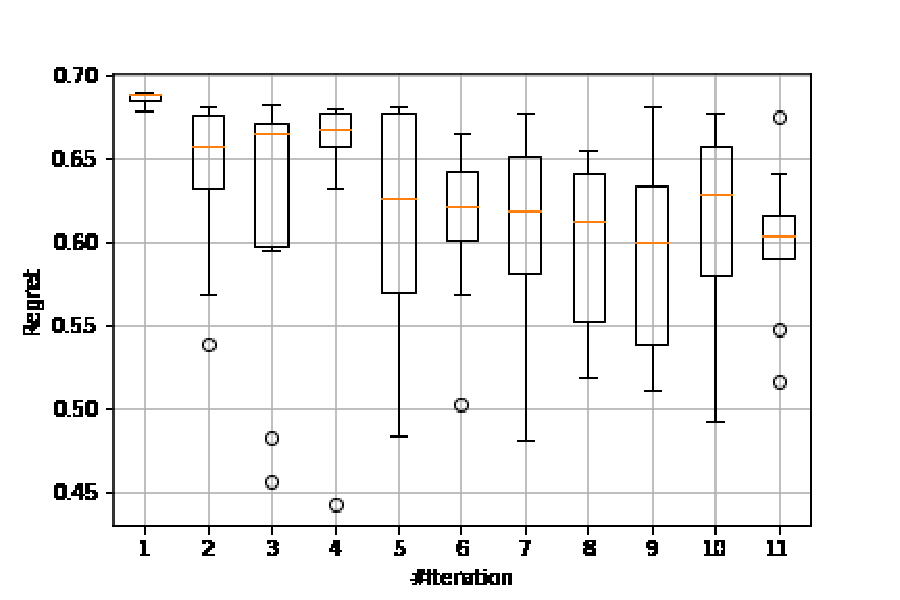}
\par\end{centering}
\caption{The regret achieved by Algorithm \ref{alg: Iterative algorithm} in
the NN cross-entropy setting as a function of the iteration $m$.\label{fig: NN-cross entropy}}
\end{figure}

\bibliographystyle{unsrtnat}
\bibliography{Representation}

\end{document}